\documentclass[nohyperref]{article}

\usepackage{microtype}
\usepackage{graphicx}
\usepackage{booktabs} 
\usepackage[algo2e,linesnumbered]{algorithm2e}

\usepackage{subcaption}
\usepackage{caption}

\usepackage{hyperref}

\usepackage[accepted]{icml2023}

\newcommand{\red}[1]{\textcolor{red}{#1}}

\usepackage{makecell}

\usepackage{amsmath}
\usepackage{amssymb}
\usepackage{mathtools}
\usepackage{amsthm}

\usepackage{amsmath,amsfonts,bm}

\usepackage[capitalize,noabbrev]{cleveref}

\theoremstyle{plain}
\newtheorem{theorem}{Theorem}[section]

\newtheorem{lemma}[theorem]{Lemma}

\newtheorem{example}{Example}
\theoremstyle{definition}

\theoremstyle{remark}

\usepackage{amsmath,amsfonts,bm}

\def\1{\bm{1}}

\def\vr{{\bm{r}}}

\def\vx{{\bm{x}}}

\DeclareMathAlphabet{\mathsfit}{\encodingdefault}{\sfdefault}{m}{sl}
\SetMathAlphabet{\mathsfit}{bold}{\encodingdefault}{\sfdefault}{bx}{n}

\newcommand{\E}{\mathbb{E}}

\DeclareMathOperator*{\argmax}{arg\,max}
\DeclareMathOperator*{\argmin}{arg\,min}

\renewcommand{\E}{\mathbb{E}}
\newcommand{\subG}{\operatorname{subG}}
\newcommand{\subE}{\operatorname{subE}}

\newcommand{\his}{\mathcal{H}}

\newcommand{\var}{\operatorname{var}}
\newcommand{\cov}{\operatorname{cov}}
\newcommand{\pr}{\mathbb{P}}

\newcommand{\name}[1]{\texttt{#1}}

\newcommand{\prob}{\mathbb{P}}
\newcommand{\ind}{\mathbb{I}}

\newlength\oversetwidth
\newlength\underwidth
\newcommand\alignedoverset[2]{
  \settowidth\oversetwidth{$\overset{#1}{#2}$}
  \settowidth\underwidth{$#2$}
  \setlength\oversetwidth{\oversetwidth-\underwidth}
  \hspace{.5\oversetwidth}
  &
  \settowidth\oversetwidth{$\overset{#1}{#2}$}
  \settowidth\underwidth{$#2$}
  \setlength\oversetwidth{\oversetwidth-\underwidth}
  \hspace{-.5\oversetwidth}
  \overset{#1}{#2}
}

\usepackage[backgroundcolor=white]{todonotes}

\newcommand{\todohw}[2][]{\todo[color=orange!20,size=\tiny,inline,#1]{HW: #2}}

\usepackage[utf8]{inputenc}
\usepackage{array}
\usepackage{amsmath}
\usepackage{amssymb}
\usepackage{bm}
\usepackage{extarrows}
\usepackage{natbib}
\usepackage{graphicx}
\usepackage{booktabs}
\usepackage{setspace}
\usepackage{textcomp}
\usepackage{datetime2}
\usepackage{multirow}
\usepackage[font=small,labelfont=bf]{caption}
\usepackage{mathtools}
\usepackage{dirtytalk}

\usepackage[backgroundcolor=white]{todonotes}

\DeclareFontFamily{U}{cjheb}{}
\DeclareFontShape{U}{cjheb}{m}{n}{%
  <-11> s*[1.2] cjhblsm
  <11-> s*[1.2] cjhbltx
}{}
\newcommand{\cjhebfamily}{\fontencoding{U}\fontfamily{cjheb}\selectfont}
\DeclareTextFontCommand{\textcjheb}{\cjhebfamily}

\icmltitlerunning{Multiplier Bootstrap-Based Exploration}

\begin{document}

\twocolumn[
\icmltitle{Multiplier Bootstrap-based Exploration
}
\icmlsetsymbol{equal}{*}

\begin{icmlauthorlist}
\icmlauthor{Runzhe Wan *}{to2}
\icmlauthor{Haoyu Wei *}{to3}
\icmlauthor{Branislav Kveton}{to2}
\icmlauthor{Rui Song}{to2}
\end{icmlauthorlist}

\icmlaffiliation{to2}{Amazon}
\icmlaffiliation{to3}{Department of Statistics, North Carolina State University}

\icmlcorrespondingauthor{Runzhe Wan}{runzhe.wan@gmail.com}


\icmlkeywords{Machine Learning, ICML}

\vskip 0.3in
]

\printAffiliationsAndNotice{\icmlEqualContribution} 


\setlength{\textfloatsep}{8pt plus 1.0pt minus 2.0pt}

\begin{abstract}
Despite the great interest in the bandit problem, 
designing efficient algorithms for complex models remains challenging, as there is typically no analytical way to quantify uncertainty.
In this paper, we propose Multiplier Bootstrap-based Exploration (\name{MBE}), a novel exploration strategy that is applicable to any reward model amenable to weighted loss minimization. 
We prove both instance-dependent and instance-independent rate-optimal regret bounds for  \name{MBE} in sub-Gaussian multi-armed bandits. 
With extensive simulation and real data experiments, we show the generality 
 and adaptivity of \name{MBE}.  
\end{abstract}





\vspace{-.2cm}

\section{Introduction}


The bandit problem has found wide applications in various areas such as clinical trials \citep{durand2018contextual}, finance \citep{shen2015portfolio}, recommendation systems \citep{zhou2017large}, among others. 
Accurate uncertainty quantification is the key to address the exploration-exploitation trade-off. 
Most existing bandit algorithms critically rely on certain analytical property of the imposed model (e.g., linear bandits) to quantify the uncertainty and derive the exploration strategy. 
Thompson Sampling \citep[TS,][]{thompson1933likelihood} and Upper Confidence Bound \citep[UCB,][]{auer2002finite} are two prominent examples, which are typically based on explicit-form posterior distributions or confidence sets, respectively. 

However, in many real problems, the reward model is fairly complex: e.g., a general graphical model \citep{chapelle2009dynamic} or a pipeline with multiple prediction modules and manual rules. 
In these cases, it is typically impossible to quantify the uncertainty in an analytical way, and frameworks such as TS or UCB are either methodologically not applicable or computationally infeasible.
Motivated by the real needs, we are concerned with the following question:



\textit{Can we design a practical bandit algorithm framework that is general, adaptive, and computationally tractable, with certain theoretical guarantee? 
}


A straightforward idea is to apply the bootstrap method \citep{efron1992bootstrap}, a widely applicable data-driven approach for measuring uncertainty. 
However, as discussed in Section \ref{sec:related_work}, most existing bootstrap-based bandit algorithms are either heuristic without a theoretical guarantee, computationally intensive, or only applicable in limited scenarios. 
To address these limitations, we propose a new exploration strategy based on \textit{multiplier bootstrap} \citep{van1996weak}, an easy-to-adapt bootstrap framework that only requires randomly weighted data points. 
We further show that a naive application of multiplier bootstrap may result in linear regret, and we introduce a suitable way to add additional perturbations for sufficient exploration.  








\textbf{Contribution. }
Our contributions are three-fold. 
First,  we propose a general-purpose bandit algorithm framework, \textit{Multiplier Bootstrap-based Exploration} (\name{MBE}). 
The main advantage of \name{MBE} is that it is \textit{general}: it is applicable to any reward model amenable to weighted loss minimization,  without need of analytical-form uncertainty quantification or case-by-case algorithm design. 
As a data-driven exploration strategy, \name{MBE} is also \textit{adaptive} to different environments. 

Second, theoretically, we prove near-optimal regret bounds for \name{MBE} under sub-Gaussian multi-armed bandits (MAB), in both the instance-dependent and the instance-independent sense. 
Compared with all existing results for bootstrap-based bandit algorithms, our result is strictly more general (see Table \ref{tab:comparison}), since existing results only apply to some special cases of sub-Gaussian distributions. 
To overcome the technical challenges, 
we proved a novel concentration inequality for some function of sub-exponential variables, and also developed the first \textit{finite-sample} concentration and anti-concentration analysis for multiplier bootstrap, to the best of our knowledge. 
Given the broad applications of multiplier bootstrap in statistics and machine learning, we believe our theoretical analysis has separate interest. 


\begin{table*}[t!]
\caption{Comparisons between several bootstrap-based bandit algorithms. 
We divide the sources of exploration into leveraging the \textit{intrinsic} randomness in the observed data and manually adding  \textit{extrinsic} perturbations that are independent of the observed data.
All papers derive near-optimal regret bounds in MAB, with different reward distribution requirements. 
To compare the computational cost, we focus on MAB to illustrate, and consider Algorithm \ref{alg:MBTS-practical} for \name{MBE}. 
See Section \ref{sec:related_work} for more details of discussions in this table. 
}
\label{tab:comparison}
\centering
\begin{tabular}{lllll}
\toprule
&  \makecell{Exploration  \\ Source }  
& \makecell{Methodology \\ Generality}
& \makecell{Theory \\ Requirement}  
& \makecell{Computation \\ Cost}  \\ \hline
\name{MBE} (this paper)  & \makecell{intrinsic \\ \& extrinsic}   &       general          &        \makecell{sub-Gaussian}  &   $\mathcal{O}(KT)$  \\ \hline
\name{GIRO} \citep{kveton2019garbage} &   \makecell{intrinsic \\ \& extrinsic}  &    general                      & \makecell{Bernoulli}  &   $\mathcal{O}(T^2)$  \\ \hline
\makecell{\name{ReBoot} } \citep{wang2020residual, wu2022residual} & \makecell{intrinsic \\ \& extrinsic}  &  \makecell{fixed \& finite \\ set of arms}       &    \makecell{ Gaussian}              &  $\mathcal{O}(KT)$ \\ \hline
\name{PHE} \citep{kveton2019perturbed, kveton2019perturbedLB, kveton2019perturbedGLB} &  \makecell{only extrinsic}  &   \makecell{general} & bounded  &   $\mathcal{O}(KT)$  \\\bottomrule 
\end{tabular}


\end{table*}

Third, with extensive simulation and real data experiments, we demonstrate that \name{MBE} yields comparable performance with existing  algorithms in different MAB settings and three real-world problems (online learning to rank, online combinatorial optimization, and dynamic slate optimization). 
This supports that \name{MBE} is easily generalizable, as it requires minimal modifications and derivations to match the performance of those near-optimal algorithms specifically designed for each problem. 
Moreover, we also show that \name{MBE} adapts to different environments and is relatively robust, due to its data-driven nature.

\section{Related Work}\label{sec:related_work}

The most popular bandit algorithms, arguably, include $\epsilon$-greedy \citep{watkins1989learning}, TS, and UCB. 
$\epsilon$-greedy is simple and thus widely used. 
However, its exploration strategy is not aware of the uncertainty in data and thus is known to be statistically sub-optimal. 
TS and UCB reply on posteriors and confidence sets, respectively. 
Yet, their closed forms only exist in limited cases, such as MAB or linear bandits. 
For a few other models (such as generalized linear model or neural nets), 
we know how to construct the \textit{approximate} posteriors or confidence sets \citep{filippi2010parametric, li2017provably, phan2019thompson, kveton2020randomized, zhou2020neural}, though the corresponding algorithms are usually costly or conservative. 
In more general cases, it is often not clear how to adapt UCB and TS in a valid and efficient way. 
Moreover, the dependency on the probabilistic model assumptions also pose challenges to being robust. 




To enable wider applications of bandit algorithms, several bootstrap-based (and related perturbation-based) methods have been proposed in the literature. 
Most algorithms are TS-type, by  replacing the posterior with a bootstrap distribution. 
We next review the related papers, and summarize those with near-optimal regret bounds in Table \ref{tab:comparison}. 


Arguably, the non-parametric bootstrap is the most well-known bootstrap method, which works by re-sampling data with replacement. 
\citet{vaswani2018new} proposes a version of non-parametric bootstrap 
with forced exploration to achieve 
a $\mathcal{O}(T^{2/3})$ regret bound in Bernoulli MAB. 
\name{GIRO} proposed in \citet{kveton2019garbage} successfully achieves a rate-optimal regret bound in Bernoulli MAB, by adding Bernoulli perturbations to non-parametric bootstrap. 
However, due to the re-sampling nature of non-parametric bootstrap, it is hard to be updated efficiently other than in Bernoulli MAB (see Section \ref{sec:computation_variant}). 
Specifically, the computational cost of re-sampling scales quadratically in $T$. 





Another line of research is the residual bootstrap-based approach (\name{ReBoot}) \citep{hao2019bootstrapping, wang2020residual, tang2021robust, wu2022residual}. 
For each arm, \name{ReBoot} randomly perturbs the residuals of the corresponding observed rewards with respect to the estimated model to quantify the uncertainty for its mean reward. 
We note that, although these methods also use random weights, they are applied to residuals and hence are in a way fundamentally different from ours. 
The limitation is that, by design, this approach is only applicable to problems with a  \textit{fixed}  and  \textit{finite} set of arms, since the residuals are attached closely to each arm (see Appendix \ref{sec:ReBoot} for more details).






The perturbed history exploration (\name{PHE}) algorithm \citep{kveton2019perturbed, kveton2019perturbedLB, kveton2019perturbedGLB} is also related. 
\name{PHE} works by adding additive noise to the observed rewards.  
\citet{osband2019deep} applies similar ideas to reinforcement learning. 
However, \name{PHE} has two main limitations. 
First, for models where adding additive noise is not feasible (e.g., decision trees), \name{PHE} is not applicable. 
Second, as demonstrated in both \cite{wang2020residual} and our experiments, the fact that \name{PHE} relies on only the extrinsically injected noise for exploration makes it less robust. 
For a complex structured problem, it may not be clear how to add the noise in a sound way \citep{wang2020residual}. 
In contrast, it is typically more natural (and hence easier to be accepted) to leverage the intrinsic randomness in the observed data. 





Finally, we note that multiplier bootstrap has been considered in the bandit literature, mostly as a computationally efficient approximation to non-parametric bootstrap studied in those papers. 
\citet{eckles2014thompson}  studies the direct adaption of multiplier bootstrap (see Section \ref{sec:fail_naive}) in simulation, and its empirical performance in contextual bandits is studied later \citep{tang2015personalized, elmachtoub2017practical, riquelme2018deep, bietti2021contextual}. 
However, no theoretical guarantee is provided in these works. 
In fact, as demonstrated in Section \ref{sec:fail_naive}, such a naive adaptation may have a linear regret.
\citet{osband2015bootstrapped} shows that, in Bernoulli MAB, a variant of multiplier bootstrap is mathematically equivalent to TS. 
No further theoretical or numerical results are established except for this special case. 
Our work is the first systematic study of multiplier bootstrap in bandits.  
Our unique contributions include: 
we identify the potential failure of naively applying multiplier bootstrap, 
highlight the importance of additional perturbations, 
design a general algorithm framework to make this heuristic idea concrete, 
provide the first theoretical guarantee in general MAB settings, 
and conduct extensive numerical experiments to study its generality and adaptivity. 



\section{Preliminary}\label{sec:preliminary}
\textbf{Setup. } 
We consider a general stochastic bandit problem. 
For any positive integer $M$, let $[M] = \{1, \dots, M\}$. 
 At each round $t \in [T]$, the agent observes a context vector $\vx_t$ (it is empty in non-contextual problems) and an action set $\mathcal{A}_t$, then chooses an action $A_t \in \mathcal{A}_t$, and finally receives the corresponding reward $R_t = f(\vx_t, A_t) + \epsilon_t$, 
Here, $f$ is an unknown function and $\epsilon_t$ is the noise term. 
Without loss of generality, we assume $f(\vx_t, A_t) \in [0,1]$. 
The goal is to minimize the cumulative regret 
\[
\operatorname{Reg}_T = \sum_{t=1}^T \E \big[ \max_{a \in \mathcal{A}_t} f(\vx_t, a)  - f(\vx_t, A_t) \big].
\]
At time $t$, 
with an existing dataset $\mathcal{D}_t = \{(\vx_l, A_l, R_l)\}_{l \in [t]}$, 
to decide the action $A_{t+1}$, 
most algorithms typically first estimate $f$ in some function class $\mathcal{F}$ by solving a weighted loss minimization problem (also called weighted empirical risk minimization or cost-sensitive training) 
\begin{equation}\label{eqn:weighted_loss_mini}
\widehat{f} = 
\argmin_{f \in \mathcal{F}} 
\frac{1}{t}
\sum_{l=1}^{t}
\omega_l \mathcal{L}\big( 
f(\vx_l, A_l), R_l
\big) + {J}(f).  
\end{equation}
Here, $\mathcal{L}$ is a loss function (e.g., $\ell_2$ loss or negative log-likelihood), $\omega_l$ is the weight of the $l$th data point, and $J$ is an optional penalty function. 
We consider the weighted problem as it is general and related to our proposal below. 
One can just set $\omega_l \equiv 1$ to get the unweighted problem. 
As the simplest example, consider the $K$-armed bandit problem where $\vx_l$ is empty and $\mathcal{A}_l = [K]$.  
Let $\mathcal{L}$ be the $\ell_2$ loss, $J \equiv 0$, and $f(\vx_l, A_l) \equiv r_{A_l}$ where $r_k$ is the mean reward of the $k$-th arm.  
Then, \eqref{eqn:weighted_loss_mini} reduces to 
$\argmin_{\{r_1, \dots, r_K\}} \sum_{l=1}^t \omega_l (R_l - r_{A_l})^2$, 
which gives the estimator $\widehat{r}_k = 
(\sum_{l:A_l = k} \omega_l)^{-1}
\sum_{l:A_l = k} \omega_l R_l$, i.e., the arm-wise weighted average. 
Similarly, in linear bandits, \eqref{eqn:weighted_loss_mini} reduces to the weighted least-square problem (see Appendix \ref{sec:MBTS_LB} for details).




\textbf{Challenges. } 
The estimation of $f$, together with the related uncertainty quantification, forms the foundation of most bandit algorithms. 
In the literature, $\mathcal{F}$ is typically a class of models that permit closed-form uncertainty quantification (e.g., linear models, Gaussian processes, etc.). 
However, in many real applications, the reward model can yield a fairly complicated structure, e.g., a hierarchical pipeline with both classification and regression modules. 
Manually specified rules are also commonly part of the model. 
It is challenging to quantify the uncertainty of these complicated models in analytical forms. 
Even when feasible, 
the dependency on the probabilistic model assumptions also pose challenges to being robust. 




Therefore, in this paper, we focus on the bootstrap-based approach due to its generality and data-driven nature.  
Bootstrapping, as a general approach to quantify the model uncertainty, has many variants. 
The most popular one, arguably, is non-parametric bootstrap (used in \name{GIRO}), which constructs bootstrap samples by re-sampling the dataset with replacement. 
However, due to the re-sampling nature, it is computationally intense (see Section \ref{sec:computation_variant} for more discussions). 
In contrast, multiplier bootstrap \citep{van1996weak}, as an efficient and easy-to-implement alternative, is popular in statistics and machine learning.


\textbf{Multiplier bootstrap. }
The main idea of multiplier bootstrap is to learn the model using randomly weighted data points. 
Specifically, given a multiplier weight distribution $\rho(\omega)$, 
for every bootstrap sample, 
we first randomly sample $\{\omega_t^{MB}\}_{t=1}^{t'} \sim \rho(\omega)$, and then solve \eqref{eqn:weighted_loss_mini} with $\omega_t = \omega_t^{MB}$ to obtain $\widehat{f}^{MB}$. 
Repeat the procedure and the distribution of $\widehat{f}^{MB}$ forms the \textit{bootstrap distribution} that quantifies our uncertainty over $f$. 
The popular choices of  $\rho(\omega)$ include  $\mathcal{N}(1, \sigma_{\omega}^2)$,  $\text{Exp}(1)$,  $\text{Poisson}(1)$, and the double-or-nothing distribution $ 2 \times \text{Bernoulli}(0.5)$.

\section{Multiplier Bootstrap-based Exploration}
\subsection{Failure of the naive adaption of multiplier bootstrap}\label{sec:fail_naive}
To design an exploration strategy based on multiplier bootstrap, 
a natural idea is to replace the posterior distribution in TS with the bootstrap distribution. 
Specifically, at every time point, 
we sample a function $\widehat{f}$ following the multiplier bootstrap procedure as described in Section \ref{sec:preliminary}, 
and then take the greedy action $\argmax_{a \in \mathcal{A}_t} \widehat{f}(\vx_t, a)$. 
However, perhaps surprisingly, such an adaptation may not be valid. 
The main reason is that the intrinsic randomness in a finite dataset is, in some cases, not enough to guarantee sufficient exploration. 
We illustrate with the following toy example. 
\begin{example}\label{eg1}
Consider a two-armed Bernoulli bandit. 
Let the mean rewards of the two arms be $p_1$ and $p_2$, respectively. 
Without loss of generality, assume $1 > p_1 > p_2 > 0$. 
Let $\prob(\omega = 0) = 0$. 
Then, with non-zero probability, an agent following the naive adaption of multiplier bootstrap (breaking ties randomly; see Algorithm \ref{alg:MBTS-naive} in Appendix \ref{sec:naive} for details) takes arm $1$ only once. 
Therefore, The agent suffers a linear regret. 

Proof. We first define two events $\mathcal{E}_1 = \{A_t = 1, R_1 = 0\}$ and $\mathcal{E}_2 = \{A_2 = 2, R_2 = 1\}$. 
By design, 
at time $t = 1$, the agent randomly choose an arm and hence will pull arm $1$ with probability $0.5$. 
Then the observed reward $R_1$ is $0$ with probability $1 - p_1$. 
Therefore, $\prob(\mathcal{E}_1) = 0.5(1-p_1)$. 
Conditioned on $\mathcal{E}_1$, at $t = 2$, the agent will pull arm $2$ (since multiplying $R_1 = 0$ with any weight always gives $0$), then it will observe reward $R_2 = 1$ with probability $p_2$. 
Conditioned on $\mathcal{E}_1 \cap \mathcal{E}_2$, 
by induction, the agent will pull arm $2$ for any $t > 2$. 
This is because the only reward record for arm $1$ is $R_1 = 0$ and hence its weighted average is always $0$, which is smaller than the weighted average for arm $2$, which is at least positive. 
In conclusion, with probability at least $0.5 \times (1-p_1) \times p_2 > 0$, the algorithm takes the optimal arm $1$ only once. 
\end{example}


\subsection{Main algorithm}\label{sec:main_alg}

The failure of the naive application of multiplier bootstrap implies that some additional randomness is needed to ensure sufficient exploration. 
In this paper, we consider achieving that by adding \textit{pseudo-rewards}, an approach that proves its effectiveness in a few other setups \citep{kveton2019garbage, wang2020residual}. 
The intuition is as follows. 
The under-exploration issue happens when, by randomness, the observed rewards are in the low-value region (compared with the expected reward). 
Therefore, if we can blend in some data points with rewards that have a relatively wide coverage, then the agent would have a higher chance to explore. 



These discussions motivate the design of our main algorithm, Multiplier Bootstrap-based Exploration (\name{MBE}), as in Algorithm \ref{alg:MBTS-general}. 
Specifically, at every round, in addition to the observed reward, we additionally add two pseudo-rewards with value $0$ and $1$. 
The pseudo-rewards are associated with the pulled arm and the context (if exists). 
Then, we solve a weighted loss minimization problem to update the model estimation  (line 8). 
The weights are first sampled from a multiplier distribution  (line 7), and then those of pseudo-rewards are additionally multiplied by a tuning parameter $\lambda$. 
In MAB, the estimates are arm-wise weighted average of all (observed or pseudo-) rewards 
$\overline{Y}_k = \sum_{\ell: A_\ell = k} (\omega_{\ell}  R_{k, \ell} +  \lambda  \omega_{\ell}')   /   \sum_{\ell: A_\ell = k} (\omega_{\ell} + \lambda \omega_{\ell}' + \lambda \omega_{\ell}'') $. 
See Appendix \ref{sec:MBTS_MAB} for details.




We make three remarks on the algorithm design. 
First, we choose to add pseudo-rewards at the boundaries of the mean reward range (i.e.,  $[0,1]$), since such a design naturally induces a high variance (and hence more exploration). 
Adding pseudo-rewards in other manners is also possible. 
Second, the tuning parameter $\lambda$ controls the amount of extrinsic perturbation and determines the degree of exploration (together with the dispersion of $\rho(\omega)$). 
In Section \ref{sec:theory}, we give a theoretically valid range for $\lambda$. 
Finally and critically, besides guaranteeing sufficient exploration, we need to make sure the optimal arm can still be identified (asymptotically) after adding the pseudo-rewards. 
Intuitively, this is guaranteed, since we shift and scale the (asymptotic) mean reward from $f(\vx, a)$ to $\big( f(\vx, a) + \lambda \big) / (1 + 2\lambda) = 
 f(\vx, a) / (1 + 2\lambda) + \lambda / (1 + 2\lambda)$, which preserves the order between arms. 
A detailed analysis for MAB can be found in Appendix \ref{sec:MBTS_MAB}.

We conclude this section by re-visiting Example \ref{eg1} to provide some insights into how do the pseudo-rewards help. 
\addtocounter{example}{-1}
\begin{example}[Continued]
    Even under the event $\mathcal{E}_1 \, \cap \, \mathcal{E}_2$, 
    Algorithm \ref{alg:MBTS-general} keeps the chance to explore. 
    To see this, consider the example where the multiplier distribution is $2 \times \text{Bernoulli}(0.5)$. 
    Then, we have 
    $\pr(A_3 = 1) \ge \pr(\overline{Y}_1 > \overline{Y}_0) = 
            \pr\left(\frac{\lambda \omega_1'}{\omega_1 + \lambda \omega_1' + \lambda \omega_1''} > \frac{\omega_2 + \lambda \omega_2'}{\omega_2 + \lambda \omega_2' + \lambda \omega_2''}\right)
            \ge 
            \pr(\omega'_1 = 2, \omega_1 = \omega_1'' = \omega_2 = \omega_2' = \omega_2'' = 0) = (1/2)^6. $
    Therefore, the agent still has chance to choose the optimal arm. 
\end{example} 

\begin{algorithm}[!h]
\setcounter{AlgoLine}{1}
\KwData{
Function class $\mathcal{F}$, 
loss function $\mathcal{L}$, 
(optional) penalty function ${J}$, 
multiplier weight distribution $\rho(\omega)$, 
tuning parameter $\lambda$
}


Initialize $\widehat{f}$ 


\For{$t = 1, \dots, T$}{

    Observe context $\vx_t$ and action set $\mathcal{A}_t$

    Offer $A_t = \argmax_{a \in \mathcal{A}_t} \widehat{f}(\vx_t, A)$ (break ties randomly)

    Observe reward $R_t$
    

    Sample the multiplier weights $\{\omega_l, \omega_l', \omega_l''\}_{l=1}^t \sim \rho(\omega)$

    Solve the weighted loss minimization problem 
    \vspace{-.4cm}
    \begin{align*}
    \widehat{f} = 
    \argmin_{f \in \mathcal{F}} 
    \sum_{l=1}^{t}
    &\Big[
    \omega_l \mathcal{L}\big( 
    f(\vx_l, A_l), R_l
    \big) \\
    &+
        \lambda\omega'_l \mathcal{L}\big( 
    f(\vx_l, A_l), 0
    \big) \\
    &+
        \lambda\omega''_l \mathcal{L}\big( 
    f(\vx_l, A_l), 1
    \big) 
    \Big]
    + {J}(f).  
    \end{align*}
    \vspace{-.6cm}

    
} 

\caption{General Template for \name{MBE} }\label{alg:MBTS-general}
\end{algorithm}



\subsection{Computationally efficient implementation}\label{sec:computation_variant}




Efficient computation is critical for real applications of bandit algorithms. 
One potential limitation of Algorithm  \ref{alg:MBTS-general}  is the computational burden: 
at every decision point, 
we need to re-sample the weights for all historical observations (line 8). 
This leads to a total computational cost of order $\mathcal{O}(T^2)$, similar to \name{GIRO}. 




Fortunately, one prominent advantage of multiplier bootstrap over other bootstrap methods (such as non-parametric bootstrap or residual bootstrap) is that the (approximate) bootstrap distribution can be efficiently updated in an online manner, such that the per-round computation cost does not grow over time. 
Suppose we have a dataset $\mathcal{D}_t$ at time  $t$, and 
denote $\mathcal{B}(\mathcal{D}_t)$ as the corresponding bootstrap distribution for $f$.
With multiplier bootstrap, it is feasible to update $\mathcal{B}(\mathcal{D}_{t+1})$ approximately based on $\mathcal{B}(\mathcal{D}_{t})$. 
We detail the procedure below and elaborate more in Algorithm \ref{alg:MBTS-practical}. 







Specifically, we maintain $B$ different models $\{\widehat{f}_{b,t}\}_{b=1}^B$ 
and the corresponding history (with random weights) as $\{\his_b, \his'_b\}_{b=1}^B$. 
 $\{\widehat{f}_{b,t}\}_{b=1}^B$ can be regarded as sampled from $\mathcal{B}(\mathcal{D}_t)$ and hence the empirical distribution over them is an approximation to the bootstrap distribution. 
At every time point $t$, for each replicate $b$, we only need to sample one weight for the new data point and then update $\widehat{f}_{b,t}$ as  $\widehat{f}_{b,t+1}$.  
Then, $\{\widehat{f}_{b,t+1}\}_{b=1}^B$ are still $B$ valid samples from $\mathcal{B}(\mathcal{D}_{t+1})$ and hence still a valid approximation. 
We note that, since we only have one new data point, 
the updating of $f$ can typically be done efficiently (e.g., with closed-form updating or via online gradient descent). 
The per-round computational cost is hence independent of $t$. 



Such an approximation is a common practice in the online bootstrap literature and can be regarded as an ensemble sampling-type algorithm \citep{lu2017ensemble, qin2022an}. 
The hyper-parameter $B$ is typically not treated as a tuning parameter but depends on the available computational resource \citep{hao2019bootstrapping}. 
In our numerical experiments, this practical variant shows desired performance with $B = 50$. 
Moreover, the algorithm is embarrassingly parallel and also easy to implement: 
given an existing implementation for estimating $f$ (i.e., solving \eqref{eqn:weighted_loss_mini}), 
the major requirement is to replicate it for $B$ times and use random weights for each. 
This feature is attactive in real applications.





\begin{algorithm}[!h]
\SetAlgoLined
\setcounter{AlgoLine}{1}
\KwData{
Number of bootstrap replicates $B$,
function class $\mathcal{F}$, 
Loss function $\mathcal{L}$, 
(optional) penalty function ${J}$, 
weight distribution $\rho(\omega)$, 
tuning parameter $\lambda$
}

Set $\his_b = \{\}$ be the history and $\his'_b = \{\}$ be the pseudo-history, for any $b \in [B]$

Initialize $\widehat{f}_{b,0}$ for any $b \in [B]$ 


\For{$t = 1, \dots, T$}{
    Observe context $\vx_t$ and action set $\mathcal{A}_t$
    
    Sample an index $b_t$ uniformly from $\{1, \dots, B\}$

    Offer $A_t = \arg \max_{A \in \mathcal{A}_t} \widehat{f}_{b_t,t-1}(\vx_t, A)$ (break ties randomly)

    Observe reward $R_t$

    \For{b = 1, \dots, B}{    
    Sample the weights $\omega_{l,b}, \omega_{l,b}', \omega_{l,b}'' \sim \rho(\omega)$

    Update $\his_b = \his_b \cup \big\{(\vx_t, A_t, R_t, \omega_{l,b})\big\}$
    and  $\his'_b = \his'_b \cup \big\{(\vx_t, A_t, 0, \omega_{l, b}'), (\vx_t, A_t, 1, \omega_{l, b}'')\big\}$ 
    
    Solve the weighted loss minimization problem 
    \vspace{-.4cm}
    \begin{align*}
    \widehat{f}_{b,t} = 
    \argmin_{f \in \mathcal{F}} 
    \sum_{l=1}^{t}
    &\Big[
    \omega_{l,b} \mathcal{L}\big( 
    f(\vx_l, A_l), R_l
    \big) \\
    &+
        \lambda\omega'_{l,b} \mathcal{L}\big( 
    f(\vx_l, A_l), 0
    \big) \\
    &+
        \lambda\omega''_{l,b} \mathcal{L}\big( 
    f(\vx_l, A_l), 1
    \big) 
    \Big]
    + {J}(f).  
    \end{align*}
    \vspace{-.6cm}
    }
} 

 \caption{Practical Implementation of \name{MBE} }\label{alg:MBTS-practical}
\end{algorithm}

\section{Regret Analysis}\label{sec:theory}
In this section, we provide the regret bound for Algorithm \ref{alg:MBTS-general} under MAB with sub-Gaussian rewards. 
We regard this as the first step towards the theoretical understanding of \name{MBE}, and leave the analysis of more general settings to future research. 
We call a random variable $X$ as $\sigma$-sub-Gaussian if $\E \exp \{ t (X - \E X)\} \leq \exp \{ {t^2 \sigma^2} / {2}\}$ for any $t \in \mathbb{R}$.

\begin{theorem}\label{thm.MAB.1}
Consider a $K$-armed bandit, where the reward distribution of arm $k$ is $1$-sub-Gaussian with mean $\mu_k$. Suppose arm $1$ is the unique best arm that has the highest mean reward and $\Delta_k=\mu_1-\mu_k$. 
Take the multiplier weight distribution as $\mathcal{N}(1, \sigma_{\omega}^2)$ in Algorithm \ref{alg:MBTS-MAB-2}. 
Let the tuning parameters satisfy $\lambda \geq \left(1 + 4 / \sigma_{\omega} \right) + \sqrt{4  \left(1 + 4 / \sigma_{\omega} \right)/ \sigma_{\omega}  }$.
Then, the problem-dependent regret is upper bounded by
\[
     \operatorname{Reg}_T \leq \sum_{k = 2}^K \Big\{ 7 \Delta_k  + \frac{10 \big[ C_1^*(\lambda, \sigma_{\omega}) + C_2^*( \lambda, \sigma_{\omega}) \big]}{\Delta_k} \log T \Big\},
\]
and the problem-independent regret is  bounded by
\begin{align*}
    \operatorname{Reg}_T  \leq 7 K \mu_1 +  C_1^*(\lambda, \sigma_{\omega}) K \log T + 2 \sqrt{ C_2^* (\lambda, \sigma_{\omega})  K T \log T}. 
\end{align*}
Here, 
$C_1^*(\lambda, \sigma_{\omega}) = 8 \sqrt{2}  C^*_3(\lambda, \sigma_{\omega}) + 38  \sigma_{\omega}^2$ and
$C_2^*( \lambda, \sigma_{\omega}) =  5 \lambda^2 + \left[ 45 (3 + \sigma_{\omega}^2)\lambda^4 C^*_3(\lambda, \sigma_{\omega})  + 38  \sigma_{\omega}^2 \right]$  
are tuning parameter-related components. 
$C^*_3(\lambda, \sigma_{\omega})$ is a logarithmic term as $\log \big[ ( 1 + 15 \sigma_{\omega}^{-2} + 3 \sigma_{\omega} + 10 \sigma_{\omega}^2)  \lambda^2\big] / (3 \log 2) + 1$. \\ 

\end{theorem}


The two regret bounds are known as near-optimal (up to a logarithm term) in both the problem-dependent and problem-independent sense \citep{lattimore2020bandit}. 
Notably, recall that the Gaussian distribution and all bounded distributions belong to the sub-Gaussian class. 
Therefore, as reviewed in Table \ref{tab:comparison}, our theory is strictly more general than all existing results for bootstrap-based MAB algorithms. 







\textbf{Technical challenges. }
It is particularly  challenging to analyze \name{MBE} due to two reasons. 
First, the probabilistic analysis of multiplier bootstrap itself is technically challenging, since the same random weights appear in both the denominator and the numerator (recall \name{MBE} uses the weighted averages 
$\big\{ \overline{Y}_k = \sum_{\ell: A_\ell = k} (\omega_{\ell}  R_{k, \ell} +  \lambda  \omega_{\ell}')   /   \sum_{\ell: A_\ell = k} (\omega_{\ell} + \lambda \omega_{\ell}' + \lambda \omega_{\ell}'')   \big\}_{k=1}^K$ to select actions in MAB).
It is notoriously complicated to analyze the ratio of random variables, especially when they are correlated. 
Besides, existing bootstrap-based papers rely on the properties of specific \textit{parametric} reward classes (e.g., Bernoulli in \citet{kveton2019garbage} and Gaussian  in  \citet{wang2020residual}), while  we lose these nice structures when considering sub-Gaussian rewards. 

To overcome these challenges, 
we start with  carefully defining two good events on which the weighted average $\overline{Y}_k$, the non-weighted average (with pseudo-rewards)  $\overline{R}_{k}^* = \big( {\sum_{\ell: A_\ell = k} (R_{k, \ell} + 1 \times \lambda + 0 \times \lambda) } \big) / \big( {\sum_{\ell: A_\ell = k} (1 + \lambda + \lambda) } \big)$, and the shifted asymptotic mean $(\mu_k + \lambda) / ({1 + 2 \lambda})$ are close to each other (see Appendix \ref{proof_thm}).
To bound the probability of the bad event and to control the regret on the bad event, 
we face two major technical challenges. 
First, when transforming the ratio into an analyzable form, a summation of correlated sub-Gaussian and sub-exponential variables appears and is hard to analyze. 
We carefully design and analyze a novel event to remove the correlation and the sub-Gaussian terms (see proof of Lemma \ref{lem_bounding_a_k_s_2}). 
Second, the proof needs a new concentration inequality for functions of sub-exponential variables that do not exist in the literature. 
We obtain such a new concentration inequality (Lemma \ref{lem_subE_ineq}) via careful analysis of sub-exponential distributions.


We believe our new concentration inequality is of separate interest for the analysis of sub-exponential distributions. 
Moreover, to the best of our knowledge, our proof provides the first \textit{finite-sample} concentration and anti-concentration analysis for multiplier bootstrap, which has broad applications in statistics and machine learning.

\textbf{Tuning parameters. }
In Theorem \ref{thm.MAB.1}, 
\name{MBE} has two tuning parameters $\lambda$ and $\sigma_{\omega}$. 
Intuitively, $\lambda$ controls the amount of external perturbation and $\sigma_{\omega}$ controls the magnitude of exploration from bootstrap. 
In general, higher values of these two parameters facilitate exploration but also lead to a slower convergence. 
The condition $\lambda \geq \left(1 + 4 / \sigma_{\omega} \right) + \sqrt{4  \left(1 + 4 / \sigma_{\omega} \right)/ \sigma_{\omega}  }$ requires that (i) $\lambda$ is not too small and (ii) the joint effect of $\lambda$ and $\sigma$ is not too small. 
Both are intuitive and reasonable. 
In practice, the theoretical condition could be loose: e.g., it requires $\lambda \geq 5 + 2\sqrt{5}$ when $\sigma_{\omega} = 1$. 
As we observe in Section \ref{sec:expriments}, \name{MBE} with a smaller $\lambda$ (e.g., $0.5$) still empirically performs well.

\section{Experiments}\label{sec:expriments}


In this section, we study the empirical performance of \name{MBE} with both simulation (Section \ref{sec:exp_MAB}) and real datasets (Section \ref{sec:exp_CB}). 


\begin{figure*}[!t]
    \includegraphics[width=\linewidth]{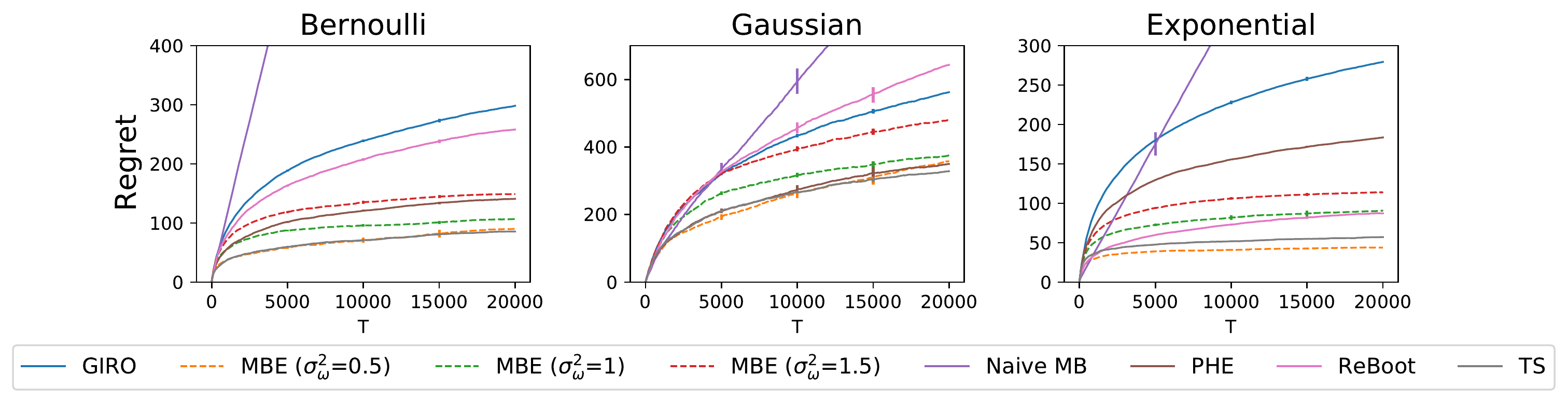}
    \vspace{-.8cm}
    \caption{Simulation results under MAB. The error bars indicate the standard errors, which may not be visible when the width is small. } 
    \label{sim_MAB}
\end{figure*}


 
\subsection{MAB Simulation}\label{sec:exp_MAB}
We first experiment with simulated MAB instances.  
The goal is to (i) further validate our theoretical findings, (ii) check whether \name{MBE} can yield comparable performance with standard methods, and (iii) study the robustness and adaptivity
 of \name{MBE}. 
We also experimented with linear bandits and the main findings are similar. 
To save space, we defer its results to Appendix \ref{sec:results_LB}.

We compare \name{MBE} with \name{TS} \citep{thompson1933likelihood},  \name{PHE} \citep{kveton2019perturbed}, \name{ReBoot} \citep{wang2020residual}, and \name{GIRO} \citep{kveton2019garbage}. 
The last three algorithms are the existing bootstrap- or perturbation-type algorithms reviewed in Section \ref{sec:related_work}. 
Specifically, \name{PHE} explores by  perturbing observed rewards with additive noise, without leveraging the intrinsic uncertainly in the data, 
\name{ReBoot} explores by perturbing the residuals of the rewards observed for each arm, 
and \name{GIRO} re-samples observed data points.  
In all experiments below, the weights of \name{MBE} are sampled from $\mathcal{N}(1,\sigma_{\omega}^2)$ \footnote{We also experimented with other weight distributions with similar main conclusions. 
Using Gaussian weights allows us to study impact of different multiplier magnitudes more clearly. 
}. 
We fix $\lambda = 0.5$ and run \name{MBE} with three different values of $\sigma_{\omega}^2$ as $0.5, 1$ and $1.5$. 
We also compare with the naive adaption of multiplier bootstrap (i.e., no pseudo-rewards; denoted as \name{Naive MB}). 
We run Algorithm \ref{alg:MBTS-practical} with $B=50$ replicates. 






We first study $10$-armed bandits, where the mean reward of each arm is independently sampled from $\text{Beta}(1, 8)$. 
We consider three reward distributions, including Bernoulli, Gaussian, and exponential. 
For Gaussian MAB, the reward noise is sampled from $\mathcal{N} (0, 1)$. 
The other two distributions are determined by their means. 
For \name{TS}, we always use the correct reward distribution class and its conjugate prior. 
The prior mean and variance are calibrated using the true model. 
Therefore, we use \name{TS} as a strong and standard baseline. 
For \name{GIRO} and \name{ReBoot}, we use the default implementations as they work well. 
For \name{PHE}, the original paper adds Bernoulli perturbation since it only studies bounded reward distributions. 
We extend \name{PHE} by sampling additive noise from the same distribution family as the true rewards, as did in \citet{wu2022residual}. 
\name{GIRO}, \name{ReBoot} and \name{PHE} all have one tuning parameter that control the degree of exploration. 
We tune these hyper-parameters over $\{2^{k-4}\}_{k=0}^6$, and report the best performance of each method. 
Without tuning, these algorithms generally do not perform well using the hyper-parameters suggested in the original papers, due to the differences in settings. 
We tuned \name{Naive MB} as well. 






\textbf{Results. }
Results over $100$ runs are  displayed in Figure \ref{sim_MAB}. 
Our findings can be summarized as follows. 
First, without knowledge of the problem settings (e.g., the reward distribution family and its parameters, and the prior distribution) and without heavy tuning, \name{MBE} performs favorably and close to \name{TS}. 
Second,  pseudo-rewards are indeed important in exploration, otherwise the algorithm suffers a linear regret. 
Third, \name{MBE} has a stable performance with different $\sigma_\omega$ (note that other methods are tuned for their best performance). 
This is thanks to the data-driven nature of  \name{MBE}. 
Finally, the other three general-purpose exploration strategies perform reasonably after tuning, as expected. 
However, \name{GIRO} is computationally intense. 
For example, in Gaussian bandits, 
the time cost for \name{GIRO} is $2$ minutes while all the other algorithms can complete within $10$ seconds. 
The computational burden is due to the limitation of non-parametric bootstrap (see Section \ref{sec:computation_variant}). 
\name{ReBoot} also performs reasonably, yet by design it is not easy to extend it to many more complex problems (e.g., problems in Section \ref{sec:exp_CB}).

\textbf{Adaptivity. }
\name{PHE} relies on sampling additive noise from an appropriate distribution, and \name{TS} has similar dependency. 
In the results above, we provide auxiliary information about the environment to them and need to modify their implementation in different setups. 
In contrast, \name{MBE} automatically adapts to these problems. 
As argued in Section \ref{sec:related_work}, one main advantage of \name{MBE} over them is its adaptiveness. 
To see this, we consider the following procedure: 
we run the Gaussian versions of \name{TS} and \name{PHE} in Bernoulli MAB, and run their Bernoulli versions in Gaussian MAB. 
We also run \name{MBE} with $\sigma_\omega^2 = 0.5$. 
\name{MBE} does not require any modifications across the two problems. 
The results presented in Figure \ref{fig:robust_dist} clearly demonstrate that \name{MBE} adapts to reward distributions.

Similarly, we also studied the adaptivity of these methods against the reward distribution scale (the standard deviation of the Gaussian noise, $\sigma$) and the task distribution (we sample the mean rewards from $\text{Beta}(\alpha, 8)$ and vary the parameter $\alpha$). 
For all settings, we use the algorithms tuned for Figure \ref{sim_MAB}. 
We observe that, \name{MBE} shows impressive adaptiveness, while \name{PHE} and \name{TS} may not perform well when the environment is not close to the one they are tuned for. 
Recall that, in real applications, heavy tuning is not possible without ground truths. 
This demonstrates the adaptivity of \name{MBE}, as a data-driven exploration strategy.

\textbf{Additional results. }
In Appendix \ref{sec:results_robust_lam}, we also try different values of $\lambda$ and $B$ for \name{MBE}. 
We also repeat the main experiment with $K = 25$. 
Our main observations above still hold and \name{MBE} is relatively robust to these tuning parameters. 



\begin{figure}[!t]
    \includegraphics[width=\linewidth]{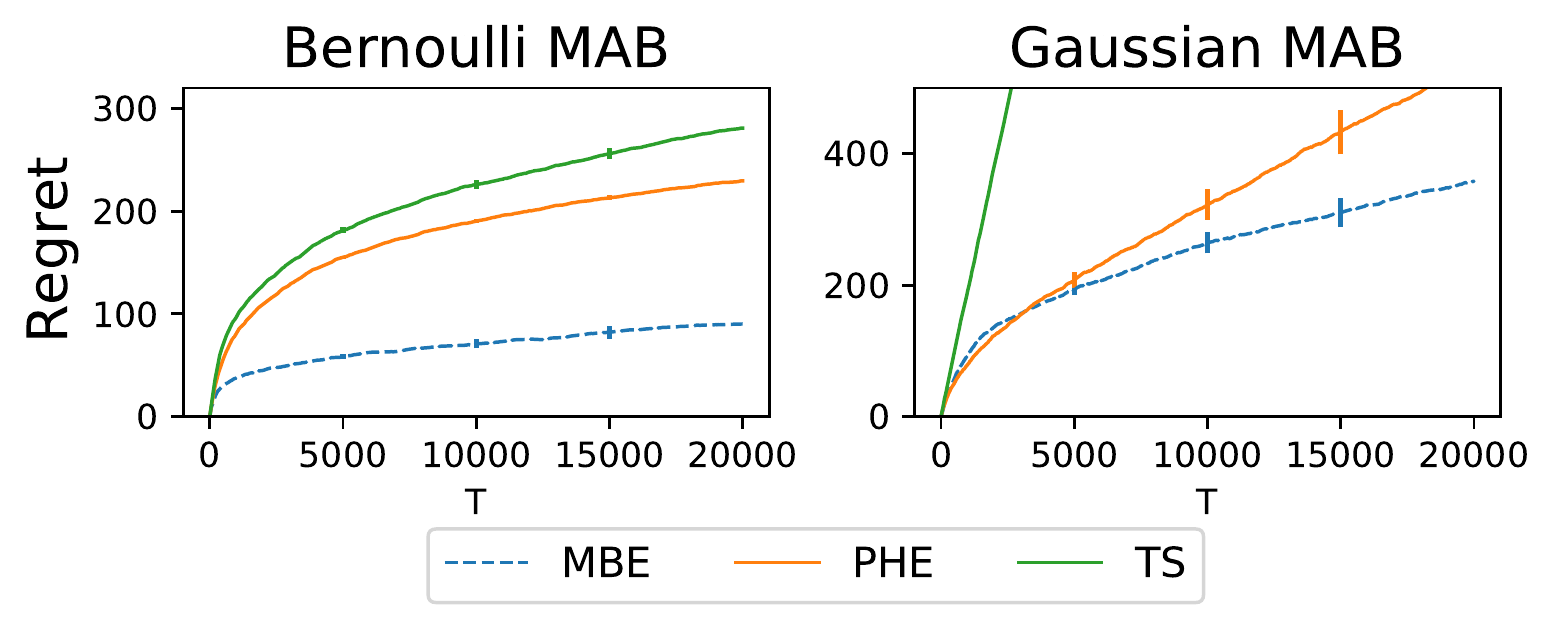}
    \vspace{-.8cm}
    \caption{Robust results against the reward distribution class. }
    \label{fig:robust_dist}
    \vspace{-.2cm}
\end{figure}


\begin{figure}[t]
    \centering
    \subcaptionbox{Trend with  $\alpha$ }
  {\includegraphics[width=0.22\textwidth]{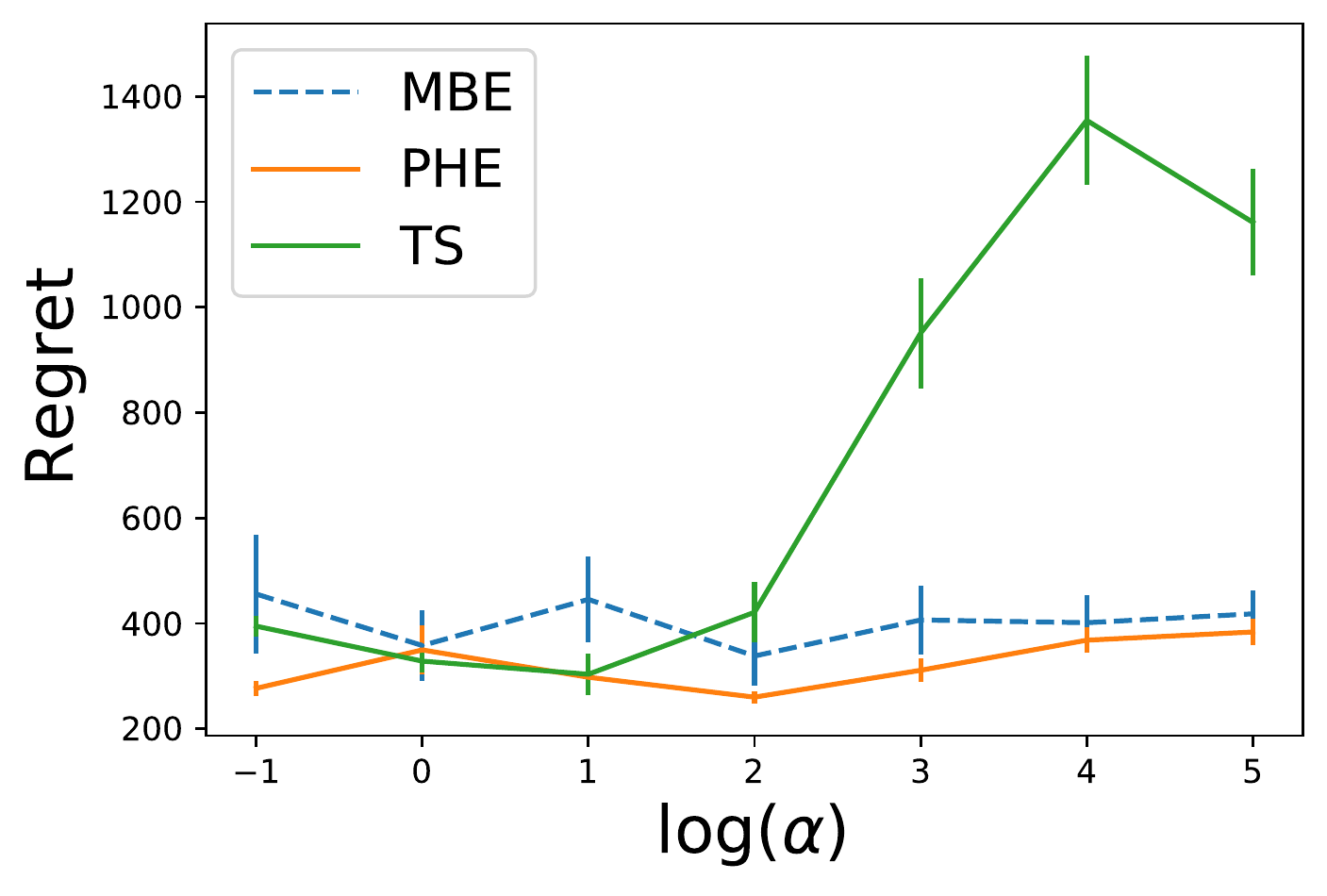}  }
  \hspace{.1cm}
    \subcaptionbox{Trend with $\sigma$}
  {\includegraphics[width=0.22\textwidth]{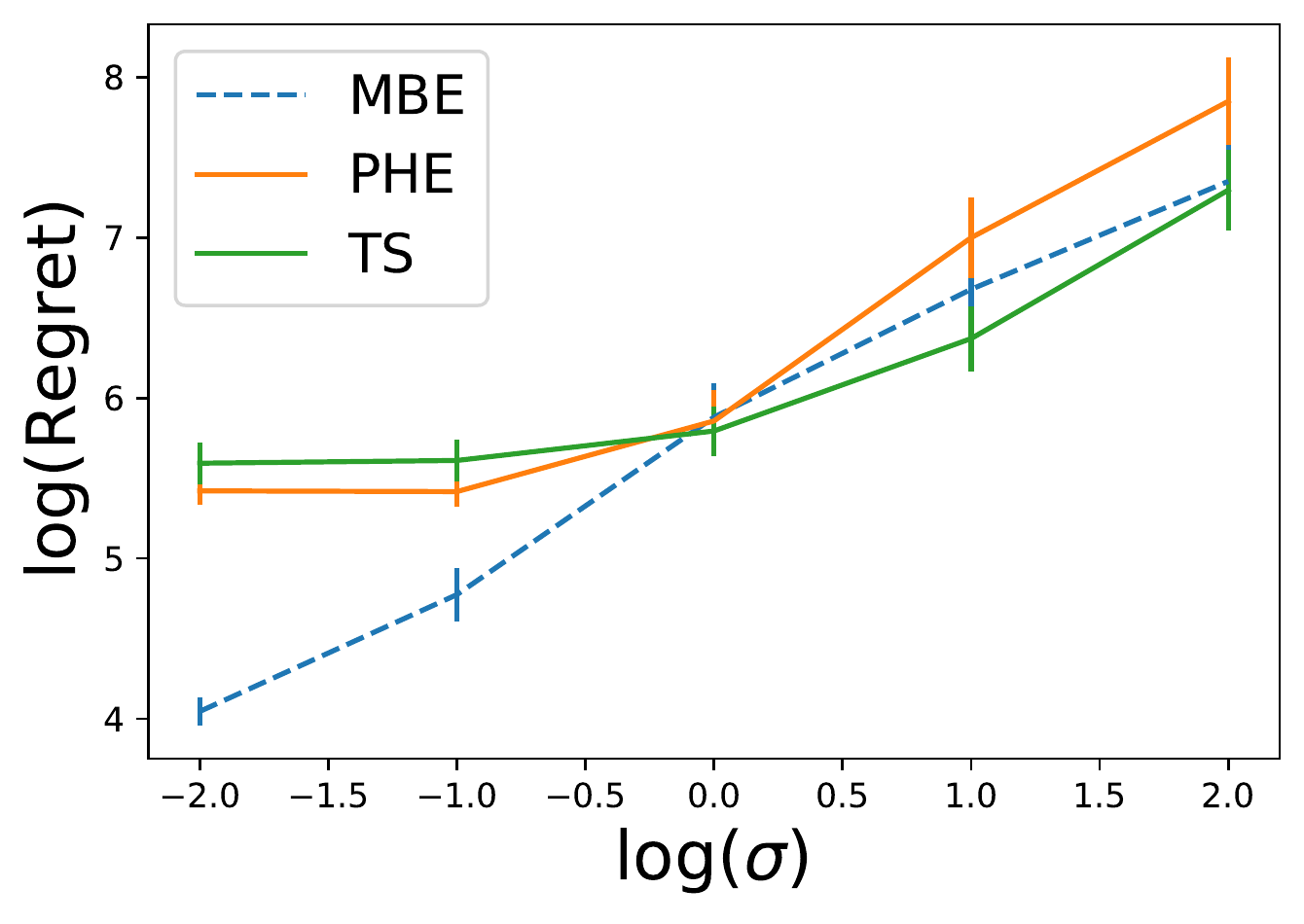}  }
    \vspace{-.1cm}
    \caption{Results with different reward variances and  task distributions. 
    For the x-axis in both figures and the y-axis in the second one, we plot at the logarithmic scale for better visualization.}
\end{figure}

\begin{figure*}[!htp]
    \includegraphics[width=0.95\linewidth]{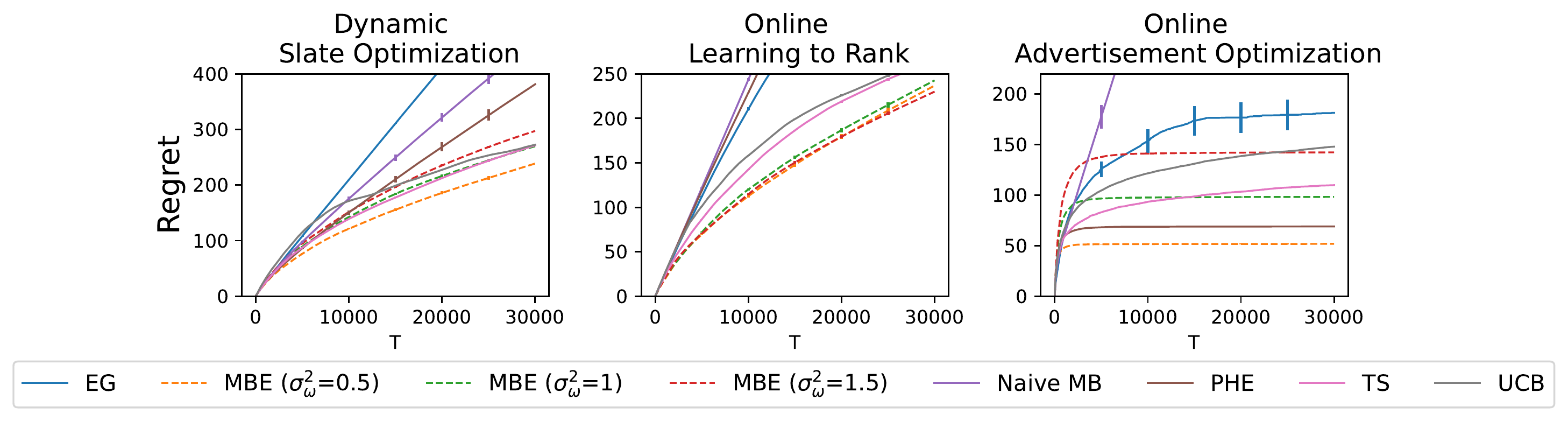} 
    \vspace{-.5cm}
    \caption{Real data results for three structured bandit problems that need domain-specific models. 
}
    \label{fig:real}
\end{figure*}

\subsection{Real data applications}\label{sec:exp_CB}
The main advantage of \name{MBE} is that it easily generalizes to complex models. 
In this section, we use real datasets to study this property. 
Our goal is to investigate that, without problem-specific algorithm design and without heavy tuning, whether \name{MBE} can achieve comparable performance with strong problem-specific baselines proposed in the literature.





We study the three problems considered in \citet{wan2022towards}, including 
cascading bandits for online learning to rank  \citep{kveton2015cascading}, 
combinatorial semi-bandits for online combinatorial optimization \citep{chen2013combinatorial}, 
and multinomial logit (MNL) bandits for dynamic slate optimization \citep{agrawal2017thompson, agrawal2019mnl}.  
All these are practical and important problems in real life. 
Yet, these domain models all have unique structures and require a case-by-case algorithm design.  
For example, the rewards in MNL bandits follow multinomial distributions that have complex dependency with the pulled arms. 
To derive the posterior or confidence bound, one has to use a delicately designed epoch-type procedure \citep{agrawal2019mnl}.

We compared \name{MBE} with state-of-the-art baselines in the literature, including \name{TS-Cascade} \citep{zhong2021thompson} and \name{CascadeKL-UCB} \citep{kveton2015cascading} for cascading bandits, 
\name{CUCB} \citep{chen2016combinatorial} and \name{CTS} \citep{wang2018thompson} for semi-bandits, 
and \name{MNL-TS} \citep{agrawal2017thompson} and \name{MNL-UCB} \citep{agrawal2019mnl} for MNL bandits. 
To save space, we denote the TS-type algorithms by \name{TS} and UCB-type ones by \name{UCB}. 
We also study \name{PHE} and $\epsilon$-greedy (\name{EG}) as two other general-purpose exploration strategies. 





We use the three datasets studied in  \citet{wan2022towards}.  
Specifically, we use the Yelp rating dataset \citep{zong2016cascading} to recommend and rank $K$ restaurants, 
use the Adult dataset \cite{Dua:2019} to send advertisements to $K/2$ men and $K/2$ women (a combinatorial semi-bandit problem with continuous rewards),  
and use the MovieLens dataset \citep{harper2015movielens} to display $K$ movies. 
In our experiments, we fix $K= 4$ and randomly sample $30$ items from the dataset to choose from. 
We provide a summary of these datasets and problems in Appendix \ref{sec:exp_details}, and refer interested readers to \citet{wan2022towards} and references therein for more details. 




For the baseline methods, as in  Section \ref{sec:exp_MAB}, we either use the default hyperparameters in \citet{wan2022towards} or tune them extensively and present their best performance. 
For \name{EG}, we set the exploration rate $\epsilon_t = \min(1, a/2\sqrt{t})$ with tuning parameter  $a$. 
For \name{MBE}, with every bootstrap sample, we estimate the reward model via maximum weighted likelihood estimation, which yields nice closed-form solution that allows online updating in all three problems. 
The other implementation details are similar to Section \ref{sec:exp_MAB}. 


We present the results in Figure \ref{fig:real}. 
The overall findings are consistent with simulation. 
First, without any additional derivations or algorithm design, 
\name{MBE} matches the performance of problem-specific algorithms. 
Second, pseudo-rewards are important to guarantee sufficient exploration, and naively applying multiplier bootstrap may fail. 
Third, \name{MBE} has relatively stable performance with $\sigma_\omega$, as its exploration is mostly data-driven. 
In contrast, we found that the hyper-parameters of \name{PHE} and \name{EG} have to be carefully tuned, due to that they rely on externally added perturbation or forced exploration. 
For example, the best parameters for \name{EG} are $a = 5$, $0.1$ and $0.5$ in three problems. 
Finally, we observe that \name{PHE} does not perform well in MNL and cascading bandits, where the outcomes are binary. 
From a closer look, we find one possible reason: 
the response rates (i.e., the probabilities for the binary outcome to be $1$) in the two datasets are low, so \name{PHE} introduces too much (additive) noise for exploration purpose, which slows down the estimation convergence. 










\section{Conclusion}

In this paper, we propose a new bandit exploration strategy, Multiplier Bootstrap-based Exploration (\name{MBE}). 
The main advantage of \name{MBE} is its generality: for any reward model that can be estimated via weighted loss minimization, 
the idea of  \name{MBE} is applicable and requires minimal efforts on derivation or implementation of the exploration mechanism. 
As a data-driven method, \name{MBE} also shows nice adaptivity. 
We prove near-optimal regret bounds for  \name{MBE} in the sub-Gaussian MAB setup, which is more general compared with other bootstrap-based bandit papers. 
Numerical experiments demonstrate that  \name{MBE} is general, efficient, and adaptive.

There are a few meaningful future extensions. 
First, the regret analysis for \name{MBE} (and more generally, other bootstrap-based bandit methods) in more complicated setups would be valuable. 
Second, adding pseudo-rewards at every round is needed for the analysis. 
We hypothesize that there exists a more adaptive way of adding them. 
Last, the practical implementation of \name{MBE} relies on an ensemble of models to approximate the bootstrap distribution and the online regression oracle to update the model estimation. 
Our numerical experiments show that such an approach works well empirically, but it would be still meaningful to have more theoretical understanding.






\newpage

\bibliography{0_Main.bib}
\bibliographystyle{icml2023}


\newpage
\appendix
\onecolumn
\renewcommand{\appendixname}{Appendix~\Alph{section}}

\section{Additional Method Details}\label{sec:appendix_details}
\subsection{\name{MBE} for MAB}\label{sec:MBTS_MAB}
In this section, we present the concrete form of \name{MBE} when being applied to MAB. 
Recall that $\vx_t$ is null, $A_t \in [K]$, and $r_k$ is the mean reward of the $k$-th arm. 
We define $f(\vx_t, A_t; \vr) = r_{A_t}$, where the parameter vector $\vr = (r_1, \dots, r_K)^{\top}$. 
We define the loss function as 
\begin{align*}\label{eqn:MAB_form}
    \frac{1}{t'}\sum_{t=1}^{t'} \omega_t (r_{A_t} - R_t)^2. 
\end{align*}
The solution 
is then $(\widehat{r}_1, \dots, \widehat{r}_K)^{\top}$ with $\widehat{r}_k = 
(\sum_{t:A_t = k} \omega_t)^{-1}
\sum_{t:A_t = k} \omega_t R_t $, i.e., the arm-wise weighted average. After adding the pseudo rewards, we can give algorithm for MAB in Algorithm \ref{alg:MBTS-MAB-2}. 

Next, we provide intuitive explanation on why Algorithm \ref{alg:MBTS-MAB-2} works. Indeed, denote $s:= |\his_{k, T}|$, where $\his_{k, T}$ is the set of observed rewards for the $k$-th arm up to round $T$. 
Let $R_{k, l}$ be the $l$-th element in $\his_{k, T}$. 
Then
\[
    \begin{aligned}
        \overline{Y}_{k, s} & = \frac{\sum_{i = 1}^s \omega_i R_{k, i} + \lambda \sum_{i = 1}^s \omega_i'}{\sum_{i = 1}^s \omega_i + \lambda \sum_{i = 1}^s \omega_i' + \lambda \sum_{i = 1}^s \omega_i''} \\
        & = \frac{s^{-1} \sum_{i = 1}^s \omega_i ( R_{k, i} - \mu_k) + s^{-1} \sum_{i = 1}^s (\omega_i - 1) + \lambda s^{-1} \sum_{i = 1}^s (\omega_i' - 1) + \mu_k + \lambda}{s^{-1} \sum_{i = 1}^s (\omega_i - 1) + \lambda s^{-1} \sum_{i = 1}^s (\omega_i' - 1)  + \lambda s^{-1} \sum_{i = 1}^s (\omega_i'' - 1) + 1 +  2 \lambda} \, \xlongrightarrow{\pr} \, \frac{\mu_k + \lambda}{1 + 2 \lambda} 
    \end{aligned}
\]
by using the law of large numbers. 
Then, by Slutsky's theorem, 
 \[
     \begin{aligned}
         \sqrt{s} \left[ \overline{Y}_{k, s} - \frac{\mu_k + \lambda}{1 + 2 \lambda} \right] = \frac{1}{1 + 2\lambda} \left[ \frac{1}{\sqrt{s}} \sum_{i = 1}^s \omega_i ( R_{k, i} - \mu_k) + \frac{1}{\sqrt{s}} \sum_{i = 1}^s (\omega_i - 1) + \frac{\lambda}{\sqrt{s}} \sum_{i = 1}^s (\omega_i' - 1) \right] + o_p(1)
     \end{aligned}
 \]
will weakly converge to a mean-zero Gaussian distribution $\mathcal{N}\left(0, \frac{\sigma_k^2 + 2}{(1 + 2 \lambda)^2}  \sigma_{\omega}^2 \right)$. Therefore, our algorithm  preserves the order
of the arms for any $\lambda > 0$.

\begin{algorithm}[!h]
\SetAlgoLined
\setcounter{AlgoLine}{1}
\KwData{
number of arms $K$, 
multiplier weight distribution $\rho(\omega)$, 
tuning parameter $\lambda$
}

Set $\his_{k} = \{\}$ be the history of the arm $k$ and $\overline{Y}_{k} = +\infty, \forall k \in [K]$

\For{$t = 1, \dots, T$}{

    Pull $A_t = \argmax_{k \in [K]} \overline{Y}_{k}$ (break tie randomly), 

    Observe reward $R_t$
    
    Set $\his_{k} = \his_{k} \cup \{R_t\}$

    \For{$k = 1, \dots, K$}{
    \If{$|\his_{k}| > 0$}{
    
    Sample the multiplier weights $\{\omega_l, \omega'_l, \omega''_l\}_{l=1}^{|\his_{k}|} \sim \rho(\omega)$. 
    
    Update the mean reward 
    \begin{align*}
        \overline{Y}_{k} = \left(\sum_{\ell = 1}^{|\his_{k}|} (\omega_{\ell} \cdot R_{k, \ell} +  \omega_{\ell}' \cdot 1 \times \lambda + \omega_{\ell}'' \cdot 0 \times \lambda) \right) / \left( \sum_{\ell = 1}^{|\his_{k}|} (\omega_{\ell} + \lambda \omega_{\ell}' + \lambda \omega_{\ell}'') \right), 
    \end{align*}
    where $R_{k, l}$ is the $l$-th element in $\his_{k}$. 

    }
    }
} 

\caption{\name{MBE} for MAB with sub-Gaussian rewards with mean bounded in $[0, 1]$}\label{alg:MBTS-MAB-2}
\end{algorithm}


\subsection{\name{MBE} for stochastic linear bandits}\label{sec:MBTS_LB}


\begin{algorithm}[!h]
\SetAlgoLined
\setcounter{AlgoLine}{1}
\KwData{
number of arms $K$, 
multiplier weight distribution $\rho(\omega)$, 
tuning parameter $\lambda$
}

Set $\his_{k} = \{\}$ be the history of the arm $k$, set  $A_0 = \mathbf{0}$, $\widehat{\boldsymbol\theta}_0 = \mathbf{0}$ with $b_0 = \mathbf{0}$, and $V_0 = (1 + \xi) I_p$.

\If{$t =1 , \ldots, p$}{
    Offer $A_t = t$.
}

\For{$t = p + 1, \dots, T$}{

    Offer $A_{t} =  \argmax_{\mathbf{a} \in \mathcal{A}_t} \mathbf{a}^{\top} \boldsymbol\theta_{t}$ (break tie randomly)

    Observe reward $R_t$
    
    Set $\his_{k} = \his_{k} \cup \{R_t\}$

    \For{$k = 1, \dots, K$}{
    \If{$|\his_{k}| > 0$}{
    
    Sample the multiplier weights $\{\omega_l, \omega'_l, \omega''_l\}_{l=1}^{|\his_{k}|} \sim \rho(\omega)$. 
    
    Update the following quantities:
    \begin{itemize}
        \item $V_{t + 1} = V_{t} + \omega_{t} A_{t} A_{t}^{\top} + \lambda \omega_t' 0  I_d + \lambda \omega_t'' I_d $;
        \item $b_{t + 1} = b_{t} + A_t \big( \omega_t R_t  + \lambda \omega_t' 0 + \lambda \omega_t'' 1 \big) $;
        \item Refresh the parameter as $ \widehat{\boldsymbol\theta}_{t + 1} = V_{t + 1}^{-1} b_{t +1}$.
    \end{itemize}

    }
    }
    }

\caption{\name{MBE} for linear bandits.}\label{alg:MBTS-LB}
\end{algorithm}

In this section, we derive the form of \name{MBE} when applied to stochastic linear bandits. 
We focus on the setup where $\vx_t$ is empty and $A_t \in \mathbb{R}^p$ is a linear feature vector, and other setups of linear bandits can be formulated similarly. 
In this case, $f(\vx_t, A_t ; \boldsymbol\theta) = A_t^{\top} \boldsymbol\theta$ where the parameter vector is $\boldsymbol\theta \in \mathbb{R}^{p}$. 
Then, the weighted loss function is 
\[
    \sum_{t = 1}^T \omega_t \big( A_t^{\top} \boldsymbol\theta - R_t\big)^2 + \frac{\xi}{2} \| \boldsymbol\theta\|_2^2,
\]
where $\xi \geq 0$ is a penalty tuning parameter. 
The solution is the standard weighted ridge regression estimator and can be updated in the following way:
\begin{itemize}
    \item[0.] Initialization: $A_0 = \mathbf{0}$, $\widehat{\boldsymbol\theta}_0 = \mathbf{0}$ with $b_0 = \mathbf{0}$, and $V_0 = (\xi + 1 ) I_{{\operatorname{dim}(A_t)}}$.
    \item[1.] $\widehat{\boldsymbol\theta}_t = V_{t}^{-1} A_{t}$;
    \item[2.] $V_{t + 1} = V_{t} + \omega_{t} A_{t} A_{t}^{\top}$, $b_{t + 1} = b_{t} + \omega_t R_t A_t$, and hence update 
    \[
        \begin{aligned}
            \widehat{\boldsymbol\theta}_{t + 1} = V_{t + 1}^{-1} b_{t +1} & = \big(V_{t} + \omega_{t} A_{t} A_{t}^{\top} \big)^{-1} \big( b_{t} + \omega_t R_t A_t \big) \\
            & = V_t^{-1}  - V_t^{-1} A_{t} (\omega_{t}^{-1} + A_{t}^{\top} V_t^{-1}A_t)^{-1} A_t^{\top} V_t^{-1}
        \end{aligned}
    \]
    \item[3.] Take the action $A_{t + 1} = \argmax_{\mathbf{a} \in \mathcal{A}_t} \mathbf{a}^{\top} \boldsymbol\theta_{t + 1}$. 
\end{itemize}
The \name{MBE} algorithm for linear bandits is presented in Algorithm \ref{alg:MBTS-LB}.




\subsection{Naive Adaptation of the Multiplier Bootstrap}\label{sec:naive}
We present the naive multiplier bootstrap-based exploration algorithm in Algorithm \ref{alg:MBTS-naive}. 
Specifically, there is no pseudo-rewards added. 

\begin{algorithm}[!h]
\SetAlgoLined
\setcounter{AlgoLine}{1}
\KwData{
Function class $\mathcal{F}$, 
loss function $\mathcal{L}$, 
(optional) penalty function ${J}$, 
multiplier weight distribution $\rho(\omega)$, 
tuning parameter $\lambda$
}

Set $\his = \{\}$ be the history be the pseudo-history

Initialize $\widehat{f}$ in an optimistic way

\For{$t = 1, \dots, T$}{

    Observe context $\vx_t$ and action set $\mathcal{A}_t$

    Offer $A_t = \arg \max_{a \in \mathcal{A}_t} \widehat{f}(\vx_t, a)$ (break tie randomly)

    Observe reward $R_t$
    
    Update $\his = \his \cup \{(\vx_t, A_t, R_t)\}$
    
    Sample the multiplier weights $\{\omega_l\}_{l=1}^t \sim \rho(\omega)$
    
    Solve the weighted loss minimization problem to update $\widehat{f}$ as
    \begin{equation*}
    \widehat{f} = 
    \argmin_{f \in \mathcal{F}} 
    \frac{1}{t}
    \sum_{l=1}^{t}
    \omega_l \mathcal{L}\big( 
    f(\vx_l, A_l), R_l
    \big) + {J}(f).  
    \end{equation*}
}

\caption{A Naive Design of \name{MBE} }\label{alg:MBTS-naive}
\end{algorithm}


\subsection{ReBoot}\label{sec:ReBoot}
For completeness, we introduce the details of \name{ReBoot} in this section and discuss its generalizability. 
More details can be found in the original papers \citep{wang2020residual, wu2022residual}. 

Consider a stochastic bandit problem with a fixed and finite set of arm $\mathcal{A}$. 
Every arm $a \in \mathcal{A}$ may have a fixed feature vector (which with slight overload of notation, we also denote as $a$). 
The mean reward of arm $a$ is $f(a)$. 
At each round $t'$, \name{ReBoot} first fit the model $f$ as $\widehat{f}$ using all data.
Then for each arm $a$, \name{ReBoot} first computes the corresponding residuals using rewards related to that arm as $\{\epsilon_t = R_t - \widehat{f}(a)\}_{t:A_t = a, t \le t'}$, 
then perturbs these residuals with random weights as $\{\omega_t \epsilon_t = R_t - \widehat{f}(a)\}_{t:A_t = a, t \le t'}$ (\name{ReBoot} also adds pseudo-residuals, which we omit for ease of notations), 
and finally use $\widehat{f}(a) + |\{t:A_t = a, t \le t'\}|^{-1} \sum_{t:A_t = a, t \le t'} \omega_t \epsilon_t$ as the perturbed estimation of the mean reward of arm $a$.  
By design, it can be seen that \name{ReBoot} critically relies on the reward history of each \textit{fixed} arm. 
Therefore, to the best of our understanding, it is not easy to extend \name{ReBoot} to problems with either changing (e.g., contextual problems) or infinite arms.





\section{More Experiment Results and Details}\label{sec:appendix_experiments}
\subsection{Results for linear bandits}\label{sec:results_LB}

\begin{figure*}[!t]
    \includegraphics[width=\linewidth]{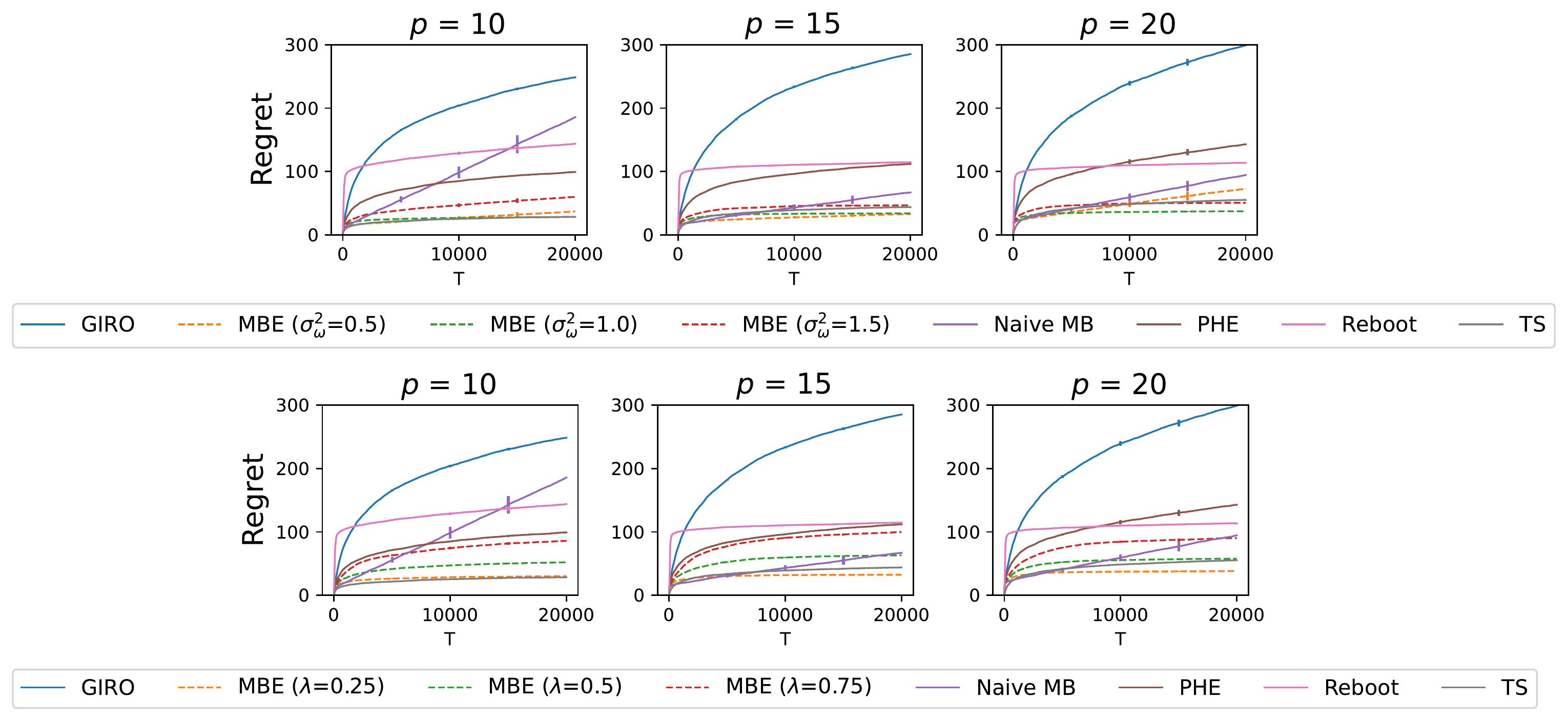} 
    \vspace{-.8cm}
    \caption{Performance of \name{MBE} in three linear bandit problems.
}
    \label{fig:LB}
\end{figure*}


We also consider the linear bandit problem. 
The linear bandit version of \name{MBE} is presented in Appendix \ref{sec:MBTS_LB}. 
We experiment with several dimensions $p = 10, 15, 20$. The number of arms is $K = 100$. 
The feature vector $x_k \in \mathbb{R}^p$ of arm $k$ is generated as follows. 
For the last 10 arms, the features are drawn uniformly at random from $(0, 1)$. 
For the first $90$ arms, we consider a practical setup where they are low-rank: we first generate a loading metric $A = (a_{ij}) \in \mathbb{R}^{p \times 5}$ from $\text{Uniform} (0, 1)$, then  sample $b  \in \mathbb{R}^{5}$ from $\text{Uniform} (0, 1)$, and finally constructs  
$x_{k} = A b$. 
The parameter vector $\theta \in \mathbb{R}^p$ is uniformly sampled from $[0, 1]^p$. 
We normalize the feature vectors such that the mean reward $\mu_k = x_k^{\top} \theta $ falls within the interval $[0, 1]$. 
The rewards of arm $k$ are drawn i.i.d. from $\text{Bernoulli}(x_k^{\top} \theta)$.



We still compare \name{MBE} with the method for linear bandit version of  \name{GIRO}, \name{PHE}, and \name{ReBoot} with tuning to their best performance over the hyper-parameter set $\{2^{k-4}\}_{k=0}^6$ and report the best performance of each method. 
For \name{TS}, we use Gaussian for both its reward and prior distribution, and calibrate their parameters using the true model. 
The total rounds are $T = 20000$ and our results are averaged over $50$ randomly chosen problems. 
Most other details are similar to our MAB experiments. 


We present the results in Figure \ref{fig:LB}, where we vary either $\sigma_{\omega}$ or $\lambda$ in the two subplots. 
We can see that \name{naive MB} leads to a linear regret. 
Hence, the pseudo-reward also matters in this problem. 
\name{MBE} achieves comparable performance with strong baselines such as \name{TS}. 
Another finding  is that \name{MBE} is robust to its tuning parameters. 
Finally, \name{Reboot} needs to pull $K$ times to initialize (the linear regret part in the first $K$ rounds) due to the nature of its design. 
In contrast, most other linear bandit algorithms typically only need $p$ rounds of forced exploration. 
This shows the limitation of the \name{ReBoot} framework (See Appendix \ref{sec:ReBoot}).




\subsection{Additional results}\label{sec:results_robust_lam}

In this section, we study the performance of \name{MBE} with respect to a few other hyper-parameters.

We first study the robustness to another tuning parameter $\lambda$. 
Recall that $\lambda$ controls the amount of external perturbation. 
Specifically, we repeat the experiment in \ref{sec:exp_MAB} with $\sigma_\omega^2$ fixed as 0.5 and with different values of $\lambda$ $(0.25, 0.5, 0.75)$.  
From Figure \ref{fig:MAB_lam}, it can be seen that a small amount of pseudo-rewards ($\lambda=0.25$) seems sufficient in these settings, and the results are fairly stable. 
We believe this is because the exploration of \name{MBE} is main driven by the internal randomness in the data.

\begin{figure*}[!t]
    \includegraphics[width=\linewidth]{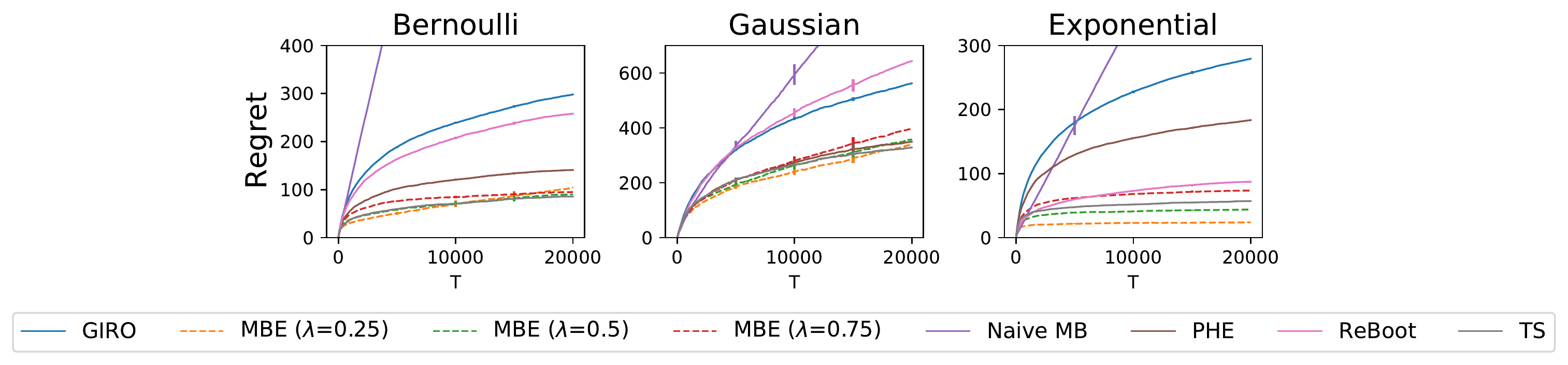}
    \vspace{-.5cm}
    \caption{Performance of \name{MBE} with different values of $\lambda$ in MAB. }
    \label{fig:MAB_lam}
\end{figure*}

In Figure \ref{fig:MAB_K}, we repeat our main experiments with $K$ changed to $25$. 
We can see that our main conclusions still hold.

\begin{figure*}[!t]
    \includegraphics[width=\linewidth]{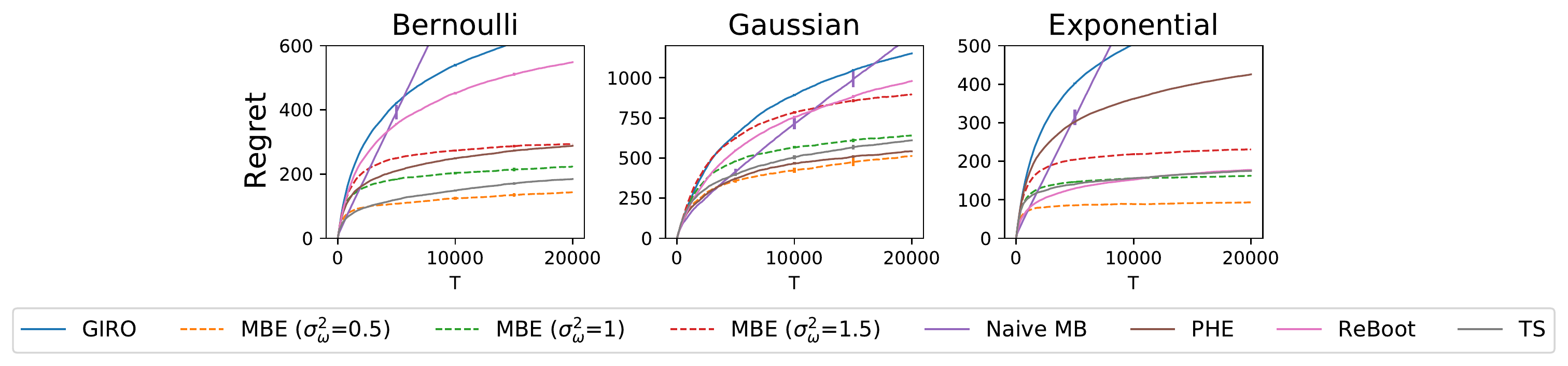}
    \vspace{-.5cm}
     \caption{Performance of \name{MBE} with $K = 25$. }
    \label{fig:MAB_K}
\end{figure*}

Finally, in Figure \ref{fig:MAB_B}, we implement \name{MBE} with different number of replicates $B$. 
As expected, more replicates does help exploration due to a better approximation to the whole bootstrapping distribution. 
Yet, we find that $B=50$ suffices to generate comparable performance with \name{TS} and the performance of \name{MBE} becomes relatively stable for larger values of $B$. 

\begin{figure*}[!t]
    \includegraphics[width=0.9\linewidth]{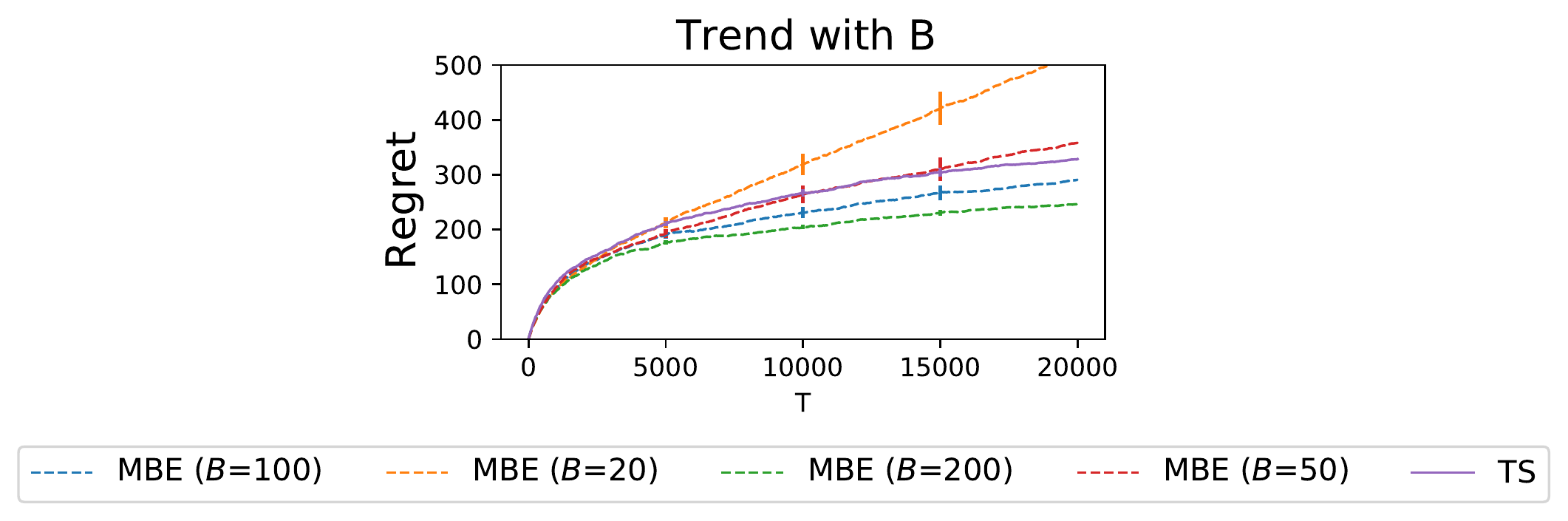}
    \vspace{-.5cm}
     \caption{Performance of \name{MBE} with different number of replicates $B$. }
    \label{fig:MAB_B}
\end{figure*}

\subsection{Details of the real data experiments}\label{sec:exp_details}

Our real data experiments closely follow \citet{wan2022towards}. 
For completeness, we provide information of the three problems here, and refer interested readers to \citet{wan2022towards} and references therein.

In an online learning to rank problem, we aim to select and rank $K$ items from a pool of $L$ ones. 
We itreatively interacts with users to learn about their preferences. 
The cascading model is popular in learning to rank \citep{kveton2015cascading}, which models the user behaviour as glancing from top to bottom (like a cascade) and choose to click an item following a Bernoulli distribution when she looks at that item. 
Therefore, we will binary outcomes for all items that the user has examined, and there are complex dependency between them. 

In a slate optimization (or called assortment optimization) problem, we aim to offer $K$ items from a pool of $L$ ones, especially when there exist substitution effects. 
The Multinomial Logit (MNL) model characterizes the choice behaviour as a multinomial distribution based on the attractiveness of each item. 
Since the offer subset changes over rounds, the joint likelihood is actually complex. 
To get the posterior or confidence bounds, one has to resort to an epoch-type offering schedule \citep{agrawal2017thompson}. 

Online combinatorial optimization also has numerous  applications \citep{wen2015efficient}, including maximum weighted matching, ads allocation, webpage optimization, etc. 
It is common that every chosen item will generate a separate observation, known as the semi-bandit problem. 
We consider a special problem in our experiments, where we need to choose $K$ persons from a pool under constraints.

The three datasets we used (and related problem setups) are studied in corresponding TS papers in the literature. 
To general random rewards, we need to either generate from a real data-calibrated model or  by directly sampling from the dataset. 
We follow  \citet{wan2022towards} and references therein. 
For cascading or MNL bandits, we split the dataset into a training and a testing set, use the training to estimate the reward model, and compare on the testing set. 
For semi-bandits, we sample rewards from the dataset.

\section{Main Proof}\label{proof_thm}

This section gives the proof of the main regret bound (Theorem \ref{thm.MAB.1}). 
Section \ref{sec_lem_components} gives the major lemmas required to bound the regret components used in this section. 
Section \ref{sec_lem_tech} lists all supporting technical lemmas, including the lower bound of the Gaussian tail and some novel results on the concentration property of sub-Gaussian and sub-exponential distributions.

We will present a complete version of the theory of \name{MBE} under MAB, Theorem \ref{thm.MAB}, as well as its proof, in which we allow arbitrary variance proxy $\sigma_k^2$ rather than requiring them to be one. 

\begin{theorem}\label{thm.MAB}
Consider a $K$-armed bandit, where the reward distribution of arm $k$ is $\subG(\sigma_k^2)$ with mean $\mu_k$. Suppose $\mu_1 = \max_{k \in [K]} \ \mu_k$ and $\Delta_k = \mu_1 - \mu_k$. 
Take the multiplier weight distribution as $\mathcal{N}(1, \sigma_{\omega}^2)$ in Algorithm \ref{alg:MBTS-MAB-2}. 
Let the tuning parameters satisfy $\lambda \geq \left( \frac{4 \sigma_1}{\sigma_{\omega}} + 1 \right) + \sqrt{\frac{4\sigma_1}{\sigma_{\omega}} \left( \frac{4 \sigma_1}{\sigma_{\omega}} + 1\right)}$,
Then the problem-dependent regret is upper bounded by
\[
     \operatorname{Reg}_T \leq \sum_{k = 2}^K \Delta_k \Bigg[ 7 + \left\{ C_1(\sigma_1, \sigma_k, \lambda, \sigma_{\omega}) + \frac{C_2(\sigma_1,  \sigma_k, \lambda, \sigma_{\omega})}{\Delta_k^{2}} \right\} \log T \Bigg],
\]
where 
\[
    C_1(\sigma_1, \sigma_k, \lambda, \sigma_{\omega}) = 10 \bigg[ 8 \sqrt{2} \max_{k \in [K]} \sigma_{k}^2 \left( \frac{\log D_2 (\sigma_1, \sigma_{k}, \lambda, \sigma_{\omega})}{3 \log 2} + 1\right) + 38  \sigma_{\omega}^2 \bigg],
\]
and
\[
    \begin{aligned}
        C_2(\sigma_1, \sigma_k, \lambda, \sigma_{\omega}) & = 50 \lambda^2 \sigma_k^2 + 10 \bigg[ D_1 (\sigma_1, \sigma_{k}, \lambda, \sigma_{\omega}) \left( \frac{\log D_2 (\sigma_1, \sigma_{k}, \lambda, \sigma_{\omega})}{3 \log 2} + 1\right) + 38  \sigma_{\omega}^2 \bigg],
    \end{aligned}
\]
with 
\[
    \begin{aligned}
        D_1(\sigma_1, \sigma_k, \lambda, \sigma_{\omega}) & = \bigg[ \big(1 +  8 \sqrt{2} \max_{k \in [K] } \sigma_k^2 \big) \bigg( 16 + \frac{\sigma_{\omega}^2}{\sigma_1^2} \bigg) + 16 \sigma_1^4 + 3 \frac{\sigma_{\omega}^2}{\sigma_1^2} + 3 \sigma_{\omega}^2 + 1 \bigg] \sigma_{\omega}^2 \lambda^4,
    \end{aligned}
\]
and
\[
    \begin{aligned}
        D_2(\sigma_1, \sigma_k, \lambda, \sigma_{\omega}) & = 3 \bigg[ 1 +  \frac{3 \sqrt{\pi}\sigma_{\omega}^2}{2 \sigma_1} \left( \frac{\sigma_1}{\sigma_{\omega}} + 3 \lambda \right) + 16 \sqrt{\pi} \max_{k \in [K]} \sigma_{k}^2 \left( \frac{\sigma_{\omega}^2}{16 \sigma_1^4} + \frac{1}{\sigma_{\omega}^2} \right) + \frac{\lambda^2 \sigma_{\omega}^2}{4 \sigma_1^4} \bigg].
    \end{aligned}
\]
Furthermore, the problem-independent regret is upper bounded by

\[
    \operatorname{Reg}_T  \leq 7 K \mu_1 + \max_{k \in [K] \setminus \{ 1 \}}  C_1(\sigma_1, \sigma_k, \lambda) K \log T + 2 \sqrt{\max_{k \in [K] \setminus \{ 1 \}}  C_2 (\sigma_1, \sigma_k, \lambda) K T \log T}.
\] 
\end{theorem}

Before starting our proof, we first give the definition of sub-exponential variables: we call a mean-zero variable $X$ sub-exponential with parameters $\lambda$ and $\alpha$ if
\[
    \E \exp \{ t X \} \leq \exp \left\{ \frac{t^2 \lambda^2}{2}\right\}, \qquad |t| \leq \frac{1}{\alpha}.
\]
And we denote as $X \sim \subE(\lambda, \alpha)$ if sub-exponential $X$ has parameters $\lambda$ and $\alpha$.
We have the following concentration {property}
\[
    \pr (X \geq t) \leq \exp \left\{ - \frac{1}{2} \left( \frac{t^2}{\lambda^2} \wedge \frac{t}{\alpha}\right) \right\}, \qquad  {\forall} \, t \geq 0,
\]
where $a \wedge b$ is defined as $\min \{ a, b\}$. See \citet{zhang2020concentration} and \citet{zhang2022sharper} for more details. 
For simplicity, 
we denote $\subE(\lambda) := \subE(\lambda, \lambda)$. 
We call a random variable $X$ sub-Gaussian with variance proxy $\sigma^2$ if $\E \exp \{ t (X - \E X)\} \leq \exp \{ {t^2 \sigma^2} / {2}\}$ for any $t \in \mathbb{R}$, and denote it as $X\sim \subG(\sigma^2)$.

\textbf{Notations:} \textit{Denote $\pr_{\xi}(A) = \int_{A} \mathrm{d} F_{\xi}(x)$ as the probability of event $A$, where $F_{\xi}(x)$ is the distribution function of the random variable $\xi$, and similarly, we denote $\E_{\xi} f(\xi) = \int f(x) \, \mathrm{d} F_{\xi} (x)$ as the expectation. Recall that we write two functions $a(s, T) \lesssim b(s, T)$ if $a(s, T) \leq c \times b(s, T)$ for some constant $c$ free of $s$ and $T$. And we write $a(s, T) \asymp b(s, T)$ if both $a(s, T) \lesssim b(s, T)$ and $a(s, T) \gtrsim b(s, T)$. Furthermore, we define $a \vee b = \max\{a, b\}$ and $a \wedge b = \min \{ a, b\}$ for any real numbers $a$ and $b$. Similarly, we can define $a \vee b \vee c = \max\{ a, b, c\}$ and $a \wedge b \wedge c = \min\{ a, b, c\}$ for any $a, b, c \in \mathbb{R}$.}

\begin{proof}
\vspace{1ex}

\textbf{Step 0: Decomposition of the regret bound.} Our proof relies on the following decomposition of the cumulative regret \citep{kveton2019garbage}. 
We denote the first $s$ rewards from pulling arm $k$ as  $\his_{k, s}$, with the $i$-th observations denoted as $R_{k, i}$. 
Let $Q_{k, s} (\tau) = \pr \big( \overline{Y}_{k, s} > \tau \mid \his_{k, s} \big)$ be the tail probability that $\overline{Y}_{k, s}$ conditioned on history $\his_{k, s}$ is a least $\tau$, and $N_{k, s} (\tau) = 1 / Q_{k, s} (\tau) - 1$ is the expected number of rounds that the arm $k$ being underestimated given $s$ sample rewards. Here
\[
    \overline{Y}_{k, s} = \frac{\sum_{i = 1}^s \left[ \omega_i R_{k, i} + \omega_i' \cdot (1 \times \lambda) + \omega_i'' \cdot (0 \times \lambda) \right]}{\sum_{i = 1}^s (\omega_i + \lambda \omega_i' + \lambda \omega_i'')}
\]
is the objective function defined in Algorithm \ref{alg:MBTS-MAB-2}.

\begin{lemma}\label{lem.a.b}
Suppose in MAB we select arms according to the rule 
$A_t = \arg \max_{k \in [K]} \overline{Y}_{k, t}$ with $\overline{Y}_{k, t}$ defined in Algorithm \ref{alg:MBTS-MAB-2}.
Then for any $\{ \tau_k \}_{k = 2}^K \subseteq \mathbb{R}$, 
the expected $T$-round regret can be bounded above as
\[
    \operatorname{Reg}_T \leq \sum_{k = 2}^K \Delta_k \big( a_k + b_k \big),
\]
where $a_k = \sum_{s = 0}^{T - 1} a_{k, s}$ and $ b_k = \sum_{s = 0}^{T - 1} b_{k, s} + 1$, and $a_{k, s} = \E \big[N_{1, s}(\tau_k) \wedge T\big] $ and $b_k = \pr \big( Q_{k, s} (\tau_k) > T^{-1} \big)$.
\end{lemma}
Recall the summation index $s$ is the number of times we pull the $k$-th arm. 
The definitions of $a_{k}$ and $b_{k}$ have important meanings: $a_{k}$ represents the expected number of rounds that optimal arm $1$ has been being underestimated, whereas $b_{k}$ is the probability that the suboptimal arm $k$ is being overestimated.
Here we only need to consider the lower bound of the tail of the distribution of the rewards from the optimal arm. 
The intuition is: (i) we only need the rewards from the optimal arm taking a relatively large value with a probability that is not too small; (ii) we do not care about the tiny probability of a large reward from suboptimal arms. 

Therefore, our target is then to bound $a_k$ and $b_k$ for any $k \geq 2$. 
These are completed in Step 1 and Step 2 below, respectively. 


\vspace{2ex}


\textbf{Step 1: Bounding $a_k$.} 

We first provide a roadmap 
for proving $a_k$ is bounded by a term of $O(\log T)$ order:
For a given constant level $\tau_k$, the probability of the optimal arm 1 being underestimated given $s$ rewards is $1-$ $Q_{1, s}\left(\tau_k\right)$. If we pick the level to satisfy $\tau_k < \frac{\mu_1 + \lambda}{1 + 2 \lambda}$, the theory of large deviation gives
\[
    \lim_{s \rightarrow \infty} Q_{1, s} (\tau_k) = 1 .
\]
Hence the expected number of rounds to observe a not-underestimated case $N_{1, s}(\tau_k) =\frac{1}{Q_{1, s}(\tau_k)}-1$ has the property $\lim_{s \rightarrow \infty} N_s (\tau_k) = 1$ as the number of pulls $s$ grows to infinity. Thus, given the round $T$, there exists a constant $s_0(T)$ such that $N_s\left(\tau_k\right) \leq T$ for all $s$ over $s_0(T)$. Consequently, the quantity $a_k$ in regret bound will be bounded by
\[
    a_k \leq \sum_{s=0}^{s_0(T)} \E \left[ N_{1, s} (\tau_k)\wedge T\right] .
\]
The fact that the constant $s_0(T)$ is at the order of $\log T$ will be shown in Lemma \ref{lem_bounding_a_k_s_1}, Lemma \ref{lem_bounding_a_k_s_2}, and Lemma \ref{lem_bounding_a_k_s_3}. For small number of pulls $s \leq s_0(T)$, we show in Lemma \ref{lem_bounding_a_k_s_constant} that $\E \left[ N_{1, s} (\tau_k)\wedge T\right] \leq 2 + 4 e^{9 / 8} $ for any $s \in \mathbb{N}$. Thus, it is enough to conclude that $a_k$ can be bounded by a term of $O(\log T)$ order.

To formally bound $a_k$ in the non-asymptotic sense following the intuition above, we need to decompose $a_k$.
For the decomposition, a common approach is to use indicators on \textit{good} events. Denote the shifted (sample-) mean reward as 
\[
    \overline{R}_{k, s}^* = \frac{\sum_{i = 1}^s R_{k, i} + \lambda s}{s(1 + 2 \lambda)} = \frac{\lambda}{1 + 2 \lambda} + \frac{1}{1 + 2 \lambda} \overline{R}_{k, s}.
\]
where $\overline{R}_{k, s} = \frac{1}{s} \sum_{i = s} R_{k, i}$ is the mean reward of the $k$-th arm. Then we can define the following \textit{good} events for $l$-th arm as
\[
    A_{l, s} = \left\{ - C_1 \Delta_k < \overline{R}_{l, s}^* - \frac{\mu_l + \lambda}{1 + 2 \lambda}  < C_1 \Delta_k \right\},
\]
and
\[
    G_{l, s} = \left\{ - C_2 \Delta_k \leq \overline{Y}_{l, s} - \overline{R}_{l, s}^* \leq C_2 \Delta_k \right\}.
\]

The definitions of $A_{l, s}$ and $G_{l, s}$ are intuitive: $A_{l, s}$ represents the sample mean does not derivative from the true mean too far, and events on $G_{l, s}$ means $\overline{Y}_{l, s}$ is not too far away from the scaled-shifted sample mean $\overline{R}_{l, s}^*$. Here $C_1, C_2 \in (0, 1)$ are some constants.

Therefore, by using these good events, we decompose $a_{k, s}$ as the following three parts:
\begin{equation}\label{def_a_k_s_1}
    a_{k, s, 1} = \E \Big[ \big( N_{1, s} (\tau_k) \wedge T \big) \ind (A_{1, s}^c)\Big],
\end{equation}
\begin{equation}\label{def_a_k_s_2}
    a_{k, s, 2} = \E \Big[ \big( N_{1, s} (\tau_k) \wedge T \big) \ind (A_{1, s}) \ind (G_{1, s}^c)\Big],
\end{equation}
and
\begin{equation}\label{def_a_k_s_3}
    a_{k, s, 3} = \E \Big[ \big( N_{1, s} (\tau_k) \wedge T \big) \ind (A_{1, s}) \ind (G_{1, s})\Big].
\end{equation}
Let $C_1 = \frac{1}{6 \lambda}$ and $C_2 = \frac{1}{12 \lambda}$ with fixed $\lambda > 1$, then $C_1, C_2 \in (0, 1)$. Denote
\[
    \mathfrak{a}_1 = \frac{192 \sigma_{\omega}^2 \lambda^4}{3 (1 + 2\lambda)^2 \Delta_k^2} , \qquad \mathfrak{b}_1 = \frac{2 \sigma_{\omega}^2 \left[ 96 \lambda^4 \sigma_1^2 + (1 + 2 \lambda)^2 C_2^2 \Delta_k^2 \right]}{3(1 + 2 \lambda)^2 \Delta_k^2},
\]
\[
    \mathfrak{a}_2 = \frac{36 \sigma_{\omega}^2 \lambda^4}{3 (\lambda - 1)^2 \Delta_k^2}, \qquad \mathfrak{b}_2 = \frac{\sigma_{\omega}^2\left[ 72 \lambda^4 \sigma_1^2 + 25 (1 + 2\lambda)^2 \Delta_k^2 \right]}{6 (\lambda - 1)^2  \Delta_k^2}, \qquad \Lambda_k = 8 \sqrt{2} \sigma_k^2.
\]
Consider the case $T \geq 2$. Define 
\[
    s_{a, j}(T) = \max \{ s: a_{k, s, j} \geq T^{-1}\}, \qquad j = 1, 2, 3.
\]
Lemma \ref{lem_bounding_a_k_s_1} in Appendix \ref{sec_lem_components} proves that by taking
\[
    s \geq s_{a, 1} (T) = \frac{144 \lambda^2 \sigma^2_1}{(1 + 2\lambda)^2 \Delta_k^2} \log T,
\]
we will have $a_{k, s, 1} \leq T^{-1}$.
Similarly, Lemma \ref{lem_bounding_a_k_s_2} and Lemma \ref{lem_bounding_a_k_s_3} in Appendix \ref{sec_lem_components} say that: if we take
\[
    \begin{aligned}
       s \geq s_{a, 2}(T) = \Bigg[ \big[ \Lambda_1 + \mathfrak{a}_1 + \mathfrak{b}_1 + \Lambda_1 \mathfrak{a}_1 \big] & \left(\frac{1}{3}\log^{-1}2 \times \log \left\{ 3 \left[  1 + \frac{\sqrt{\pi \mathfrak{b}_1}}{2} + \frac{\sqrt{2 \pi} \Lambda_1 \mathfrak{a}_1}{2 \mathfrak{b}_1^2} + \frac{2 \mathfrak{a}_1}{\mathfrak{b}_1} \right]  \right\}  + 1\right)\\
        &  \vee \frac{18 \sigma_{\omega}^2 (2 \mu_1 - 1)^2}{(\lambda^2 + \lambda + 1 / 4)^2} \vee \frac{2 \sigma_{\omega}^2 (1 + 2 \lambda^2)}{(1 + 2 \lambda)^2} \Bigg] \times 3 \log T
    \end{aligned}
\]
and
\[
    \begin{aligned}
        s \geq s_{a, 3} (T) =  \Bigg[ \big[ \Lambda_1 + \mathfrak{a}_2 + \mathfrak{b}_2 + \Lambda_1 \mathfrak{a}_2 \big] & \left(\frac{1}{3}\log^{-1}2 \times \log \left\{ 3 \left[  1 +\frac{\sqrt{\pi \mathfrak{b}_1}}{2}  + \frac{\sqrt{2 \pi} \Lambda_1 \mathfrak{a}_2}{2 \mathfrak{b}_2^2} + \frac{2 \mathfrak{a}_2}{\mathfrak{b}_2} \right]  \right\}  + 1\right) \\
        & \vee \frac{18 \sigma_{\omega}^2 (2 \mu_k - 1)^2}{(\lambda^2 + \lambda + 1 / 4)^2} \vee \frac{2 \sigma_{\omega}^2 (1 + 2 \lambda^2)}{(1 + 2 \lambda)^2} \Bigg] \times 3 \log T,
    \end{aligned}
\]
then $a_{k, s, 2} \leq T^{-1}$ and $a_{k, s, 3} \leq T^{-1}$, respectively.
Denote $\Lambda_{\max} = \max_{k \in [K]} \Lambda_k$ and $\Lambda_{\min} = \min_{k \in [K]} \Lambda_k$, then for any 
\[
    \begin{aligned}
        s & \geq s_{a} (T) = \frac{144 \lambda^2 \sigma_1^2}{(1 + 2 \lambda)^2 \Delta_k^2} \log T +  \Bigg[ \big[\Lambda_{\max} + ( \mathfrak{a}_1 + \mathfrak{a}_2) + (\mathfrak{b}_1 + \mathfrak{b}_2) + \Lambda_{\max}  (\mathfrak{a}_1 + \mathfrak{a}_2) \big]  \\
        & \qquad \qquad \qquad \qquad \left(\frac{1}{3}\log^{-1}2 \times \log \left\{ 3 \left[  1 + \frac{\sqrt{\pi}}{2} \left( {\sqrt{\mathfrak{b}_1}} + \sqrt{\mathfrak{b}_2} \right) + \frac{\sqrt{2 \pi} \Lambda_{\max}  }{2}  \left( \frac{\mathfrak{a}_1}{\mathfrak{b}_1^2} + \frac{\mathfrak{a}_2}{\mathfrak{b}_2^2} \right) + 2\left( \frac{\mathfrak{a}_1}{\mathfrak{b}_1} + \frac{\mathfrak{a}_2}{\mathfrak{b}_2} \right) \right]  \right\}  + 1\right) \\
        &  \qquad \qquad \qquad \qquad  \vee \frac{18 \sigma_{\omega}^2 (2 \mu_1 - 1)^2 \vee (2 \mu_k - 1)^2}{(\lambda^2 + \lambda + 1 / 4)^2} \vee \frac{2 \sigma_{\omega}^2 (1 + 2 \lambda^2)}{(1 + 2 \lambda)^2} \Bigg] \times 3 \log T,
    \end{aligned}
\]
we have $s \geq \max_{j = 1, 2, 3} s_{a, j} (T)$ due to that
\[
    s_a(T) = s_{a, 1} (T) + \max_{j =2 , 3} s_{a, j} (T)  \geq \max_{j = 1, 2, 3} s_{a, j} (T).
\]
Hence, for any $s \geq s_{a} (T)$, we have 
\[
    a_{k, s} = a_{k, s, 1} + a_{k, s, 2} + a_{k, s, 3} \leq 3 T^{-1}.
\]
Finally, Lemma \ref{lem_bounding_a_k_s_constant} in Appendix \ref{sec_lem_components} ensures if we take $\lambda \geq \frac{4 \sigma_1}{\sigma_{\omega}}  + \sqrt{\frac{4\sigma_1}{\sigma_{\omega}} \left( \frac{4 \sigma_1}{\sigma_{\omega}} + 1\right)}$, the component $a_{k, s} \leq 2 + 4 e^{9 / 8}$ for any $s \geq 0$. Therefore, by setting $\lambda \geq \left( \frac{4 \sigma_1}{\sigma_{\omega}} + 1 \right) + \sqrt{\frac{4\sigma_1}{\sigma_{\omega}} \left( \frac{4 \sigma_1}{\sigma_{\omega}} + 1\right)}$, we have 
\[
    \begin{aligned}
    	a_{k} & = \sum_{s = 0}^{T - 1} a_{k, s}\\
            & \leq \sum_{s < s_{a} (T)} \max_{s \in \{ 0,1, \ldots, T - 1\}} a_{k, s} + \sum_{s_{a} (T) \leq s < T } 3 T^{-1} \\
            & = \max_{s \in \{ 0,1, \ldots, T - 1\}} a_{k, s} \times s_{a} (T) + 3 T^{-1} \times \big(T -  s_{a} (T) \big) \\
            & \leq 2 (1 + 2 e^{9 / 8}) \Bigg\{ \frac{144 \lambda^2 \sigma_1^2}{(1 + 2 \lambda)^2 \Delta_k^2} \log T +   \bigg[ \big[ \Lambda_{\max} + ( \mathfrak{a}_1 + \mathfrak{a}_2) + (\mathfrak{b}_1 + \mathfrak{b}_2) + \Lambda_{\max}  (\mathfrak{a}_1 + \mathfrak{a}_2) \big]  \\
        & \qquad \qquad \qquad \qquad \left(\frac{1}{3}\log^{-1}2 \times \log \left\{ 3 \left[  1 + \frac{\sqrt{\pi}}{2} \left( {\sqrt{\mathfrak{b}_1}} + \sqrt{\mathfrak{b}_2} \right) + \frac{\sqrt{2 \pi} \Lambda_{\max} }{2}  \left( \frac{\mathfrak{a}_1}{\mathfrak{b}_1^2} + \frac{\mathfrak{a}_2}{\mathfrak{b}_2^2} \right) + 2\left( \frac{\mathfrak{a}_1}{\mathfrak{b}_1} + \frac{\mathfrak{a}_2}{\mathfrak{b}_2} \right) \right]  \right\}  + 1\right) \\
        &  \qquad \qquad \qquad \qquad  \vee \frac{18 \sigma_{\omega}^2 (2 \mu_1 - 1)^2 \vee (2 \mu_k - 1)^2 }{(\lambda^2 + \lambda + 1 / 4)^2} \vee \frac{2 \sigma_{\omega}^2 (1 + 2 \lambda^2)}{(1 + 2 \lambda)^2} \bigg]\Bigg\} \log T + 3
    \end{aligned}
\]
for any $T \geq 2$.

\textbf{Step 2: Bounding $b_k$.} 

Again, we will provide a roadmap 
for proving $a_k$ is bounded by a term of $O(\log T)$ order. Similar to Step 1, we set a fixed level $\tau_k$ such that $\tau_k > \frac{\mu_k + \lambda}{1 + 2\lambda}$, then the theory of large deviation gives
\[
    \lim_{s \rightarrow \infty} Q_{k, s} (\tau_k) = 0.
\]
Thus, given the time horizon $T$, there exists a constant $s_0(T)$ such that $Q_{k, s}(\tau_k) \leq T^{-1}$ for all $s$ over $s_0(T)$. As a result, the event $\left\{Q_{k, s} (\tau_k ) > T^{-1}\right\}$ is empty if the number of pulls $s$ is beyond $s_0(T)$. Consequently, 
\[
    b_k \leq \sum_{s=0}^{s_0(T)} \pr \left(Q_{k, s}\left(\tau_k\right)>T^{-1}\right).
\]
The fact that the constant $s_0(T)$ is of $O(\log T)$ order will be shown in Lemma \ref{lem_bounding_b_k_s_1}, Lemma \ref{lem_bounding_b_k_s_2}, and Lemma \ref{lem_bounding_b_k_s_3}. For small number of pull $s \leq s_0(T)$, we apply a trivial bound $\pr \left(Q_{k, s}\left(\tau_k\right)>T^{-1}\right) \leq 1$ that holds for any $s$. 
Therefore, it is sufficient to conclude that $b_k$ can be bounded by a term of $O(\log T)$ order.

Note that $b_{k, s}$ is naturally bounded by {the} constant $1$. Similar to Step 2, decompose $b_{k, s} = \pr \big( Q_{k, s} (\tau_k) > T^{-1}\big)$ as $b_k = b_{k, s, 1} + b_{k, s, 2} + b_{k, s, 3}$ with
\begin{equation}\label{def_b_k_s_1}
    b_{k, s, 1} = \E \big[ \ind ( Q_{k, s} (\tau_k) > T^{-1}) \ind (A_{k, s}^{c} )\big],
\end{equation}
\begin{equation}\label{def_b_k_s_2}
    b_{k, s, 2} = \E \big[ \ind ( Q_{k, s} (\tau_k) > T^{-1}) \ind (A_{k, s}) \ind (G_{k, s}^c)\big],
\end{equation}
and
\begin{equation}\label{def_b_k_s_3}
    b_{k, s, 3} = \E \big[ \ind ( Q_{k, s} (\tau_k) > T^{-1}) \ind (A_{k, s}) \ind (G_{k, s})\big],
\end{equation}
Again, define $s_{b, j}: = \max \{ s: b_{k, s, j} \geq T^{-1}\}$ for $j = 1, 2, 3$, then Lemma \ref{lem_bounding_b_k_s_1} in Appendix \ref{sec_lem_components} guarantees that take $s \geq s_{b, 1} (T)$ with
\[
    s_{b, 1} (T) = \frac{72 \lambda^2 \sigma^2_k}{(1 + 2\lambda)^2 \Delta_k^2} \log T,
\]
then $b_{k, s, 1} \leq T^{-1}$. Consider $T \geq 2$ and let $C_1 = \frac{1}{6 \lambda}$ and $C_2 = \frac{1}{12 \lambda}$ with fixed $\lambda > 1$ as in Step 1, 
Lemma \ref{lem_bounding_b_k_s_2} and Lemma \ref{lem_bounding_b_k_s_3} in Appendix \ref{sec_lem_components} proves that:  if we take
\[
    \begin{aligned}
        s \geq s_{b, 2}(T) = \Bigg[ \big[ \Lambda_k + \mathfrak{a}_1 + \mathfrak{b}_1 + \Lambda_k \mathfrak{a}_1 \big] & \left(\frac{1}{3}\log^{-1}2 \times \log \left\{ 3 \left[  1 + \frac{\mathfrak{b}_1}{4 \Lambda_k \mathfrak{a}_1} + \frac{\sqrt{2 \pi} \Lambda_k \mathfrak{a}_1}{2 \mathfrak{b}_1^2} + \frac{2 \mathfrak{a}_1}{\mathfrak{b}_1} \right]  \right\}  + 1\right)\\
        &  \vee \frac{18 \sigma_{\omega}^2 (2 \mu_1 - 1)^2}{(\lambda^2 + \lambda + 1 / 4)^2} \vee \frac{2 \sigma_{\omega}^2 (1 + 2 \lambda^2)}{(1 + 2 \lambda)^2} \Bigg] \times 3 \log T
    \end{aligned}
\]
and
\[
    \begin{aligned}
        s \geq s_{b, 3} (T) =  \Bigg[ \big[ \Lambda_k + \mathfrak{a}_2 + \mathfrak{b}_2 + \Lambda_k \mathfrak{a}_2 \big] & \left(\frac{1}{3}\log^{-1}2 \times \log \left\{ 3 \left[  1 + \frac{\mathfrak{b}_2}{4 \Lambda_k \mathfrak{a}_2} + \frac{\sqrt{2 \pi} \Lambda_k \mathfrak{a}_2}{2 \mathfrak{b}_2^2} + \frac{2 \mathfrak{a}_2}{\mathfrak{b}_2} \right]  \right\}  + 1\right) \\
        & \vee \frac{18 \sigma_{\omega}^2 (2 \mu_k - 1)^2}{(\lambda^2 + \lambda + 1 / 4)^2} \vee \frac{2 \sigma_{\omega}^2 (1 + 2 \lambda^2)}{(1 + 2 \lambda)^2} \Bigg] \times 3 \log T,
    \end{aligned}
\]
then we have $b_{k, s, 2} \leq T^{-1}$ and $b_{k, s, 3} \leq T^{-1}$, respectively. Let 
\[
    \begin{aligned}
        s_{b}(T) & := \frac{72 \lambda^2 \sigma_k^2}{(1 + 2 \lambda)^2 \Delta_k^2} \log T  +  \Bigg[ \big[ \Lambda_{\max} + ( \mathfrak{a}_1 + \mathfrak{a}_2) + (\mathfrak{b}_1 + \mathfrak{b}_2) + \Lambda_{\max} (\mathfrak{a}_1 + \mathfrak{a}_2) \big]  \\
        & \qquad \qquad  \left(\frac{1}{3}\log^{-1}2 \times \log \left\{ 3 \left[  1 + \frac{\sqrt{\pi}}{2} \left( {\sqrt{\mathfrak{b}_1}} + \sqrt{\mathfrak{b}_2} \right) + \frac{\sqrt{2 \pi} \Lambda_{\max} }{2}  \left( \frac{\mathfrak{a}_1}{\mathfrak{b}_1^2} + \frac{\mathfrak{a}_2}{\mathfrak{b}_2^2} \right) + 2\left( \frac{\mathfrak{a}_1}{\mathfrak{b}_1} + \frac{\mathfrak{a}_2}{\mathfrak{b}_2} \right) \right]  \right\}  + 1\right) \\
        &  \qquad \qquad \qquad \qquad  \vee \frac{18 \sigma_{\omega}^2 (2 \mu_1 - 1)^2 \vee (2 \mu_k - 1)^2}{(\lambda^2 + \lambda + 1 / 4)^2} \vee \frac{2 \sigma_{\omega}^2 (1 + 2 \lambda^2)}{(1 + 2 \lambda)^2} \Bigg] \times 3 \log T,
    \end{aligned}
\]
then for any $s \geq s_b(T)$, we have $s \geq \max_{j = 1, 2, 3} s_{b, j} (T)$ due to $s_b(T) = s_{b, 1} (T) + \max_{j = 2, 3} s_{b, j} (T)$. Thus, $b_{k, s, 1} + b_{k, s, 2} + b_{k, s, 3} \leq 3 T^{-1}$ for any $s \geq s_b(T)$.
Note that $b_{k, s} = \pr \big( Q_{k, s} (\tau_k) > T^{-1}\big) \leq 1$ for any $s \geq 0$, then 
\[
    \begin{aligned}
        b_{k} & = 1 + \sum_{s = 0}^{T - 1} b_{k, s, 1}\\
        & \leq 1 + 1 \times s_{b}(T) + 3 T^{-1} \times \big(T - s_{b}(T) \big) \\
        & \leq \Bigg\{ \frac{72 \lambda^2 \sigma_k^2}{(1 + 2 \lambda)^2 \Delta_k^2} + \bigg[ \big[ \Lambda_{\max} + ( \mathfrak{a}_1 + \mathfrak{a}_2) + (\mathfrak{b}_1 + \mathfrak{b}_2) + \Lambda_{\max} (\mathfrak{a}_1 + \mathfrak{a}_2) \big]  \\
        & \qquad \qquad  \left(\frac{1}{3}\log^{-1}2 \times \log \left\{ 3 \left[  1 + \frac{\sqrt{\pi}}{2} \left( {\sqrt{\mathfrak{b}_1}} + \sqrt{\mathfrak{b}_2} \right) + \frac{\sqrt{2 \pi} \Lambda_{\max} }{2}  \left( \frac{\mathfrak{a}_1}{\mathfrak{b}_1^2} + \frac{\mathfrak{a}_2}{\mathfrak{b}_2^2} \right) + 2\left( \frac{\mathfrak{a}_1}{\mathfrak{b}_1} + \frac{\mathfrak{a}_2}{\mathfrak{b}_2} \right) \right]  \right\}  + 1\right) \\
        &  \qquad \qquad \qquad \qquad  \vee \frac{18 \sigma_{\omega}^2 (2 \mu_1 - 1)^2 \vee (2 \mu_k - 1)^2}{(\lambda^2 + \lambda + 1 / 4)^2} \vee \frac{2 \sigma_{\omega}^2 (1 + 2 \lambda^2)}{(1 + 2 \lambda)^2} \bigg] \Bigg\} \log T + 4
    \end{aligned}
\]
for any $T \geq 2$.

\textbf{Step 3: Aggregating results.}

Denote 
\[
    d_1 = \Delta_k^2 \big\{ ( \mathfrak{a}_1 + \mathfrak{a}_2) + (\mathfrak{b}_1 + \mathfrak{b}_2) + \Lambda_{\max}  (\mathfrak{a}_1 + \mathfrak{a}_2) \big\}
\]
and
\[
    d_2 = 3 \left[  1 + \frac{\sqrt{\pi}}{2} \left( {\sqrt{\mathfrak{b}_1}} + \sqrt{\mathfrak{b}_2} \right) + \frac{\sqrt{2 \pi} \Lambda_{\max} }{2}  \left( \frac{\mathfrak{a}_1}{\mathfrak{b}_1^2} + \frac{\mathfrak{a}_2}{\mathfrak{b}_2^2} \right) + 2\left( \frac{\mathfrak{a}_1}{\mathfrak{b}_1} + \frac{\mathfrak{a}_2}{\mathfrak{b}_2} \right) \right],
\]
then the bounds in Step 1 and Step 2 imply
\[
    \begin{aligned}
        a_k  \leq 2 (1 + 2 e^{9 / 8}) & \Bigg\{ \frac{144 \lambda \sigma_1^2}{(1 + 2 \lambda)^2 \Delta_k^2} \\
        & + \bigg[ ( \Lambda_{\max} + d_1) \bigg( \frac{1}{3}\frac{\log  d_2}{\log 2} + 1 \bigg) \vee \frac{18 \sigma_{\omega}^2 (2 \mu_1 - 1)^2 \vee (2 \mu_k - 1)^2}{(\lambda^2 + \lambda + 1 / 4)^2} \vee \frac{2 \sigma_{\omega}^2 (1 + 2 \lambda^2)}{(1 + 2 \lambda)^2} \Bigg]  \Bigg\} \log T + 3.
    \end{aligned}
\]
and
\[
    \begin{aligned}
        b_k & \leq \Bigg\{ \frac{72 \lambda^2 \sigma_k^2}{(1 + 2 \lambda)^2 \Delta_k^2} \\
        & \qquad \qquad \qquad + \bigg[ ( \Lambda_{\max} + d_1) \bigg( \frac{1}{3}\frac{\log d_2}{\log 2} + 1 \bigg) \vee \frac{18 \sigma_{\omega}^2 (2 \mu_1 - 1)^2 \vee (2 \mu_k - 1)^2}{(\lambda^2 + \lambda + 1 / 4)^2} \vee \frac{2 \sigma_{\omega}^2 (1 + 2 \lambda^2)}{(1 + 2 \lambda)^2} \Bigg]  \Bigg\} \log T + 4
    \end{aligned}
\]
for any $T \geq 2$. Therefore, by applying Lemma \ref{lem.a.b}, we have 
\begin{equation}\label{reg_T_med}
    \begin{aligned}
        \operatorname{Reg}_T & \leq \sum_{k = 2}^K \Delta_k \big( a_k + b_k \big) \\
        &\leq \sum_{k = 2}^K \Delta_k \bigg[ 7 + \left\{ c_1(\mu_1, \sigma_1, \mu_k, \sigma_k, \lambda, \sigma_{\omega}) + c_2(\mu_1, \sigma_1, \mu_k, \sigma_k, \lambda, \sigma_{\omega})\Delta_k^{-2} \right\} \log T \bigg]
    \end{aligned}
\end{equation}
for any $T \geq 2$, where
\[
    c_1(\mu_1, \sigma_1, \mu_k, \sigma_k, \lambda) = (3 + 2 e^{9 / 8}) \bigg[ \bigg( \Lambda_{\max} \left( \frac{\log d_2}{3 \log 2} + 1\right) \bigg) \vee \frac{18 \sigma_{\omega}^2 \{ (2 \mu_1 - 1)^2 \vee (2 \mu_k - 1)^2 \}}{(\lambda^2 + \lambda + 1 / 4)^2} \vee \frac{2 \sigma_{\omega}^2 (1 + 2 \lambda^2)}{(1 + 2 \lambda)^2}  \bigg],
\]
and
\[
    \begin{aligned}
        c_2(\mu_1, \sigma_1, \mu_k, \sigma_k, \lambda) & = \frac{72 \lambda^2 \sigma_k^2}{(1 + 2 \lambda)^2 } ( 5 + 4 e^{9 / 8}) \\
        & \qquad + (3 + 2 e^{9 / 8}) \bigg[ \bigg( d_1 \left( \frac{\log d_2}{3 \log 2} + 1\right) \bigg) \vee \frac{18 \sigma_{\omega}^2 \{ (2 \mu_1 - 1)^2 \vee (2 \mu_k - 1)^2 \} }{(\lambda^2 + \lambda + 1 / 4)^2} \vee \frac{2 \sigma_{\omega}^2 (1 + 2 \lambda^2)}{(1 + 2 \lambda)^2}  \bigg],
    \end{aligned}
\]
for $k = 2, \ldots, K$. When the total round $T = 1$, the bound \eqref{reg_T_med} still holds because $\operatorname{Reg}_T \leq \max_{k = 2, \ldots K} \Delta_k \leq 7 \sum_{k = 2}^K \Delta_k$. 
\vspace{2ex}

\textbf{Step 4: Simplifying bounds.}

Note that
\[
    \begin{aligned}
        d_1 = & \Delta_k^2 \big\{ ( \mathfrak{a}_1 + \mathfrak{a}_2) + (\mathfrak{b}_1 + \mathfrak{b}_2) + \Lambda_{\max}  (\mathfrak{a}_1 + \mathfrak{a}_2) \big\} \\
        \alignedoverset{\text{by }\Delta_k \leq 1}{\leq} \frac{1 + \Lambda_{\max}}{3} \sigma_{\omega}^2 \left[ \frac{192 \lambda^4}{(1 + 2 \lambda)^2} + \frac{36 \lambda^4}{(\lambda - 1)^2}\right]  \\
        & \qquad \qquad \qquad + \sigma_{\omega}^2 \left[ \frac{2 \left[96 \lambda^4 \sigma_1^2 + (1 + 2\lambda)^2 / (36 \lambda^2) \right]}{3 (1 + 2\lambda)^2} + \frac{72 \lambda^4 \sigma_1^2 + 25 (1 + 2\lambda)^2}{6 (\lambda - 1)^2}\right].
    \end{aligned}
\]
By $\lambda \geq \left( \frac{4 \sigma_1}{\sigma_{\omega}} + 1 \right) + \sqrt{\frac{4\sigma_1}{\sigma_{\omega}} \left( \frac{4 \sigma_1}{\sigma_{\omega}} + 1\right)}$, we have $(\lambda - 1)^2 \geq 16 \sigma_1^2 / \sigma_{\omega}^2$, and hence
\[
    \frac{192 \lambda^4}{(1 + 2 \lambda)^2} + \frac{36 \lambda^4}{(\lambda - 1)^2} \leq \frac{192 \lambda^4}{4 \lambda^2} + \frac{36 \lambda^4}{16 \sigma_1^2 / \sigma_{\omega}^2 } = 48 \lambda^2 + 3 \frac{\sigma_{\omega}^2}{\sigma_1^2}  \lambda^4.
\]
Similarly,
\[
     \begin{aligned}
         & \frac{2 \left[96 \lambda^4 \sigma_1^2 + (1 + 2\lambda)^2 / (36 \lambda^2) \right]}{3 (1 + 2\lambda)^2} + \frac{72 \lambda^4 \sigma_1^2 + 25 (1 + 2\lambda)^2}{6 (\lambda - 1)^2} \\
         \leq & \frac{2 \times 96 \lambda^4 \sigma_1^2}{3 \times 4 \lambda^2} + \frac{2 }{3 \times 36 \lambda^2} + \frac{72 \lambda^4 \sigma_1^2 + 25 (1 + 2\lambda)^2}{6 \times 16 \sigma_1^2 / \sigma_{\omega}^2 } \\
         \leq & 16 \lambda^2 \sigma_1^2 + \frac{1}{48} + \lambda^4 \sigma_{\omega}^2 + \frac{25 \times 9 \lambda^2}{6 \times 16 \sigma_1^2 / \sigma_{\omega}^2} \leq 16 \lambda^2 \sigma_1^2 + \frac{1}{48} + 3 \lambda^4 \sigma_{\omega}^2 + 3 \lambda^2 \sigma_{\omega}^2 / \sigma_1^2.
     \end{aligned}
\]
Thus, 
\begin{equation}\label{final_bound_1}
    \begin{aligned}
        d_1 = & \Delta_k^2 \big\{ ( \mathfrak{a}_1 + \mathfrak{a}_2) + (\mathfrak{b}_1 + \mathfrak{b}_2) + \Lambda_{\max}  (\mathfrak{a}_1 + \mathfrak{a}_2) \big\} \\
        \leq & \frac{(1 + \Lambda_{\max}) \sigma_{\omega}^2}{3} \left[ 48 \lambda^2 + 3 \frac{\sigma_{\omega}^2}{\sigma_1^2}  \lambda^4 \right] + \sigma_{\omega}^2 \left[ 16 \lambda^2 \sigma_1^2 + \frac{1}{48} + 3 \lambda^4 \sigma_{\omega}^2 + 3 \frac{ \sigma_{\omega}^2}{\sigma_1^2} \lambda^2 \right] \\
        = & \sigma_{\omega}^2 \bigg\{ \bigg[ 16 (1 + \Lambda_{\max}) + 16 \sigma_1^4 + 3 \frac{\sigma_{\omega}^2}{\sigma_1^2}  \bigg] \lambda^2 + \bigg[ (1 + \Lambda_{\max}) \frac{\sigma_{\omega}^2}{\sigma_1^2}  + 3 \sigma_{\omega}^2\bigg] \lambda^4 +\frac{1}{48} \bigg\}  \\
        \leq & \sigma_{\omega}^2 \bigg[ 16 (1 + \Lambda_{\max}) + 16 \sigma_1^4 + 3 \frac{\sigma_{\omega}^2}{\sigma_1^2}  + (1 + \Lambda_{\max}) \frac{\sigma_{\omega}^2}{\sigma_1^2}  + 3 \sigma_{\omega}^2 + 1 \bigg] \lambda^4 \\
        \leq &  \bigg[ (1 + \Lambda_{\max}) \bigg( 16 + \frac{\sigma_{\omega}^2}{\sigma_1^2} \bigg) + 16 \sigma_1^4 + 3 \frac{\sigma_{\omega}^2}{\sigma_1^2} + 3 \sigma_{\omega}^2 + 1 \bigg] \sigma_{\omega}^2 \lambda^4.
    \end{aligned}
\end{equation}
Define
\[
    D_1(\sigma_1, \sigma_{k}, \lambda, \sigma_{\omega}) :=  \bigg[ (1 + \Lambda_{\max}) \bigg( 16 + \frac{\sigma_{\omega}^2}{\sigma_1^2} \bigg) + 16 \sigma_1^4 + 3 \frac{\sigma_{\omega}^2}{\sigma_1^2} + 3 \sigma_{\omega}^2 + 1 \bigg] \sigma_{\omega}^2 \lambda^4,
\]
then $d_1 \leq D_1(\sigma_1, \sigma_{k}, \lambda, \sigma_{\omega})$.
On the other hand, for $d_2$, we have 
\[
    \begin{aligned}
        d_2 = & 3 \left[  1 + \frac{\sqrt{\pi}}{2} \left( {\sqrt{\mathfrak{b}_1}} + \sqrt{\mathfrak{b}_2} \right) + \frac{\sqrt{2 \pi} \Lambda_{\max} }{2}  \left( \frac{\mathfrak{a}_1}{\mathfrak{b}_1^2} + \frac{\mathfrak{a}_2}{\mathfrak{b}_2^2} \right) + 2\left( \frac{\mathfrak{a}_1}{\mathfrak{b}_1} + \frac{\mathfrak{a}_2}{\mathfrak{b}_2} \right) \right] \\
        \leq & 3 \Bigg[ 1 +  \frac{\sqrt{\pi}}{2} \left( \sqrt{\frac{2 \sigma_{\omega}^2 (1 + 2 \lambda)^2  \Delta_k^2 / (144 \lambda^2) }{3 (1 + 2 \lambda)^2 \Delta_k^2}} + \sqrt{\frac{25 \sigma_{\omega}^2 (1 + 2\lambda)^2 \Delta_k^2}{6 (\lambda - 1)^2 \Delta_k^2}} \right)\\
        & \qquad \quad  +  \frac{\sqrt{2 \pi} \Lambda_{\max} }{2} \left( \frac{(1 + 2\lambda)^2}{64 \lambda^4 \sigma_1^4} + \frac{(\lambda - 1)^2}{12 \lambda^4 \sigma_1^2 \sigma_{\omega}^2} \right) \\
        & \qquad \quad + 2 \left( \frac{96 \lambda^4}{96 \lambda^4 \sigma_1^2 + (1 + 2 \lambda)^2 / (36 \lambda^2)} + \frac{72 \lambda^4}{72 \lambda^4 \sigma_1^2 + 25 (1 + 2\lambda)^2} \right)\ \Bigg].
    \end{aligned}
\]
Similarly, by the fact $\lambda > 1$,
\[
   \begin{aligned}
        \sqrt{\frac{2 \sigma_{\omega}^2 (1 + 2 \lambda)^2 \Delta_k^2 / (144 \lambda^2)}{3 (1 + 2 \lambda)^2 \Delta_k^2}} + \sqrt{\frac{25 \sigma_{\omega}^2 (1 + 2\lambda)^2 \Delta_k^2}{6 (\lambda - 1)^2 \Delta_k^2}} & \leq \frac{\sigma_{\omega}}{12 \lambda} + \frac{5 \sigma_{\omega} (1 + 2 \lambda)}{2(\lambda - 1)} \\
        & \leq \sigma_{\omega} + 3 \frac{\sigma_{\omega}^2}{\sigma_1} (1 + 2\lambda) \leq \frac{3\sigma_{\omega}^2}{\sigma_1} \left[ \frac{\sigma_1}{\sigma_{\omega}} + 3 \lambda \right]
   \end{aligned}
\]
\vspace{1ex}
\[
    \begin{aligned}
        \frac{(1 + 2\lambda)^2}{64 \lambda^4 \sigma_1^4} + \frac{(\lambda - 1)^2}{12 \lambda^4 \sigma_1^2 \sigma_{\omega}^2} & \leq \frac{9 \lambda^2}{64 \lambda^4 \sigma_1^4} + \frac{ 16 \sigma_1^2 / \sigma_{\omega}^2}{12 \lambda^4 \sigma_1^2 \sigma_{\omega}^2} \\
        & \leq  \frac{9 \lambda^2}{64 \lambda^4 \sigma_1^4} + \frac{ 16 \sigma_1^2 / \sigma_{\omega}^2}{12 \lambda^2  \sigma_1^2 \sigma_{\omega}^2} \\
        & \leq 2 \left[ \frac{1}{\sigma_1^2 \lambda^2} + \frac{1}{\sigma_{\omega}^2 \lambda^2}\right] \leq  2 \left[ \frac{\sigma_{\omega}^2}{16 \sigma_1^4} + \frac{1}{\sigma_{\omega}^2}\right],
    \end{aligned}
\]
and 
\[
    \begin{aligned}
        & \frac{96 \lambda^4}{96 \lambda^4 \sigma_1^2 + (1 + 2 \lambda)^2 / (36 \lambda^2)} + \frac{72 \lambda^4}{72 \lambda^4 \sigma_1^2 + 25 (1 + 2\lambda)^2} \\
        \leq &  \frac{96 \lambda^4}{96 \lambda^4 \sigma_1^2} + \frac{72 \lambda^4}{72 \lambda^4 \sigma_1^2} = \frac{2 \lambda^2}{\lambda^2 \sigma_1^2} \leq \frac{2 \lambda^2}{16 \sigma_1^4 / \sigma_{\omega}^2} = \frac{\lambda^2 \sigma_{\omega}^2}{8 \sigma_1^4}.
    \end{aligned}
\]
Thus,
\begin{equation}\label{final_bound_2}
    \begin{aligned}
        d_2 & = 3 \left[  1 +  \frac{\sqrt{\pi}}{2} \left( {\sqrt{\mathfrak{b}_1}} + \sqrt{\mathfrak{b}_2} \right) + \frac{\sqrt{2 \pi} \Lambda_{\max} }{2}  \left( \frac{\mathfrak{a}_1}{\mathfrak{b}_1^2} + \frac{\mathfrak{a}_2}{\mathfrak{b}_2^2} \right) + 2\left( \frac{\mathfrak{a}_1}{\mathfrak{b}_1} + \frac{\mathfrak{a}_2}{\mathfrak{b}_2} \right) \right] \\
        & \leq 3 \bigg[ 1 + \frac{3 \sqrt{\pi}\sigma_{\omega}^2}{2 \sigma_1} \left( \frac{\sigma_1}{\sigma_{\omega}} + 3 \lambda \right) + \sqrt{2 \pi} \Lambda_{\max} \left( \frac{\sigma_{\omega}^2}{16 \sigma_1^4} + \frac{1}{\sigma_{\omega}^2} \right) + \frac{\lambda^2 \sigma_{\omega}^2}{4 \sigma_1^4} \bigg].
    \end{aligned}
\end{equation}
Again, we define 
\[
    D_2 (\sigma_1, \sigma_{k}, \lambda, \sigma_{\omega}) = 3 \bigg[ 1 +  \frac{3 \sqrt{\pi}\sigma_{\omega}^2}{2 \sigma_1} \left( \frac{\sigma_1}{\sigma_{\omega}} + 3 \lambda \right) + \sqrt{2 \pi} \Lambda_{\max} \left( \frac{\sigma_{\omega}^2}{16 \sigma_1^4} + \frac{1}{\sigma_{\omega}^2} \right) + \frac{\lambda^2 \sigma_{\omega}^2}{4 \sigma_1^4} \bigg],
\]
then $d_2 \leq D_2 (\sigma_1, \sigma_{k}, \lambda, \sigma_{\omega})$.

The next step is giving a simply bounds for $c_1(\mu_1, \sigma_1, \mu_k, \sigma_k, \lambda, \sigma_{\omega})$ and $c_2(\mu_1, \sigma_1, \mu_k, \sigma_k, \lambda, \sigma_{\omega})$. We use the fact $a \vee b \leq a + b$ and $a \vee b \vee c = [ (a \vee b) \vee c ]$, then obtain that
\[
    \begin{aligned}
         & \bigg( \Lambda_{\max} \left( \frac{\log d_2}{3 \log 2} + 1\right) \bigg) \vee \frac{18 \sigma_{\omega}^2 \{ (2 \mu_1 - 1)^2 \vee (2 \mu_k - 1)^2 \}}{(\lambda^2 + \lambda + 1 / 4)^2} \vee \frac{2 \sigma_{\omega}^2 (1 + 2 \lambda^2)}{(1 + 2 \lambda)^2}  \\
        \leq & \bigg( \Lambda_{\max} \left( \frac{\log d_2}{3 \log 2} + 1\right) \bigg) \vee \frac{72 \sigma_{\omega}^2 \{ (2 \mu_1 - 1)^2 + (2 \mu_k - 1)^2 \} }{(1 + 2\lambda)^2} \vee \frac{2 \sigma_{\omega}^2 (1 + 2 \lambda^2)}{(1 + 2 \lambda)^2} \\
        \leq & \bigg( \Lambda_{\max} \left( \frac{\log d_2}{3 \log 2} + 1\right) \bigg) \vee \frac{72 \sigma_{\omega}^2 \times 2 + 2 \sigma_{\omega}^2 (1 + 2 \lambda^2)}{(1 + 2 \lambda)^2} \\
        \leq & \bigg( \Lambda_{\max} \left( \frac{\log d_2}{3 \log 2} + 1\right) \bigg) \vee \frac{144 \sigma_{\omega}^2 + 2 \sigma_{\omega}^2 + 4 \lambda^2 \sigma_{\omega}^2}{4 \lambda^2} \\
        \leq & \bigg[ \Lambda_{\max} \left( \frac{\log d_2}{3 \log 2} + 1\right) \bigg] + 38  \sigma_{\omega}^2 \leq \Lambda_{\max} \left( \frac{\log D_2 (\sigma_1, \sigma_{k}, \lambda, \sigma_{\omega})}{3 \log 2} + 1\right) + 38  \sigma_{\omega}^2.
    \end{aligned}
\]
and 
\[
    \begin{aligned}
        & \bigg( d_1 \left( \frac{\log d_2}{3 \log 2} + 1\right) \bigg) \vee \frac{18 \sigma_{\omega}^2 \{ (2 \mu_1 - 1)^2 \vee (2 \mu_k - 1)^2 \} }{(\lambda^2 + \lambda + 1 / 4)^2} \vee \frac{2 \sigma_{\omega}^2 (1 + 2 \lambda^2)}{(1 + 2 \lambda)^2} \\
        \leq & D_1 (\sigma_1, \sigma_{k}, \lambda, \sigma_{\omega}) \left( \frac{\log D_2 (\sigma_1, \sigma_{k}, \lambda, \sigma_{\omega})}{3 \log 2} + 1\right) + 38  \sigma_{\omega}^2,
    \end{aligned}
\]
where the last steps in two inequalities above are by \eqref{final_bound_1} and \eqref{final_bound_2}.
Define 
\[
    C_1 (\sigma_1, \sigma_k, \lambda, \sigma_{\omega}) = 10 \bigg[ \Lambda_{\max} \left( \frac{\log D_2 (\sigma_1, \sigma_{k}, \lambda, \sigma_{\omega})}{3 \log 2} + 1\right) + 38  \sigma_{\omega}^2 \bigg],
\]
and
\[
    C_2 (\sigma_1, \sigma_k, \lambda, \sigma_{\omega}) = 50 \lambda^2 \sigma_k^2 + 10 \bigg[ D_1 (\sigma_1, \sigma_{k}, \lambda, \sigma_{\omega}) \left( \frac{\log D_2 (\sigma_1, \sigma_{k}, \lambda, \sigma_{\omega})}{3 \log 2} + 1\right) + 38  \sigma_{\omega}^2 \bigg].
\]
Note that 
\[
    3 + 2 e^{9 / 8} \leq 10, \qquad \frac{72 \lambda^2 \sigma_k^2}{(1 + 2 \lambda)^2}  (5 + 4 e^{9 / 8}) \leq \frac{72 \lambda^2 \sigma_k^2}{25}  (5 + 4 e^{9 / 8}) \leq 50 \lambda^2 \sigma_k^2,
\]
we will have $c_1 (\mu_1, \sigma_1, \mu_k, \sigma_k,\lambda, \sigma_{\omega}) \leq C_1 (\sigma_1, \sigma_k, \lambda, \sigma_{\omega})$ and $c_2 (\mu_1, \sigma_1, \mu_k, \sigma_k,\lambda, \sigma_{\omega}) \leq C_2 (\sigma_1, \sigma_k, \lambda, \sigma_{\omega})$.
Finally, by using the bound in \eqref{reg_T_med},  one has
\[
    \begin{aligned}
        \operatorname{Reg}_T & \leq \sum_{k = 2}^K \Delta_k \bigg[ 7 + \left\{ c_1(\mu_1, \sigma_1, \mu_k, \sigma_k, \lambda, \sigma_{\omega}) + c_2(\mu_1, \sigma_1, \mu_k, \sigma_k, \lambda, \sigma_{\omega})\Delta_k^{-2} \right\} \log T \bigg] \\
        & \leq \sum_{k = 2}^K \Delta_k \bigg[ 7 + \left\{ C_1(\sigma_1, \sigma_k, \lambda, \sigma_{\omega}) + C_2(\sigma_1, \sigma_k, \lambda, \sigma_{\omega})\Delta_k^{-2} \right\} \log T \bigg].
    \end{aligned}
\]
For the problem-dependent case. For the problem-independent case, denote
\[
    J_{k, T} := \sum_{t = 1}^T \ind(I_t = k) = \sum_{k = 2}^K (a_k + b_k)
\]
then
\begin{equation}\label{pro_ind_tech}
    \begin{aligned}
        \operatorname{Reg}_T & = \sum_{\Delta_k : \Delta_k < \Delta} \Delta_k \E J_{k, T} + \sum_{\Delta_k : \Delta_k \geq \Delta} \Delta_k \E J_{k, T} \\
        \alignedoverset{\text{by the bounds of } a_k, b_k}{\leq} T \Delta + \sum_{\Delta_k : \Delta_k \geq \Delta}  \Delta_k \bigg[ 7 + \left\{ C_1(\sigma_1, \sigma_k, \lambda, \sigma_{\omega}) + C_2(\sigma_1, \sigma_k, \lambda, \sigma_{\omega})\Delta_k^{-2} \right\} \log T \bigg] \\
        & = T \Delta + 7 \sum_{\Delta_k : \Delta_k \geq \Delta} \Delta_k + \sum_{\Delta_k : \Delta_k \geq \Delta} C_1(\sigma_1, \sigma_k, \lambda, \sigma_{\omega})  \Delta_k \log T + \sum_{\Delta_k : \Delta_k \geq \Delta} \frac{C_2 (\sigma_1, \sigma_k, \lambda, \sigma_{\omega}) }{\Delta_k} \log T \\
        \alignedoverset{\text{by } \Delta_k \leq \mu_1 \leq 1}{\leq} T \Delta + 7 K \mu_1 + \max_{k \in [K] \setminus \{ 1 \}}  C_1(\sigma_1, \sigma_k, \lambda, \sigma_{\omega}) K \log T + \max_{k \in [K] \setminus \{ 1 \}}  C_2 ( \sigma_1, \sigma_k, \lambda, \sigma_{\omega}) \frac{K \log T}{\Delta},
    \end{aligned}
\end{equation}
for any previously specified $\Delta \in (0, 1)$. Take $\Delta =  \sqrt{\max_{k \in [K] \setminus \{ 1 \}}  C_2 (\sigma_1, \sigma_k, \lambda, \sigma_{\omega}) K \log T / T}$, we have 
\[
     \operatorname{Reg}_T \leq 7 K \mu_1 + \max_{k \in [K] \setminus \{ 1 \}}  C_1(\sigma_1, \sigma_k, \lambda, \sigma_{\omega}) K \log T + 2 \sqrt{\max_{k \in [K] \setminus \{ 1 \}}  C_2 (\sigma_1, \sigma_k, \lambda, \sigma_{\omega}) K T \log T}.
\]
Therefore, we finish the proof of Theorem \ref{thm.MAB}.

For proving Theorem \ref{thm.MAB.1}. Set $\sigma_k \equiv 1$ for $k \in [K]$, we have 
\[
    \begin{aligned}
        D_1 (1, 1, \lambda, \sigma_{\omega}) & = \bigg[ \big(1 +  8 \sqrt{2} \big) \big( 16 + \sigma_{\omega}^2 \big) + 16 + 3 \sigma_{\omega}^2 + 3 \sigma_{\omega}^2 + 1 \bigg] \sigma_{\omega}^2 \lambda^4 \\
        & \leq \bigg[ 13 \big( 16 + 3 \sigma_{\omega}^2 \big) + 6 \sigma_{\omega}^2 + 17 \bigg] \sigma_{\omega}^2 \lambda^4 = 45 (3 +  \sigma_{\omega}^2) \lambda^4
    \end{aligned}
\]
and
\[
    \begin{aligned}
        D_2 (1, 1, \lambda, \sigma_{\omega}) &=  3 \bigg[ 1 +  \frac{3 \sqrt{\pi}\sigma_{\omega}^2}{2} \left( \frac{1}{\sigma_{\omega}} + 3 \lambda \right) + 8 \sqrt{ \pi}\left( \frac{\sigma_{\omega}^2}{16} + \frac{1}{\sigma_{\omega}^2} \right) + \frac{\lambda^2 \sigma_{\omega}^2}{4 } \bigg] \\
        & \leq  3 \bigg[ 1 +  \frac{3 \sqrt{\pi}\sigma_{\omega}^2}{2} \left( \frac{1}{\sigma_{\omega}} + 3 \right) + 8\sqrt{\pi}\left( \frac{\sigma_{\omega}^2}{16 } + \frac{1}{\sigma_{\omega}^2} \right) + \frac{\sigma_{\omega}^2}{4} \bigg] \lambda^2 \leq \big( 1 + 15 \sigma_{\omega}^{-2} + 3 \sigma_{\omega} + 10 \sigma_{\omega}^2\big) \lambda^2.
    \end{aligned}
\]
Therefore, we have 
\[
    \begin{aligned}
        C_1 (1, 1, \lambda, \sigma_{\omega}) & = 10 \bigg[ 8 \sqrt{2}  \left( \frac{\log D_2 (1, 1, \lambda, \sigma_{\omega})}{3 \log 2} + 1\right) + 38  \sigma_{\omega}^2 \bigg] \\
        & \leq 10 \bigg[ 8 \sqrt{2}  \left( \frac{\log \big[ ( 1 + 15 \sigma_{\omega}^{-2} + 3 \sigma_{\omega} + 10 \sigma_{\omega}^2)  \lambda^2\big]}{3 \log 2} + 1\right) + 38  \sigma_{\omega}^2 \bigg],
    \end{aligned}
\]
and
\[
    \begin{aligned}
        C_2( 1,1, \lambda, \sigma_{\omega}) & = 50 \lambda^2 + 10 \bigg[ D_1 (1, 1, \lambda, \sigma_{\omega}) \left( \frac{\log D_2 (1, ,1, \lambda, \sigma_{\omega})}{3 \log 2} + 1\right) + 38  \sigma_{\omega}^2 \bigg] \\
        & \leq 50 \lambda^2 + 10 \left[ 45( 3+ \sigma_{\omega}^2) \left( \frac{\log \big[ ( 1 + 15 \sigma_{\omega}^{-2} + 3 \sigma_{\omega} + 10 \sigma_{\omega}^2)  \lambda^2\big]}{3 \log 2} + 1\right) \lambda^4  + 38  \sigma_{\omega}^2 \right].
    \end{aligned}
\]
The above bounds together with the fact that
\[
    C_1(\sigma_1, \sigma_k, \lambda, \sigma_{\omega}) + C_2(\sigma_1, \sigma_k, \lambda, \sigma_{\omega})\Delta_k^{-2} \leq \frac{  C_1(\sigma_1, \sigma_k, \lambda, \sigma_{\omega}) + C_2(\sigma_1, \sigma_k, \lambda, \sigma_{\omega})}{\Delta_k^2}
\]
gives the first result in Theorem \ref{thm.MAB.1}. Finally, by using the same technique as in \eqref{pro_ind_tech}, we obtain the problem-independent regret in we require.
\end{proof}

\section{Lemmas on Bounding Regret Components}\label{sec_lem_components}

\subsection{Lemmas on bounding $a_k$.}

\begin{lemma}[Bounding $a_{k, s}$ for any $s>0$]\label{lem_bounding_a_k_s_constant}
    Set 
    \[
        \lambda \geq \frac{4 \sigma_1}{\sigma_{\omega}} + \sqrt{\frac{4\sigma_1}{\sigma_{\omega}} \left( \frac{4 \sigma_1}{\sigma_{\omega}} + 1\right)},
    \]
    then 
    \[
        a_{k, s} \leq 2 + 4 e^{9 / 8}
    \]
    for any $s \geq 0$.
\end{lemma}

\begin{proof}
Note that if we take
\begin{equation}\label{requirement_0}
    \tau_k \leq \frac{\mu_1 + \lambda}{1 + 2 \lambda},
\end{equation}
then we will have
\begin{equation}\label{constant_bound}
    \begin{aligned}
    	a_{k, s} & = \E \left[ \left( \frac{1}{Q_{1, s}(\tau_k)}  - 1\right) \wedge T \right] \\
        \alignedoverset{\text{by } Q_{1, s}(\cdot) \text{ is decreasing}}{\leq}\E \frac{1}{Q_{1, s}(\tau_k)} \leq \E Q_{1, s}^{-1} \left( \frac{\mu_1 + \lambda}{1 + 2\lambda}\right),
    \end{aligned}
\end{equation}
Hence, we need to find a lower bound of the tail probability $Q_{1, s} \left( \frac{\mu_1 + \lambda}{1 + 2\lambda}\right) = \pr \left( \overline{Y}_{1, s} > \frac{\mu_1 + \lambda}{1 + 2\lambda} \mid \his_{1, s}\right)$. And throughout the proof, we will use the choice of \eqref{requirement_0} for $\tau_k$, and we can take advantage of the bound \eqref{constant_bound}. 

For further analysis \eqref{constant_bound}, we write it as the probability with respect to the weighted random summation of the Gaussian random variables $\{ \omega_i, \omega_i', \omega_i''\}$.

Let $x_i := (2 \lambda + 1) R_{1, i} - (\mu_1 + \lambda)$, $y_i := \lambda (\mu_1 - 1 - \lambda) $, $z_i :=  \lambda (\lambda + \mu_1) $ for $i \in [s]$, then
\[
    S_x = \sum_{i = 1}^s x_i = s \big( (2 \lambda + 1) \overline{R}_{1, s} - (\mu_1 + \lambda) \big), \quad S_y = \sum_{i = 1}^s y_i = s \lambda (\mu_1 - 1 - \lambda), \quad S_z = \sum_{i = 1}^s z_i = s \lambda (\lambda + \mu_1),
\]
with $S_x - S_y - S_z = s (2 \lambda + 1) (\overline{R}_{1, s} - \mu_1)$ and
\[
    T_x = \sum_{i = 1}^s x_i^2 = \sum_{i = 1}^s \big[  (2 \lambda + 1) \overline{R}_{1, s} - (\mu_1 + \lambda) \big]^2, \quad T_y = \sum_{i = 1}^s y_i^2 = s \lambda^2 (\mu_1 - 1 - \lambda)^2, \quad T_z = \sum_{i = 1}^s z_i^2 = s \lambda^2 (\lambda + \mu_1)^2,
\]
with $T_x + T_y + T_z = \sum_{i = 1}^s \big[  (2 \lambda + 1) \overline{R}_{1, s} - (\mu_1 + \lambda) \big]^2  + s \big[ \lambda^2 (\mu_1 - 1 - \lambda)^2 + \lambda^2 (\lambda + \mu_1)^2\big]$. For a standard Gaussian variable $Z$, we have Lemma \ref{lemma:gaussian_weight}, and
then we write the probability as
\[
    \begin{aligned}
        & Q_{1, s} \left( \frac{\mu_1 + \lambda}{1 + 2 \lambda}\right) \\
        = & \pr \left( \overline{Y}_{1, s} > \frac{\mu_1 + \lambda}{1 + 2\lambda} \mid \his_{1, s}\right) \\
        = & \pr_{\omega, \omega', \omega''} \left( \sum_{i = 1}^s \big( (2 \lambda + 1) R_{1, i} - (\mu_1 + \lambda) \big)\omega_i > \lambda (\mu_1 - 1 - \lambda) \sum_{i = 1}^s \omega_i' + \lambda (\lambda + \mu_1) \sum_{i = 1}^s \omega_i'' \right).
    \end{aligned}
\]
Hence, by applying the first inequality in Lemma \ref{lemma:gaussian_weight} and using the fact that $T_x + T_y + T_z \geq T_z \geq s \lambda^4$, we can get that
\[
    \begin{aligned}
        & Q_{1, s} \left( \frac{\mu_1 + \lambda}{1 + 2 \lambda}\right) \\
        = & \pr_{\omega, \omega', \omega''} \left( \sum_{i = 1}^s x_i \omega_i - \sum_{i = 1}^s y_i \omega_i' -   \sum_{i = 1}^s z_i \omega_i'' > 0 \right) \\
        = & \pr_{Z} \left( Z > \frac{- (S_x - S_y - S_z)}{\sigma_{\omega} \sqrt{T_x + T_y + T_z}}\right) \\
        \geq & \frac{1}{2} \times \ind \left( \frac{- (S_x - S_y - S_z)}{\sigma_{\omega} \sqrt{T_x + T_y + T_z}} < 0 \right) + \frac{1}{4} \exp \left\{ - \frac{(S_x - S_y - S_z)^2}{\sigma_{\omega}^2 (T_x + T_y + T_z)} \right\} \times  \ind \left( \frac{- (S_x - S_y - S_z)}{\sigma_{\omega} \sqrt{T_x + T_y + T_z}} > 0 \right) \\
        \geq & \frac{1}{2} \times \ind \left(S_x - S_y - S_z > 0 \right) + \frac{1}{4} \exp \left\{ - \frac{(S_x - S_y - S_z)^2}{\sigma_{\omega}^2 \cdot s \lambda^4} \right\} \times  \ind \left( S_x - S_y - S_z \leq 0 \right) \\
        = & \frac{1}{2} \ind \left( \overline{R}_{1, s} - \mu_1 > 0 \right) + \frac{1}{4} \exp \left\{ - \frac{s (2 \lambda + 1)^2 (\overline{R}_{1, s} - \mu_1 )^2}{\sigma_{\omega}^2\lambda^4}\right\} \ind \left( \overline{R}_{1, s} - \mu_1 \leq 0 \right).
    \end{aligned}
\]
where $Z$ again is the standard Gaussian variable and we will use this notation throughout this proof. Then this result combined with the fact \eqref{constant_bound} gives the upper bound for $a_{k, s}$ as follows:
\[
    \begin{aligned}
    	a_{k, s} & \leq 2 \pr \left( \overline{R}_{1, s} - \mu_1 > 0 \right) + 4  \E \exp \left\{ \frac{s (2 \lambda + 1)^2 (\overline{R}_{1, s} - \mu_1 )^2}{\sigma_{\omega}^2\lambda^4}\right\} \ind \left( \overline{R}_{1, s} - \mu_1 \leq 0 \right) \\
    	& \leq 2 + 4  \E \exp \left\{ \frac{s (2 \lambda + 1)^2 (\overline{R}_{1, s} - \mu_1 )^2}{\sigma_{\omega}^2\lambda^4}\right\}.
    \end{aligned}
\]
From the above expression, for proving $a_{k, s}$ is bounded by a constant free of $s$, it remains to show the expectation
\begin{equation}\label{a_k_constant_3}
	\E \exp \left\{ \frac{s (2 \lambda + 1)^2 (\overline{R}_{1, s} - \mu_1 )^2}{\sigma_{\omega}^2\lambda^4}\right\} 
\end{equation}
can be bounded below some constants free of $s$.
Applying Lemma \ref{lem_subGtosubE}, we know that if we take
\[
    \frac{s(2 \lambda + 1)^2}{\sigma_{\omega}^2 \lambda^4} \leq \frac{s}{8 \sigma_1^2},
\]
then \eqref{a_k_constant_3} can be upper bounded by $e^{9 / 8}$. A sufficient condition for the above inequality is
\[
    \frac{s(2 \lambda + 1)^2}{\sigma_{\omega}^2 \lambda^4} \leq \frac{s}{16 \sigma_1^2}, \qquad \lambda^2 - \frac{8\sigma_1}{\sigma_{\omega}} \lambda - \frac{4 \sigma_1}{\sigma_{\omega}} \geq 0,
\]
or say,
\begin{equation}\label{tuning_condition_exact}
	\lambda \geq \frac{4 \sigma_1}{\sigma_{\omega}} + \sqrt{\frac{4\sigma_1}{\sigma_{\omega}} \left( \frac{4 \sigma_1}{\sigma_{\omega}} + 1\right)}.
\end{equation}

Therefore, if we take the tuning parameters $\lambda$ and $\sigma_{\omega}$ as in \eqref{tuning_condition_exact}, then $a_{k, s}$ will be bounded by a constant such that 
\begin{equation}\label{a_k_constant_final}
	a_{k, s} \leq 2 + 4 e^{9 / 8}.
\end{equation}
Note that this step derives the tuning conditions for the tuning parameters in Theorem \ref{thm.MAB}. 
\end{proof}

\begin{lemma}[Bounding $a_{k, s, 1}$ at \eqref{def_a_k_s_1}]\label{lem_bounding_a_k_s_1}
    Take 
    \[
        s_{a, 1}(T) := \frac{4 \sigma_1^2}{C_1^2 (1 + 2 \lambda)^2 \Delta_k^2} \times\log T.
    \]
    Then for any $s \geq s_{a, 1}(T)$,
    \[
        a_{k, s, 1} \leq T^{-1},
    \]
    when $T \geq 2$.
\end{lemma}

\begin{proof}
    We will bound $a_{k, s, 1}$ by bounding the probability of the event $A_{k, s}^c$. Write
\begin{equation}\label{a_k_1_bound}
    a_{k, s, 1} = \E \Big[ \big( N_{1, s} (\tau_k) \wedge T \big) \ind (A_{1, s}^c)\Big] \leq \E \Big[ T  \ind(A_{1, s}^c)\Big] = T \pr (A_{1, s}^c) 
\end{equation}
Since the summation of independent sub-Gaussian variables is still sub-Gaussian, \eqref{lem_subGtosubE_1} in Lemma \ref{lem_subGtosubE} gives $\overline{R}_{1, s} - \mu_1 \sim \subG (\sigma_1^2 /  s)$. By the concentration inequality of the sub-Gaussian variable,
\[
    \begin{aligned}
        \pr (A_{1, s}^c) & = \pr \left( \big| \overline{R}_{1, s} - \mu_1 \big| \geq (1 + 2\lambda) C_1 \Delta_k \right) \\
        & \leq 2 \exp \left\{ - \frac{(1 + 2\lambda)^2 C_1^2 \Delta_k^2 s}{2 \sigma_1^2}\right\} \\
        \alignedoverset{\text{take } s \geq s_{a, 1} (T) \geq  \begin{matrix} \frac{2 \sigma_1^2}{C_1^2 (1 + 2 \lambda)^2 \Delta_k^2} \times\log (\sqrt{2} T) \end{matrix}}{\leq} \frac{1}{T^2}.
    \end{aligned}
\]
Hence, we obtain that
\[
    a_{k, s, 1} \leq  T \pr(A_{1, s}^c) \leq T^{-1}, \qquad \text{ for any } s \geq s_{a, 1} (T) \asymp \log T.
\]
\end{proof}

\begin{lemma}[Bounding $a_{k, s, 2}$ at \eqref{def_a_k_s_2}]\label{lem_bounding_a_k_s_2}
    Take 
    \[
        \begin{aligned}
            s_{a, 2}(T) := \Bigg[ \big[ \Lambda_1 + \mathfrak{a}_1 + \mathfrak{b}_1 + \Lambda_1\mathfrak{a}_1 \big] & \left(\frac{1}{3}\log^{-1}2 \times \log \left\{ 3 \left[  1 + \frac{\sqrt{\pi \mathfrak{b}_1}}{2}  + \frac{\sqrt{2 \pi} \Lambda_1 \mathfrak{a}_1}{2 \mathfrak{b}_1^2} + \frac{2 \mathfrak{a}_1}{\mathfrak{b}_1} \right]  \right\}  + 1\right)\\
            & \vee \frac{18 \sigma_{\omega}^2 (2 \mu_1 - 1)^2}{(\lambda^2 + \lambda + 1 / 4)^2} \vee \frac{2 \sigma_{\omega}^2 (1 + 2 \lambda^2)}{(1 + 2 \lambda)^2} \Bigg] \times 3 \log T,
        \end{aligned}
    \]
    where 
    \[
        \mathfrak{a}_1 = \frac{4 \sigma_{\omega}^2 \lambda^2}{3 C_2^2 (1 + 2\lambda)^2 \Delta_k^2} , \qquad \mathfrak{b}_1 = \frac{2 \sigma_{\omega}^2 \left[ 2 \lambda^2 \sigma_1^2 + (1 + 2 \lambda)^2 C_2^2 \Delta_k^2 \right]}{3 C_2^2 (1 + 2 \lambda)^2 \Delta_k^2}, \qquad \Lambda_1 = 8 \sqrt{2} \sigma_1^2.
    \]
    Then for any $s \geq s_{a, 2}(T)$,
    \[
        a_{k, s, 2} \leq T^{-1},
    \]
    for any $T \geq 2$.
\end{lemma}

\begin{proof}
    We will bound the probability of $G_{1, s}^c$ conditioning on $\his_{1, s}$ to prove this lemma.
Note that 
\[
    \begin{aligned}
    	& \E_{\omega, \omega', \omega''} \left[ \sum_{i = 1}^s \big[ (1 + 2\lambda) R_{1, i} - (\overline{R}_{1, s} + \lambda)\big] \omega_i + \sum_{i = 1}^s \lambda (1 + \lambda - \overline{R}_{1, s}) \omega_i' - \sum_{i = 1}^s \lambda (\overline{R}_{1, s} + \lambda) \omega_i'' \right] \\
    	 = & s \big( (1 + 2\lambda) \overline{R}_{1, s} - (\overline{R}_{1, s} + \lambda) + \lambda (1 + \lambda - \overline{R}_{1, s}) - \lambda (\overline{R}_{1, s} + \lambda) \big) = 0,
    \end{aligned}
\]
we can bound the tail probability $\pr (G_{1, s}^c \mid \his_{1, s})$ by
\begin{equation}\label{pr_G_bound}
    \begin{aligned}
        & \pr (G_{1, s}^c \mid \his_{1, s})\\
        = & \pr \big( |\overline{Y}_{k', s} - \overline{R}^*_{k', s}| > C_3 \Delta_k \mid \his_{1, s}\big)\\
        = & 2 \pr_{\omega, \omega', \omega''}  \left\{ \left| \frac{\sum_{i = 1}^s \big[ (1 + 2\lambda) R_{1, i} - (\overline{R}_{1, s} + \lambda)\big] \omega_i + \sum_{i = 1}^s \lambda (1 + \lambda - \overline{R}_{1, s}) \omega_i' - \sum_{i = 1}^s \lambda (\overline{R}_{1, s} + \lambda) \omega_i''}{(1 + 2 \lambda)(\sum_{i = 1}^s \omega_i + \lambda \sum_{i = 1}^s \omega_i' + \lambda \sum_{i = 1}^s \omega_i'')} \right| > C_2 \Delta_k \right\} \\
        \alignedoverset{\text{by Lemma \ref{lem_gau_frac}}}{\leq} 2 \pr_{\omega, \omega', \omega''} \Bigg\{ \sum_{i = 1}^s \big[ (1 + 2\lambda) R_{1, i} - (\overline{R}_{1, s} + \lambda)\big] \omega_i + \sum_{i = 1}^s \lambda (1 + \lambda - \overline{R}_{1, s}) \omega_i'  \\
        & ~~~~~~~~~~~~~~~~~~~~~~~~~~~~~~~~~~~ - \sum_{i = 1}^s \lambda (\overline{R}_{1, s} + \lambda) \omega_i'' - C_2 (1 + 2\lambda) \Delta_k \bigg[ \sum_{i = 1}^s \omega_i + \lambda \sum_{i = 1}^s \omega_i' + \lambda \sum_{i = 1}^s \omega_i''\bigg] > 0\Bigg\} \\
        & ~~~~~~~~~~  + \pr_{\omega, \omega', \omega''} \left\{ \sum_{i = 1}^s \omega_i + \lambda \sum_{i = 1}^s \omega_i' + \lambda \sum_{i = 1}^s \omega_i'' < 0\right\} = 2I + II,
    \end{aligned}
\end{equation}
where 
\[
    \begin{aligned}
        I & = \pr_{\omega, \omega', \omega''} \Bigg\{ \sum_{i = 1}^s \big[ (1 + 2\lambda) R_{1, i} - (\overline{R}_{1, s} + \lambda)\big] \omega_i + \sum_{i = 1}^s \lambda (1 + \lambda - \overline{R}_{1, s}) \omega_i'  \\
        & ~~~~~~~~~~~~~~~~~~~~~~~~~~~~~~~~~~~~~~~ - \sum_{i = 1}^s \lambda (\overline{R}_{1, s} + \lambda) \omega_i'' - C_2 (1 + 2\lambda) \Delta_k \bigg[ \sum_{i = 1}^s \omega_i + \lambda \sum_{i = 1}^s \omega_i' + \lambda \sum_{i = 1}^s \omega_i''\bigg] > 0\Bigg\}.
    \end{aligned}
\]
and 
\[
    II = \pr_{\omega, \omega', \omega''} \left\{ \sum_{i = 1}^s \omega_i + \lambda \sum_{i = 1}^s \omega_i' + \lambda \sum_{i = 1}^s \omega_i'' < 0\right\}.
\]
For bounding $\pr (G_{1, s}^c \mid \his_{1, s})$, it suffices to bound $I^{(1)}$ and $II^{(2)}$ separately. we first study $II^{(1)}$. Write
\[
    \begin{aligned}
    	II & = \pr_{\omega, \omega', \omega''} \left( \sum_{i = 1}^s \omega_i + \lambda \sum_{i = 1}^s \omega_i' + \lambda \sum_{i = 1}^s \omega_i'' < 0\right) \\
    	& = \pr_{\omega, \omega', \omega''} \left( \frac{\sum_{i = 1}^s (\omega_i - 1) + \lambda \sum_{i = 1}^s (\omega_i' - 1) + \lambda \sum_{i = 1}^s (\omega_i'' - 1)}{\sqrt{s \sigma_{\omega}^2 (1 +2 \lambda^2)}} < - \frac{(1 + 2\lambda) s}{\sqrt{s \sigma_{\omega}^2 (1 +2 \lambda^2)}} \right) \\
    	& = \pr_{\omega, \omega', \omega''} \left( Z < - \frac{(1 + 2\lambda) }{\sigma_{\omega} \sqrt{1 + 2 \lambda^2}} \times \sqrt{s} \right) \\
    	\alignedoverset{\text{by Lemma \ref{lemma:gaussian_weight}}}{\leq} \frac{1}{2} \exp \left\{ - \frac{s (1 + 2 \lambda)^2}{2 \sigma_{\omega}^2 (1 + 2 \lambda^2)}\right\}.
    \end{aligned}
\]
For studying $I^{(1)}$, we define the following functions of $\{ R_{1,i}\}_{i = 1}^s$:
\[
   f \big(\{ R_{1,i}\}_{i = 1}^s\big) = f_1 \big(\{ R_{1,i}\}_{i = 1}^s \big) + f_2 \big(\{ R_{1,i}\}_{i = 1}^s \big) - f_3 \big(\{ R_{1,i}\}_{i = 1}^s \big)
\]
with 
\[
    \begin{aligned}
    	f_1 \big(\{ R_{1,i}\}_{i = 1}^s \big) & = \sum_{i = 1}^s \big[ (R_{1, i} - \overline{R}_{1, s}) + \lambda (2 \overline{R}_{1, i} - 1) - (1 + 2 \lambda ) C_2 \Delta_k  \big] \omega_i,
    \end{aligned}
\]
\[
    f_2 \big(\{ R_{1,i}\}_{i = 1}^s \big) = \sum_{i = 1}^s \big[ \lambda (\lambda + 1  - \overline{R}_{1, s}) - (1 + 2 \lambda ) C_2 \Delta_k  \big] \omega_i',
\]
and
\[
    f_3 \big(\{ R_{1,i}\}_{i = 1}^s \big) = \sum_{i = 1}^s \big[ \lambda (\overline{R}_{1, s} + \lambda) + (1 + 2 \lambda ) C_2 \Delta_k  \big] \omega_i''.
\]
Then we can write $I^{(1)} = 2 \pr_{\omega, \omega', \omega'} (f_1 + f_2 - f_3 > 0)$. Note that $f_1, f_2$, and $f_3$ are mutually independent conditioning on $\his_{1, s}$. Then 
the expectation is
\[
    \E \big[ f_1 + f_2 - f_3 \mid \his_{1, s}\big] = - 3 C_2 (1 + 2 \lambda) \Delta_k s,
\] 
and the variance is
\begin{equation}\label{bound_var_f}
    \begin{aligned}
        & \var \big(f_1 + f_2 - f_3 \mid \his_{1, s} \big) \\
        = & \sigma_{\omega}^2 \Bigg [ \sum_{i = 1}^s \big[ (R_{1, i} - \overline{R}_{1, s}) + \lambda (2 \overline{R}_{1, i} - 1) - (1 + 2 \lambda ) C_2 \Delta_k  \big]^2 \\
        & \qquad \qquad \qquad + \sum_{i = 1}^s \big[ \lambda (\lambda + 1  - \overline{R}_{1, s}) - (1 + 2 \lambda ) C_2 \Delta_k  \big]^2 + \sum_{i = 1}^s \big[ \lambda (\overline{R}_{1, s} + \lambda) + (1 + 2 \lambda ) C_2 \Delta_k  \big]^2 \Bigg] \\ 
        = & \sigma_{\omega}^2 \Big[ V_1 + V_2 + V_3 \Big],
    \end{aligned}
 \end{equation}
 where 
 \[
     V_1 = \sum_{i = 1}^s \big[ (R_{1, i} - \overline{R}_{1, s}) + \lambda (2 \overline{R}_{1, i} - 1) - (1 + 2 \lambda ) C_2 \Delta_k  \big]^2,
 \]
 \[
      V_2 = \sum_{i = 1}^s \big[ \lambda (\lambda + 1  - \overline{R}_{1, s}) - (1 + 2 \lambda ) C_2 \Delta_k  \big]^2,
 \]
 and
 \[
      V_3 = \sum_{i = 1}^s \big[ \lambda (\overline{R}_{1, s} + \lambda) + (1 + 2 \lambda ) C_2 \Delta_k  \big]^2.
 \]
 For bounding the conditional variance above, we calculate its components as follows. 
 \begin{equation}\label{bound_I}
 	\begin{aligned}
 	    V_1 & = \sum_{i = 1}^s  \big[ (R_{1, i} - \overline{R}_{1, s}) + \lambda (2 \overline{R}_{1, i} - 1) \big]^2  + s \lambda^2 (\lambda + 1  - \overline{R}_{1, s})^2 + s \lambda^2 (\overline{R}_{1, s} + \lambda)^2 \\
 	    & = \sum_{i = 1}^s (R_{1, i} - \overline{R}_{1, s})^2 + 2 s \lambda^2 \big( 3 \overline{R}_{1, s}^2 - 3 \overline{R}_{1, s} + \lambda^2 + \lambda + 1 \big) \\
 	    & = \sum_{i = 1}^s (R_{1, i} - \mu_1)^2 + (6 \lambda^2 - 1)s (\overline{R}_{1, s} - \mu_1)^2 + 6  s \lambda^2 (2 \mu_1 - 3)  (\overline{R}_{1, s} - \mu_1) + 2 s \lambda^2 \big[\lambda^2 + \lambda + 1 + 3 \mu_1 (\mu_1 - 1) \big] \\
 	    \alignedoverset{\text{Cauchy inequality}}{\leq} 6 \lambda^2 \sum_{i = 1}^s (R_{1, i} - \mu_1)^2  + 6  s \lambda^2 (2 \mu_1 - 1)  (\overline{R}_{1, s} - \mu_1) + 2 s \lambda^2 \big( \lambda^2 + \lambda + 1 - \mu_1 (1 - \mu_1) \big) \\
 	    & = 6 \lambda^2 \sum_{i = 1}^s (R_{1, i} - \mu_1)^2 + 6  s \lambda^2 (2 \mu_1 - 1)  (\overline{R}_{1, s} - \mu_1)  + 2 s \lambda^2 \left( \lambda^2 + \lambda + 1 - \mu_1 (1 - \mu_1)\right),
 	\end{aligned}
 \end{equation}
 \[
     \begin{aligned}
         V_2 & = - 2 (1 + 2\lambda) C_3 \Delta_k \bigg[ \sum_{i = 1}^s  \big( (R_{1, i} - \overline{R}_{1, s}) + \lambda (2 \overline{R}_{1, i} - 1) \big) + s \lambda (\lambda + 1  - \overline{R}_{1, s}) -  s \lambda (\overline{R}_{1, s} + \lambda) \bigg] \\
         & = - 2 s (1 + 2\lambda) C_3 \Delta_k \Big[ \big(\overline{R}_{1, s} - \overline{R}_{1, s}) + \lambda (2 \overline{R}_{1, i} - 1) + \lambda (\lambda + 1  - \overline{R}_{1, s}) \big) -  \lambda (\overline{R}_{1, s} + \lambda) \Big] = 0,
     \end{aligned}
 \]
 and
$$V_3= 3 s (1 + 2\lambda)^2 C_3^2 \Delta_k^2. $$ 
Therefore, the conditional variance in \eqref{bound_var_f} is bounded by
 \begin{equation}\label{bound_var_f_new}
 	\begin{aligned}
 		 \sigma_{\omega}^{-2}\var \big(f_1 + f_2 - f_3 \mid \his_{1, s} \big) & \leq I^{(2)} + II^{(2)} + III^{(2)} \\
 		 & \leq  \underbrace{6 \lambda^2 \sum_{i = 1}^s (R_{1, i} - \mu_1)^2 + 6 s \lambda^2 (2\mu_1 - 1) \left( \overline{R}_{1, s} - \mu_1\right)}_{\text{the random part}} \\
 		 & ~~~~~~~~~~~~~~~~~~~ + \underbrace{2 s \lambda^2 \left( \lambda^2 + \lambda + 1 - 3 \mu_1 (1 - \mu_1)\right) + 3 s (1 + 2\lambda)^2 C_2^2 \Delta_k^2}_{\text{the determined part}} \\
 		 & = \mathfrak{R}_1 + \mathfrak{R}_2 + \mathfrak{D},
 	\end{aligned}
 \end{equation}
where
 \[
     \mathfrak{R}_1 = 6 \lambda^2 \sum_{i = 1}^s (R_{1, i} - \mu_1)^2, \qquad \mathfrak{R}_2 =  6 s \lambda^2 (2\mu_1 - 1) \left( \overline{R}_{1, s} - \mu_1\right) + 2 s \lambda^2 \left( \lambda^2 + \lambda + 1 - 3 \mu_1 (1 - \mu_1)\right),
 \]
and
\[
    \mathfrak{D} = 3 s (1 + 2\lambda)^2 C_2^2 \Delta_k^2.
\]
Obviously, $\mathfrak{R}_1$ and $\mathfrak{R}_2$ are strictly positive. Furthermore, for $\mathfrak{R}_2$, we can also prove it is non-positive with a high probability.
Indeed, note that $0 \leq 3 \mu_1 (1 - \mu_1) \leq \frac{3}{4}$, we have
 \begin{equation}\label{R_2_pr}
     \begin{aligned}
     	\pr ( \mathfrak{R}_2 > 0) & = \pr \left( 3 (2 \mu_1 - 1)\left( \overline{R}_{1, s} - {\mu_1} \right) > \lambda^2 + \lambda + 1 - 3 \mu_1 (1 - \mu_1)\right) \\
     	& \leq \pr \left( \Big| 3 (2 \mu_1 - 1)\left( \overline{R}_{1, s} - {\mu_1} \right)\Big| > \Big| \lambda^2 + \lambda + 1 - 3 \mu_1 (1 - \mu_1) \Big| \right) \ind\left( \mu_1 \neq  \frac{1}{2}\right) + \pr (\varnothing) \times \ind\left( \mu_1 = \frac{1}{2}\right) \\
     	\alignedoverset{0 \leq 3 \mu_1 (1 - \mu_1) \leq \frac{3}{4}}{\leq} \pr \left( \Big| \overline{R}_{1, s} - {\mu_1} \Big| > \frac{\lambda^2 + \lambda + 1 / 4}{|3 (2 \mu_1 - 1)|} \right) \ind\left( \mu_1 \neq  \frac{1}{2}\right) + 0 \times \ind\left( \mu_1 = \frac{1}{2}\right) \\
     	\alignedoverset{\text{by sub-Gaussian inequality}}{\leq} \exp \left\{ - \frac{(\lambda^2 + \lambda + 1 / 4)^2 s}{18 \sigma_{\omega}^2 (2 \mu_1 - 1)^2}\right\} \times \ind\left( \mu_1 \neq \frac{1}{2}\right) + 0 \times \ind\left( \mu_1 = \frac{1}{2}\right).
     \end{aligned}
 \end{equation}
Then by applying Lemma \ref{lemma:gaussian_weight} again,
\[
    \begin{aligned}
        \pr (G_{1, s}^c \mid \his_{1, s}) & = 2 I + II\\
        & = 2 \pr \left(\frac{f_1 + f_2 - f_3 - \E \big( f_1 + f_2 - f_3 \mid \his_{1, s}\big)}{\sqrt{\var \big(f_1 + f_2 - f_3 \mid \his_{1, s}\big)}} > - \frac{\E \big( f_1 + f_2 - f_3 \mid \his_{1, s}\big)}{\sqrt{\var \big(f_1 + f_2 - f_3 \mid \his_{1, s}\big)}} \right) \\
        & ~~~~~~~~~~~ + \frac{1}{2} \exp \left\{ - \frac{s (1 + 2 \lambda)^2}{2 \sigma_{\omega}^2 (1 + 2 \lambda^2)}\right\} \\
        & \leq 2 \exp \left\{ - \frac{\E^2 \big( f_1 + f_2 - f_3 \mid \his_{1, s}\big)}{2 \var \big(f_1 + f_2 - f_3 \mid \his_{1, s}\big)} \right\} + \frac{1}{2} \exp \left\{ - \frac{s (1 + 2 \lambda)^2}{2 \sigma_{\omega}^2 (1 + 2 \lambda^2)}\right\}.
    \end{aligned}
\]
Therefore, combining with \eqref{bound_var_f_new}, 
$a_{k, s, 2}$ will be bounded by 
\[
    \begin{aligned}
        a_{k, s, 2} & \leq \E \Big[ T  \ind(G_{1, s}^c)\Big] =  T \E \Big[ \pr(G_{1, s}^c \mid \his_{1, s}) \Big] \\
        \alignedoverset{\text{by \eqref{pr_G_bound}}}{\leq} \leq T \big[ 2 I + II \big] \\
        & \leq 2 T \E \exp \left\{ - \frac{\E^2 \big( f_1 + f_2 - f_3 \mid \his_{1, s}\big)}{2 \var \big(f_1 + f_2 - f_3 \mid \his_{1, s}\big)} \right\} + \frac{T}{2} \exp \left\{ - \frac{s (1 + 2 \lambda)^2}{2 \sigma_{\omega}^2 (1 + 2 \lambda^2)}\right\} \\
        & = T \E \exp \left\{ - \frac{\E^2 \big( f_1 + f_2 - f_3 \mid \his_{1, s}\big)}{2 \var \big(f_1 + f_2 - f_3 \mid \his_{1, s}\big)} \right\} \big( \ind (\mathfrak{R}_2 \leq 0) + \ind(\mathfrak{R}_2 > 0) \big) + \frac{T}{2} \exp \left\{ - \frac{s (1 + 2 \lambda)^2}{2 \sigma_{\omega}^2 (1 + 2 \lambda^2)}\right\}\\
        \alignedoverset{\text{by the decomposition \eqref{bound_var_f_new}}}{\leq} T \E \exp \left\{ - \frac{\E^2 \big( f_1 + f_2 - f_3 \mid \his_{1, s}\big)}{2 \sigma_{\omega}^2 \big(\mathfrak{R}_1 + \mathfrak{R}_2 + \mathfrak{D} \big)}\right\} \ind (\mathfrak{R}_2 \leq 0) + T \pr (\mathfrak{R}_2 > 0) +  \frac{T}{2} \exp \left\{ - \frac{s (1 + 2 \lambda)^2}{2 \sigma_{\omega}^2 (1 + 2 \lambda^2)}\right\}\\
        & \leq T \E \exp \left\{ - \frac{\E^2 \big( f_1 + f_2 - f_3 \mid \his_{1, s}\big)}{2 \sigma_{\omega}^2 \big(\mathfrak{R}_1 + \mathfrak{D} \big)}\right\} + T \pr (\mathfrak{R}_2 > 0) + \frac{T}{2} \exp \left\{ - \frac{s (1 + 2 \lambda)^2}{2 \sigma_{\omega}^2 (1 + 2 \lambda^2)}\right\}.
    \end{aligned}
\]
Recall $R_{1, i} - \mu_1 \sim \subG(\sigma_1^2)$, and then $(R_{1, i} - \mu_1)^2 - \var (R_{1, i}) := \xi_i \sim \subE(8 \sqrt{2} \sigma_1^2)$.
Therefore, for furthermore bounding $\mathfrak{R}_1 + \mathfrak{D}$, we have
\begin{equation}\label{G_function_bound}
	\begin{aligned}
		\mathfrak{R}_1 + \mathfrak{D} & = 6 \lambda^2 \left[ \sum_{i = 1}^s (R_{1, i} - \mu_1)^2 - s \var (R_{1, i}) \right] + \left[ s \var (R_{1, i}) + 3 s(1 + 2\lambda)^2 C_2^2 \Delta_k^2\right] \\
             & = 6 \lambda^2 \sum_{i = 1}^s \xi_i + s \left[ 6 \lambda^2\var (R_{1, i}) + 3 (1 + 2\lambda)^2 C_2^2 \Delta_k^2\right] \\
		\alignedoverset{\text{by $\var(R_{1, i}) \leq \sigma_1^2$}}{\leq}  6\lambda^2 \sum_{i = 1}^s \xi_i + 3 s \left[ 2 \lambda^2 \sigma_1^2 + (1 + 2\lambda)^2 C_2^2 \Delta_k^2\right]  \\
        & = 3 s \left[ 2 \lambda^2 \overline{\xi} +  2 \lambda^2 \sigma_1^2 + (1 + 2\lambda)^2 C_2^2 \Delta_k^2  \right],
	\end{aligned}
\end{equation}
where $\overline{\xi} = \frac{1}{s} \sum_{i = 1}^s \xi_i$. Then 
\begin{equation}\label{a_2_bound}
	\begin{aligned}
		a_{k, s, 2} & \leq T \E \exp \left\{ - \frac{\E^2 \big( f_1 + f_2 - f_3 \mid \his_{1, s}\big)}{2 \sigma_{\omega}^2 \big(\mathfrak{R}_1 + \mathfrak{D} \big)}\right\} + T \pr (\mathfrak{R}_2 > 0) + \frac{T}{2} \exp \left\{ - \frac{s (1 + 2 \lambda)^2}{2 \sigma_{\omega}^2 (1 + 2 \lambda^2)}\right\} \\
		\alignedoverset{\text{by \eqref{G_function_bound} and \eqref{R_2_pr}}}{\leq} T \E\exp \left\{ - \frac{9 C_2^2 (1 + 2\lambda)^2 \Delta_k^2 s^2 }{6  \sigma_{\omega}^2 s  \left[ 2 \lambda^2 \overline{\xi} +  2 \lambda^2 \sigma_1^2 + (1 + 2\lambda)^2 C_2^2 \Delta_k^2  \right] }\right\} \\
		 & ~~~~~~~~~~~~~~~ + T \left[ \exp \left\{ - \frac{(\lambda^2 + \lambda + 1 / 4)^2 s}{18 \sigma_{\omega}^2 (2 \mu_1 - 1)^2}\right\} \times \ind\left( \mu_1 \neq \frac{1}{2}\right) + 0 \times \ind\left( \mu_1 = \frac{1}{2}\right) \right] \\
           & ~~~~~~~~~~~~~~~ + \frac{T}{2} \exp \left\{ - \frac{s (1 + 2 \lambda)^2}{2 \sigma_{\omega}^2 (1 + 2 \lambda^2)}\right\} \\
		 & = T \E\exp \left\{ - \frac{3 C_2^2 (1 + 2\lambda)^2 \Delta_k^2 s }{ 2 \sigma_{\omega}^2  \left[ 2 \lambda^2 \overline{\xi} +  2 \lambda^2 \sigma_1^2 + (1 + 2\lambda)^2 C_2^2 \Delta_k^2  \right] }\right\} \\
		 & ~~~~~~~~~~~~~~~ + T \left[ \exp \left\{ - \frac{(\lambda^2 + \lambda + 1 / 4)^2 s}{18 \sigma_{\omega}^2 (2 \mu_1 - 1)^2}\right\} \times \ind\left( \mu_1 \neq \frac{1}{2}\right) + 0 \times \ind\left( \mu_1 = \frac{1}{2}\right) \right] \\
          & ~~~~~~~~~~~~~~~ + \frac{T}{2} \exp \left\{ - \frac{s (1 + 2 \lambda)^2}{2 \sigma_{\omega}^2 (1 + 2 \lambda^2)}\right\} .
	\end{aligned}
\end{equation}
The next step is to apply Lemma \ref{lem_subE_ineq}. Take 
\[
     \mathfrak{a}_1 = \frac{4 \sigma_{\omega}^2 \lambda^2}{3 C_2^2 (1 + 2\lambda)^2 \Delta_k^2} , \qquad \mathfrak{b}_1 = \frac{2 \sigma_{\omega}^2 \left[ 2 \lambda^2 \sigma_1^2 + (1 + 2 \lambda)^2 C_2^2 \Delta_k^2 \right]}{3 C_2^2 (1 + 2 \lambda)^2 \Delta_k^2}, \qquad \lambda_i \equiv 8 \sqrt{2} \sigma_1^2, \qquad i \in [s],
\]
and $\Lambda_1 = (s^{-1}\sum_{i = 1}^s \lambda_i^2)^{1 / 2} = 8 \sqrt{2} \sigma_1^2$, then Lemma \ref{lem_subE_ineq} gives 
\begin{equation}\label{app_sub_E_ineq_1}
    \begin{aligned}
        & \E\exp \left\{ - \frac{3 C_2^2 (1 + 2\lambda)^2 \Delta_k^2 s }{ 2 \sigma_{\omega}^2  \left[ 2 \lambda^2 \overline{\xi} +  2 \lambda^2 \sigma_1^2 + (1 + 2\lambda)^2 C_2^2 \Delta_k^2  \right] }\right\} \\
        & = \E \exp \left\{ - \frac{s}{\frac{4 \sigma_{\omega}^2 \lambda^2}{3 C_2^2 (1 + 2\lambda)^2 \Delta_k^2} \overline{\xi}  + \frac{2 \sigma_{\omega}^2 \left[ 2 \lambda^2 \sigma_1^2 + (1 + 2 \lambda)^2 C_2^2 \Delta_k^2 \right]}{3 C_2^2 (1 + 2 \lambda)^2 \Delta_k^2}}\right\} = \E \exp \left\{ - \frac{s}{\mathfrak{a}_1 \overline{\xi} + \mathfrak{b}_1} \right\} \\
        & \leq \left[  1 + \frac{\sqrt{\pi \mathfrak{b}_1}}{2} + \frac{\sqrt{2 \pi} \Lambda_1 \mathfrak{a}_1}{2 \mathfrak{b}_1^2} + \frac{2 \mathfrak{a}_1}{\mathfrak{b}_1} \right]  \exp \bigg\{ - \frac{s}{\sqrt{2 \Lambda_1\mathfrak{a}_1} \vee (\mathfrak{b}_1 + \Lambda_1\mathfrak{a}_1)}\bigg\}  \\
        \alignedoverset{\text{by } \sqrt{2 \Lambda_1\mathfrak{a}_1} \vee (\mathfrak{b}_1 + \Lambda_1\mathfrak{a}_1) \leq \mathfrak{b}_1 + \Lambda_1\mathfrak{a}_1 + \Lambda_1+ \mathfrak{a}_1}{\leq} \left[  1 + \frac{\sqrt{\pi \mathfrak{b}_1}}{2}  + \frac{\sqrt{2 \pi} \Lambda_1\mathfrak{a}_1}{2 \mathfrak{b}_1^2} + \frac{2 \mathfrak{a}_1}{\mathfrak{b}_1} \right]  s \exp \bigg\{ - \frac{s}{\Lambda_1+ \mathfrak{a}_1 + \mathfrak{b}_1 + \Lambda_1\mathfrak{a}_1}\bigg\}.
    \end{aligned}
\end{equation}
Hence, by taking 
\[
    \begin{aligned}
        s & \geq \big[ \Lambda_1+ \mathfrak{a}_1 + \mathfrak{b}_1 + \Lambda_1\mathfrak{a}_1 \big] \left(\log^{-1}2 \times \log \left\{ 3 \left[  1 + \frac{\sqrt{\pi \mathfrak{b}_1}}{2}  + \frac{\sqrt{2 \pi} \Lambda_1\mathfrak{a}_1}{2 \mathfrak{b}_1^2} + \frac{2 \mathfrak{a}_1}{\mathfrak{b}_1} \right]  \right\}  + 3\right) \log T \\
        & \geq \big[ \Lambda_1+ \mathfrak{a}_1 + \mathfrak{b}_1 + \Lambda_1\mathfrak{a}_1 \big] \log \left\{ 3 \left[  1 + \frac{\sqrt{\pi \mathfrak{b}_1}}{2}  + \frac{\sqrt{2 \pi} \Lambda_1\mathfrak{a}_1}{2 \mathfrak{b}_1^2} + \frac{2 \mathfrak{a}_1}{\mathfrak{b}_1} \right] T^3 \right\},
    \end{aligned}
\]
we have
\[
    \E\exp \left\{ - \frac{3 C_2^2 (1 + 2\lambda)^2 \Delta_k^2 s }{ 2 \sigma_{\omega}^2  \left[ 2 \lambda^2 \overline{\xi} +  2 \lambda^2 \sigma_1^2 + (1 + 2\lambda)^2 C_2^2 \Delta_k^2  \right] }\right\}  \leq \frac{s}{3 T^3} \leq \frac{1}{3T^2}.
\]
Similarly, 
\[
    T \left[ \exp \left\{ - \frac{(\lambda^2 + \lambda + 1 / 4)^2 s}{18 \sigma_{\omega}^2 (2 \mu_1 - 1)^2}\right\} \times \ind\left( \mu_1 \neq \frac{1}{2}\right) + 0 \times \ind\left( \mu_1 = \frac{1}{2}\right) \right] \leq \frac{1}{3T}, 
\]
and
\[
    \frac{T}{2} \exp \left\{ - \frac{s (1 + 2 \lambda)^2}{2 \sigma_{\omega}^2 (1 + 2 \lambda^2)}\right\} \leq \frac{1}{3T}
\]
have the solution
\[
    s \geq \frac{18 \sigma_{\omega}^2 (2 \mu_1 - 1)^2}{(\lambda^2 + \lambda + 1 / 4)^2} \left[ \log 3 + 2 \log T \right].
\]
and
\[
    s \geq \frac{2 \sigma_{\omega}^2 (1 + 2 \lambda^2)}{(1 + 2 \lambda)^2} \left[ \log (3 / 2) + 2 \log T \right]
\]
respectively.
Therefore, one will have $a_{k, s, 2} \leq \frac{1}{T}$  if we take 
\[
    \begin{aligned}
        s_{a, 2}(T) = \Bigg[ \big[ \Lambda_1+ \mathfrak{a}_1 + \mathfrak{b}_1 + \Lambda_1\mathfrak{a}_1 \big] & \left(\frac{1}{3}\log^{-1}2 \times \log \left\{ 3 \left[  1 + \frac{\sqrt{\pi \mathfrak{b}_1}}{2}  + \frac{\sqrt{2 \pi} \Lambda_1\mathfrak{a}_1}{2 \mathfrak{b}_1^2} + \frac{2 \mathfrak{a}_1}{\mathfrak{b}_1} \right]  \right\}  + 1\right)\\
        &  \vee \frac{18 \sigma_{\omega}^2 (2 \mu_1 - 1)^2}{(\lambda^2 + \lambda + 1 / 4)^2} \vee \frac{2 \sigma_{\omega}^2 (1 + 2 \lambda^2)}{(1 + 2 \lambda)^2} \Bigg] \times 3 \log T
    \end{aligned}
\]
for any $T \geq 2$.
\end{proof}

\begin{lemma}[Bounding $a_{k, s, 3}$ at \eqref{def_a_k_s_3}]\label{lem_bounding_a_k_s_3}
    Take 
    \[
        \begin{aligned}
             s_{a, 3}(T) = \Bigg[ \big[ \Lambda_1+ \mathfrak{a}_2 + \mathfrak{b}_2 + \Lambda_1\mathfrak{a}_2 \big] & \left(\frac{1}{3}\log^{-1}2 \times \log \left\{ 3 \left[  1 + \frac{\sqrt{\pi \mathfrak{b}_2}}{2}  + \frac{\sqrt{2 \pi} \Lambda_1\mathfrak{a}_2}{2 \mathfrak{b}_2^2} + \frac{2 \mathfrak{a}_2}{\mathfrak{b}_2} \right]  \right\}  + 1\right) \\
             & \vee \frac{18 \sigma_{\omega}^2 (2 \mu_1 - 1)^2}{(\lambda^2 + \lambda + 1 / 4)^2} \vee \frac{2 \sigma_{\omega}^2 (1 + 2 \lambda^2)}{(1 + 2 \lambda)^2} \Bigg] \times 3 \log T
        \end{aligned}
    \]
    where 
    \[
        \mathfrak{a}_2 = \frac{\sigma_{\omega}^2 \lambda^2}{3 (\lambda - 1)^2 C_1^2 \Delta_k^2}, \qquad \mathfrak{b}_2 = \frac{\sigma_{\omega}^2\left[ 2 \lambda^2 \sigma_1^2 + 25 (1 + 2\lambda)^2 C_1^2 \Delta_k^2 \right]}{6 (\lambda - 1)^2 C_1^2 \Delta_k^2}, \qquad \Lambda_1= 8 \sqrt{2} \sigma_1^2.
    \]
    Then for any $s \geq s_{a, 3}(T)$ and $\lambda > 1$,
    \[
        a_{k, s, 3} \leq T^{-1}.
    \]
    for any $T \geq 2$.
\end{lemma}

\begin{proof}
    Different from the proofs for bounding $a_{k, s, 1}$ and $a_{k, s, 2}$, in which we control the probability of the bad events,
$a_{k, s, 3}$ is defined on good event instead. We need a different technique to deal with $a_{k, s, 3}$. Note that 
\begin{equation}\label{a_k_3_sw}
    \begin{aligned}
        \{N_{1, s}(\tau_{k})<T^{-1}\} & \supseteq \{N_{1, s}(\tau_{k}) < (T^{2}-1)^{-1}\} \\
        \alignedoverset{\text{by } (T^2 - 1)^{-1} \leq T^{-1}}{=} \big\{Q_{1, s}\left(\tau_{k}\right)^{-1}<1+\left(T^{2}-1\right)^{-1}\big\} \\
        & = \big\{Q_{1, s}(\tau_{k})^{-1} < (1-T^{-2})^{-1}\big\} = \big\{\left[1-Q_{1, s}(\tau_{k})\right]<T^{-2}\big\},
    \end{aligned}
\end{equation}
and $a_{k, s, 3} \leq \E \big[ N_{1, s} (\tau_k) \ind (A_{1, s} \cap G_{1, s}) \big]$. 
Thus, similar to 
\citet{wang2020residual}, to bound $a_{k, s, 3} \leq T^{-1}$ for $s \geq s_{a. 3}$, it is sufficient to find $s \geq s_{a,3}$ satisfying $\big\{[1-Q_{1,s}(\tau_{k})]<T^{-2}\big\}$ on the event $\{A_{1, s} \cap G_{1, s}\}$.

We write 
\[
    \big[1-Q_{1, s}(\tau_{k})\big]  \ind (A_{1, s})\ind  (G_{1,s}) = \pr \big(\overline{Y}_{1, s} - \overline{R}_{1, s}^* \leq \Gamma_1 \, | \, \his_{1, s} \big)  \ind (A_{1, s} \, \cap \, G_{1, s})
\]
with $\Gamma_1 := \Gamma_1(\tau_1) = \tau_1 - \overline{R}_{1, s}^*$ measuring the difference between $\overline{Y}_{1, s}$ and $\overline{R}_{1,s}^*$. 
On the event $A_{1, s} \, \cap \, G_{1, s}$, $\Gamma_1$ can be bounded in an interval.
To see this, 
\[
    \begin{aligned}
    	A_{1, s} & = \bigg\{ \overline{R}_{1, s}^* - \frac{\mu_1 + \lambda}{1 + 2\lambda} > - C_1 \Delta_k \bigg\} \, \cap \, \bigg\{ \overline{R}_{k, s}^* - \frac{\mu_k + \lambda}{1 + 2 \lambda}  < C_1 \Delta_k \bigg\} \\
    	& = \bigg\{ \Gamma_1 < C_1 \Delta_k - \frac{\mu_1 + \lambda}{1 + 2\lambda} + \tau_1 \bigg\} \, \cap \, \bigg\{ \Gamma_1 > - C_2 \Delta_k - \frac{\mu_1 + \lambda}{1 + 2\lambda} + \tau_1 \bigg\} \\
    	& = \bigg\{ \tau_1 - C_1 \Delta_k - \frac{\mu_1 + \lambda}{1 + 2\lambda}  < \Gamma_1 < \tau_1 + C_1 \Delta_k - \frac{\mu_1 + \lambda}{1 + 2\lambda}  \bigg\}
    \end{aligned}
\]
and
\[
    \begin{aligned}
        G_{1, s} & = \big\{ - C_2 \Delta_k \leq \overline{Y}_{1, s} - \overline{R}_{1, s}^* \leq C_2 \Delta_k \big\} \\
        & = \big\{ \tau_1 -C_2 \Delta_k \leq \overline{Y}_{1, s} + \Gamma_1 \leq \tau_1 + C_2 \Delta_k \big\} \\
        & = \big\{ \tau_1 - C_2 \Delta_k - \overline{Y}_{1, s} \leq \Gamma_1 \leq \tau_1 + C_2 \Delta_k - \overline{Y}_{1, s}\big\} \\
        \alignedoverset{\text{by } - C_2 \Delta_k + \overline{R}_{1, s}^* \leq \overline{Y}_{1, s}  \leq C_2 \Delta_k + \overline{R}_{1, s}^*}{\subseteq} \big\{ \tau_1 - 2 C_2 \Delta_k - \overline{R}_{1, s}^* \leq \Gamma_1 \leq \tau_1 + 2 C_2 \Delta_k - \overline{R}_{1, s}^* \big\}.
    \end{aligned}
\]
Then combine the above two results and take 
\begin{equation}\label{C_1_2}
    C_1 = 2 C_2,
\end{equation}
we will have
\begin{equation}\label{Gamma_bound}
    \begin{aligned}
    	& A_{1, s} \, \cap \, G_{1, s}  \\
     & \subseteq \bigg\{ \tau_1 - C_1 \Delta_k - \overline{R}_{1, s}^* \wedge \frac{\mu_1 + \lambda}{1 + 2 \lambda} \leq \Gamma_1 \leq \tau_1 + C_1 \Delta_k - \overline{R}_{1, s}^* \vee \frac{\mu_1 + \lambda}{1 + 2 \lambda}  \bigg\} \\
    	\alignedoverset{\text{by} \frac{\mu_1 +  \lambda}{1 + 2 \lambda} - C_1 \Delta_k  < \overline{R}_{1, s}^* < \frac{\mu_1 +  \lambda}{1 + 2 \lambda} + C_1 \Delta_k }{\subseteq} \bigg\{\tau_1 - 2 C_1 \Delta_k - \frac{\mu_1 + \lambda}{1 + 2 \lambda} \leq \Gamma_1 \leq \tau_1 + 2 C_1 \Delta_k - \frac{\mu_1 + \lambda}{1 + 2 \lambda} \bigg\}.
    \end{aligned}
\end{equation}
Now take
\[
    \tau_1 :=  \frac{\mu_1 + \lambda}{1 + 2\lambda} - \frac{6 \lambda C_1 \Delta_k}{1 + 2\lambda}.
\]
then $ \tau_1 \leq \frac{\mu_1 + \lambda}{1 + 2 \lambda}$, and thus this choice satisfying \eqref{requirement_0}.
We furthermore get on $A_{1, s} \, \cap \, G_{1, s}$, 
\[
    \begin{aligned}
        \Gamma_1 & \leq \tau_1 + 2 C_1 \Delta_k - \frac{\mu_1 + \lambda}{1 + 2\lambda} \\
        & = 2 C_1 \Delta_k - \frac{6 \lambda C_1 \Delta_k}{1 + 2\lambda} = - 2 C_1 \Delta_k \frac{\lambda - 1}{2 \lambda + 1} \\
        \alignedoverset{\text{by } \lambda \geq 1}{\leq} - \frac{2 C_1 (\lambda - 1)  \Delta_k}{2 \lambda + 1} < 0,
    \end{aligned}
\]
and
\[
    \begin{aligned}
        0 > \Gamma_1 & \geq \tau_1 - 2 C_1 \Delta_k - \frac{\mu_1 + \lambda}{1 + 2\lambda} \\
        & = - 2 C_1 \Delta_k - \frac{6 \lambda C_1 \Delta_k}{1 + 2\lambda} = - 2 C_1 \Delta_k \left[ 1 + \frac{3 \lambda}{1 + 2 \lambda}\right] \\
        \alignedoverset{\text{by } \frac{3 \lambda}{1 + 2\lambda} \leq \frac{3}{2}}{\geq} - 2 C_1 \Delta_k \cdot \frac{5}{2} = - 5 C_1 \Delta_k.
    \end{aligned}
\]
Thus, we get $ - \Gamma_1 \in \left( \frac{2 C_1 (\lambda - 1)  \Delta_k}{2 \lambda + 1}, 5 C_1 \Delta_k \right) \subseteq (0, + \infty)$. Return to the quantity we care about, $ \big[ 1 - Q_{1,s} (\tau_k)\big] \ind (A_{1, s}) \ind (G_{1, s})$, we write it as 
\[
    \begin{aligned}
        & \big[ 1 - Q_{1,s} (\tau_k)\big] \ind (A_{1, s}) \ind (G_{1, s}) \\
        = & \pr \big(\overline{Y}_{1, s} - \overline{R}_{1, s}^* \leq \Gamma_1 \, | \, \his_{1, s} \big)  \ind (A_{1, s} \, \cap \, G_{1, s})\\
        = & \E \Bigg[ \pr_{\omega, \omega', \omega''} \Bigg\{ \sum_{i = 1}^s (R_{1, i} - \overline{R}_{1, s}) \omega_i + \lambda \sum_{i = 1}^s (2 R_{1, i} \omega_i - \overline{R}_{1, s} \omega_i' - \overline{R}_{1, s} \omega_i'') 
        \\
        & ~~~~~~~~~~~~~~~~~~~~~~~~~~~~~~ + \lambda \sum_{i = 1}^s (\omega_i' - \omega_i) + \lambda^2 \sum_{i = 1}^s (\omega_i' - \omega''_i) \\
        & ~~~~~~~~~~~~~~~~~~~~~~~~~~~~~~ - \Gamma_1 (1 + 2\lambda) \bigg[ \sum_{i = 1}^s \omega_i + \lambda \sum_{i = 1}^s \omega_i' + \lambda \sum_{i = 1}^s \omega_i''\bigg] < 0\Bigg\} \ind (A_{1, s} \, \cap \, G_{1, s}) \Bigg] .
    \end{aligned}
\]
Similar to the technique we use to bound the conditional probability $\pr (G_{1, s}^c \mid \his_{1, s})$, we just need to replace $C_2 \Delta_k$ by $-\Gamma_1$ in \eqref{a_2_bound} and get that 
\begin{equation}\label{Gamma_1_bound_next}
    \begin{aligned}
         & \big[ 1 - Q_{1,s} (\tau_k)\big] \ind (A_{1, s}) \ind(G_{1, s}) \\
         \alignedoverset{\text{apply steps in  \eqref{pr_G_bound}}}{\leq} T \E \Bigg[ \exp \left\{ - \frac{9 (- \Gamma_1)^2 (1 + 2 \lambda)^2 s^2 }{6  \sigma_{\omega}^2 s  \left[ 2 \lambda^2 \overline{\xi} +  2 \lambda^2 \sigma_1^2 + (- \Gamma_1)^2 (1 + 2 \lambda)^2 \right] }\right\}  \ind (A_{1, s} \, \cap \, G_{1, s}) \Bigg] \\
         & ~~~~~~~~~~~~~~~ + T \left[ \exp \left\{ - \frac{(\lambda^2 + \lambda + 1 / 4)^2 s}{18 \sigma_{\omega}^2 (2 \mu_1 - 1)^2}\right\} \times \ind\left( \mu_1 \neq \frac{1}{2}\right) + 0 \times \ind\left( \mu_1 = \frac{1}{2}\right) \right] \\
          & ~~~~~~~~~~~~~~~ + \frac{T}{2} \exp \left\{ - \frac{s (1 + 2 \lambda)^2}{2 \sigma_{\omega}^2 (1 + 2 \lambda^2)}\right\} \\
         \alignedoverset{\text{by the bound of } -\Gamma_1}{\leq} T \E\exp \left\{ - \frac{36 (\lambda - 1)^2 C_1^2 \Delta_k^2 s^2 }{6  \sigma_{\omega}^2 s  \left[ 2 \lambda^2 \overline{\xi} +  2 \lambda^2 \sigma_1^2 + 25 (1 + 2 \lambda)^2 C_1^2 \Delta_k^2  \right] }\right\}  \\
         & ~~~~~~~~~~~~~~~ + T \left[ \exp \left\{ - \frac{(\lambda^2 + \lambda + 1 / 4)^2 s}{18 \sigma_{\omega}^2 (2 \mu_1 - 1)^2}\right\} \times \ind\left( \mu_1 \neq \frac{1}{2}\right) + 0 \times \ind\left( \mu_1 = \frac{1}{2}\right) \right] \\
          & ~~~~~~~~~~~~~~~ + \frac{T}{2} \exp \left\{ - \frac{s (1 + 2 \lambda)^2}{2 \sigma_{\omega}^2 (1 + 2 \lambda^2)}\right\} .
    \end{aligned}
\end{equation}
%
Similarly, we define
\[
    \mathfrak{a}_2 = \frac{\sigma_{\omega}^2 \lambda^2}{3 (\lambda - 1)^2 C_1^2 \Delta_k^2}, \qquad \mathfrak{b}_2 = \frac{\sigma_{\omega}^2\left[ 2 \lambda^2 \sigma_1^2 + 25 (1 + 2\lambda)^2 C_1^2 \Delta_k^2 \right]}{6 (\lambda - 1)^2 C_1^2 \Delta_k^2}.
\]
Take
\[
    s \geq \big[ \Lambda_1+ \mathfrak{a}_2 + \mathfrak{b}_2 + \Lambda_1\mathfrak{a}_2 \big] \left(\log^{-1}2 \times \log \left\{ 3 \left[  1 + \frac{\sqrt{\pi \mathfrak{b}_2}}{2}  + \frac{\sqrt{2 \pi} \Lambda_1\mathfrak{a}_2}{2 \mathfrak{b}_2^2} + \frac{2 \mathfrak{a}_2}{\mathfrak{b}_2} \right]  \right\}  + 3\right) \log T,
\]
then applying Lemma \ref{lem_subE_ineq} and the steps in \eqref{app_sub_E_ineq_1} again, 
\[
     \begin{aligned}
         & T \E\exp \left\{ - \frac{36 (\lambda - 1)^2 C_1^2 \Delta_k^2 s^2 }{6  \sigma_{\omega}^2 s  \left[ 2 \lambda^2 \overline{\xi} +  2 \lambda^2 \sigma_1^2 + 25 (1 + 2 \lambda)^2 C_1^2 \Delta_k^2  \right] }\right\} \\
         & = T \E\exp \left\{ - \frac{6 (\lambda - 1)^2 C_1^2 \Delta_k^2 s}{\sigma_{\omega}^2 \left[ 2 \lambda^2 \overline{\xi} +  2 \lambda^2 \sigma_1^2 + 25 (1 + 2 \lambda)^2 C_1^2 \Delta_k^2  \right] }\right\} \\
         & = T \E \exp \left\{- \frac{s}{\frac{\sigma_{\omega}^2 \lambda^2}{3 (\lambda - 1)^2 C_1^2 \Delta_k^2} \overline{\xi} + \frac{\sigma_{\omega}^2\left[ 2 \lambda^2 \sigma_1^2 + 25 (1 + 2\lambda)^2 C_1^2 \Delta_k^2 \right]}{6 (\lambda - 1)^2 C_1^2 \Delta_k^2}} \right\} = T \E \exp \left\{ -\frac{s}{\mathfrak{a}_2 \overline{\xi} + \mathfrak{b}_2}\right\} \\
         \alignedoverset{\text{by applying steps in \eqref{app_sub_E_ineq_1}}}{\leq} \frac{1}{3T}
     \end{aligned}
\]
Therefore, to find $s$ such that $[1 - Q_{1, s} (\tau_k)] < T^{-2}$ on $\{ A_{1, s} \, \cap \, G_{1, s}\}$, we let
\[
   \begin{aligned}
       s \geq s_{a, 3}(T) = \Bigg[ \big[ \Lambda_1+ \mathfrak{a}_2 + \mathfrak{b}_2 + \Lambda_1\mathfrak{a}_2 \big] & \left(\frac{1}{3}\log^{-1}2 \times \log \left\{ 3 \left[  1 + \frac{\sqrt{\pi \mathfrak{b}_2}}{2}  + \frac{\sqrt{2 \pi} \Lambda_1\mathfrak{a}_2}{2 \mathfrak{b}_2^2} + \frac{2 \mathfrak{a}_2}{\mathfrak{b}_2} \right]  \right\}  + 1\right) \\
       & \vee \frac{18 \sigma_{\omega}^2 (2 \mu_1 - 1)^2}{(\lambda^2 + \lambda + 1 / 4)^2} \vee \frac{2 \sigma_{\omega}^2 (1 + 2 \lambda^2)}{(1 + 2 \lambda)^2} \Bigg] \times 3 \log T,
   \end{aligned}
\]
and then we get that $a_{s, 3} \leq T^{-1}$.
\end{proof}

\subsection{Lemmas on bounding $b_k$.}

\begin{lemma}[Bounding $b_{k, s, 1}$ at \eqref{def_b_k_s_1}]\label{lem_bounding_b_k_s_1}
    Take 
    \[
        s_{b, 1} (T) = \frac{2 \sigma_k^2}{C_1^2 (1 + 2 \lambda)^2 \Delta_k^2} \times\log T.
    \]
    Then for any $s \geq s_{b, 1}(T)$,
    \[
        b_{k, s, 1} \leq T^{-1},
    \]
    when $T \geq 2$.
\end{lemma}

\begin{proof}
    Similar to Lemma \ref{lem_bounding_b_k_s_1}, by noting 
    \[
        s_{b, 1}(T) \geq \frac{\sigma_k^2}{C_1^2 (1 + 2 \lambda)^2 \Delta_k^2} \times\log (2 T),
    \]
    the Hoeffding inequality gives
    \[
        \begin{aligned}
            b_{k, s, 1} & = \E  \big[ \ind(Q_{k, s} (\tau_k) > T^{-1} ) \ind (A_{k, s}^c)\big] \\
            & \leq \pr (A_{k, s}^c) \\
            & = \pr \left( |\overline{R}_{k, s} - \mu_k| \geq (1 + 2\lambda) C_1 \Delta_k \right) \overset{\text{by Hoeffding inequality}}{\leq} T^{-1}.
        \end{aligned}
    \]
\end{proof}

\begin{lemma}[Bounding $b_{k, s, 2}$ at \eqref{def_b_k_s_2}]\label{lem_bounding_b_k_s_2}
    There exists $s_{b, 2}(T) \asymp \log T$ such that 
    Take
    \[
        \begin{aligned}
            s_{b, 2}(T) = \Bigg[ \big[ \Lambda_k + \mathfrak{a}_1 + \mathfrak{b}_1 + \Lambda_k \mathfrak{a}_1 \big] & \left(\frac{1}{3}\log^{-1}2 \times \log \left\{ 3 \left[  1 + \frac{\sqrt{\pi \mathfrak{b}_1}}{2}  + \frac{\sqrt{2 \pi} \Lambda_k \mathfrak{a}_1}{2 \mathfrak{b}_1^2} + \frac{2 \mathfrak{a}_1}{\mathfrak{b}_1} \right]  \right\}  + 1\right) \\
            & \vee \frac{18 \sigma_{\omega}^2 (2 \mu_k - 1)^2}{(\lambda^2 + \lambda + 1 / 4)^2}  \vee \frac{2 \sigma_{\omega}^2 (1 + 2 \lambda^2)}{(1 + 2 \lambda)^2} \Bigg] \times 3 \log T,
        \end{aligned}
    \]
    where 
    \[
        \mathfrak{a}_1 = \frac{4 \sigma_{\omega}^2 \lambda^2}{3 C_2^2 (1 + 2\lambda)^2 \Delta_k^2} , \qquad \mathfrak{b}_1 = \frac{2 \sigma_{\omega}^2 \left[ 2 \lambda^2 \sigma_1^2 + (1 + 2 \lambda)^2 C_2^2 \Delta_k^2 \right]}{3 C_2^2 (1 + 2 \lambda)^2 \Delta_k^2}, \qquad \Lambda_k = 8 \sqrt{2} \sigma_k^2.
    \]
    Then for any $s \geq s_{b, 2}(T)$,
    \[
        b_{k, s, 2} \leq T^{-1}
    \]
    when $T \geq 2$.
\end{lemma}

\begin{proof}
    Similar to bounding $a_{k, s, 2}$ in Lemma \ref{lem_bounding_a_k_s_2}, if we take
    \[
        \begin{aligned}
            s_{b, 2}(T) = \Bigg[ \big[ \Lambda_k + \mathfrak{a}_1 + \mathfrak{b}_1 + \Lambda_k \mathfrak{a}_1 \big] & \left(\frac{1}{3}\log^{-1}2 \times \log \left\{ 3 \left[  1 + \frac{\sqrt{\pi \mathfrak{b}_1}}{2}  + \frac{\sqrt{2 \pi} \Lambda_k \mathfrak{a}_1}{2 \mathfrak{b}_1^2} + \frac{2 \mathfrak{a}_1}{\mathfrak{b}_1} \right]  \right\}  + 1\right) \\
            & \vee \frac{18 \sigma_{\omega}^2 (2 \mu_k - 1)^2}{(\lambda^2 + \lambda + 1 / 4)^2}  \vee \frac{2 \sigma_{\omega}^2 (1 + 2 \lambda^2)}{(1 + 2 \lambda)^2} \Bigg] \times 3 \log T,
        \end{aligned}
    \]
    with $\Lambda_k = 8 \sqrt{2} \sigma_k^2$, we have 
    \[
        \begin{aligned}
            b_{k, s, 2} & \leq  \E \big[ \ind(G_{k, s}^c) \big] \\
            & \leq \E\exp \left\{ - \frac{3 C_2^2 (1 + 2\lambda)^2 \Delta_k^2 s }{ 2 \sigma_{\omega}^2  \left[ 2 \lambda^2 \overline{\zeta^{(k)}}  +  2 \lambda^2 \sigma_k^2 + (1 + 2\lambda)^2 C_2^2 \Delta_k^2  \right] }\right\} \\
		& ~~~~~~~~ +  \left[ \exp \left\{ - \frac{(\lambda^2 + \lambda + 1 / 4)^2 s}{18 \sigma_{\omega}^2 (2 \mu_k - 1)^2}\right\} \times \ind\left( \mu_k \neq \frac{1}{2}\right) + 0 \times \ind\left( \mu_k = \frac{1}{2}\right) \right] + \frac{T}{2} \exp \left\{ - \frac{s (1 + 2 \lambda)^2}{2 \sigma_{\omega}^2 (1 + 2 \lambda^2)}\right\} \leq T^{-2},
        \end{aligned}
    \]
    where $\overline{\zeta^{(k)}} = \frac{1}{s} \sum_{i = 1}^s \zeta^{(k)}_i$ with $\zeta^{(k)}_i$ are iid independent sub-Exponential variables such that $\zeta^{(k)}_i \sim \subE(8 \sqrt{2} \sigma_k^2)$. Then $b_{k, s, 2} \leq T^{-2} \leq T^{-1}$.
\end{proof}

\begin{lemma}[Bounding $b_{k, s, 3}$ at \eqref{def_b_k_s_3}]\label{lem_bounding_b_k_s_3}
    Take
    \[
        \begin{aligned}
            s_{b, 3} (T) = \Bigg[\big[ \Lambda_k + \mathfrak{a}_2 + \mathfrak{b}_2 + \Lambda_k \mathfrak{a}_2 \big] & \left(\frac{1}{3}\log^{-1}2 \times \log \left\{ 3 \left[  1 + \frac{\sqrt{\pi \mathfrak{b}_2}}{2} + \frac{\sqrt{2 \pi} \Lambda_k\mathfrak{a}_2}{2 \mathfrak{b}_2^2} + \frac{2 \mathfrak{a}_2}{\mathfrak{b}_2} \right]  \right\}  + 1\right) \\
            & \vee \frac{18 \sigma_{\omega}^2 (2 \mu_k - 1)^2}{(\lambda^2 + \lambda + 1 / 4)^2}  \vee \frac{2 \sigma_{\omega}^2 (1 + 2 \lambda^2)}{(1 + 2 \lambda)^2} \Bigg] \times 3 \log T,
        \end{aligned}
    \]
    where
    \[
        \mathfrak{a}_2 = \frac{\sigma_{\omega}^2 \lambda^2}{3 (\lambda - 1)^2 C_1^2 \Delta_k^2}, \qquad \mathfrak{b}_2 = \frac{\sigma_{\omega}^2\left[ 2 \lambda^2 \sigma_1^2 + 25 (1 + 2\lambda)^2 C_1^2 \Delta_k^2 \right]}{6 (\lambda - 1)^2 C_1^2 \Delta_k^2}, \qquad \Lambda_k = 8 \sqrt{2} \sigma_k^2.
    \]
    Then for any $s \geq s_{b, 3}(T)$ and $\lambda > 1$,
    \[
        b_{k, s, 3} \leq T^{-1},
    \]
    when $T \geq 2$.
\end{lemma}

\begin{proof}
    The basic idea is the same as bounding $a_{k, s, 3}$. The only difference is we replace $\big[ 1 - Q_{k, s} (\tau_k) \big]$ by $Q_{k, s} (\tau_k)$. Again, we will first bound $\Gamma_k = \overline{Y}_{k, s}  - \overline{R}_{k, s}^*$ on the event $A_{k, s} \, \cap \, G_{k, s}$. 
Exactly the same as before, we let
\[
    \tau_k = \frac{\mu_k + \lambda}{1 + 2 \lambda} + \frac{6 \lambda C_1 \Delta_k}{1 + 2 \lambda}
\]
which satisfies $\tau_k \geq \frac{\mu_k + \lambda}{1 + 2 \lambda} $. To ensure $\tau_k \leq \frac{\mu_1 + \lambda}{1 + 2 \lambda}$, one just need to take $C_1 = \frac{1}{6 \lambda}$, and then $C_2 = \frac{1}{2} C_1 =  \frac{1}{12 \lambda}$ by \eqref{C_1_2}. Next, we will obtain a bound for $\Gamma_k$ on $A_{k, s} \, \cap \, G_{k, s}$ as
\[
    \begin{aligned}
        \Gamma_k & \geq \tau_k - 2 C_1 \Delta_k - \frac{\mu_k + \lambda}{1 + 2 \lambda} \\
        & = 2 C_1 \Delta_k \frac{\lambda - 1}{2 \lambda + 1} \\
        \alignedoverset{\text{by } \lambda > 1}{>} 0,
    \end{aligned}
\]
and
\[
    \begin{aligned}
        \Gamma_k & \leq \tau_k + 2 C_1 \Delta_k - \frac{\mu_k + \lambda}{1 + 2 \lambda} \\
        & = 2 C_1 \Delta_k \left( 1 + \frac{3 \lambda}{1 + 2 \lambda}\right) \\
        \alignedoverset{\text{by } \frac{3 \lambda}{1 + 2\lambda} \geq \frac{3}{2}}{\leq} 2 C_1 \Delta_k \left( 1 + \frac{3}{2}\right) = 5 C_1 \Delta_k,
    \end{aligned}
\]
i.e., $\Gamma_k \in \left(\frac{2 C_1 (\lambda - 1) \Delta_k}{2 \lambda + 1}, 5 C_1 \Delta_k \right) \subseteq (0, \infty)$. Therefore,
\[
    \begin{aligned}
         &  Q_{k,s} (\tau_k) \ind (A_{k, s}) \ind(G_{k, s}) \\
         & = \pr \left( \overline{Y}_{k, s} - \overline{R}_{k, s}^* \geq \Gamma_k \mid \his_{k, s} \right)  \ind (A_{k, s} \, \cap \, G_{k, s}) \\
         \alignedoverset{\text{apply steps in  \eqref{pr_G_bound}}}{\leq} T \E \Bigg[ \exp \left\{ - \frac{9 \Gamma_k^2 (1 + 2 \lambda)^2 s^2 }{6  \sigma_{\omega}^2 s  \left[ 2 \lambda^2 \overline{\zeta^{(k)}} +  2 \lambda^2 \sigma_1^2 + \Gamma_k^2 (1 + 2 \lambda)^2 \right] }\right\}  \ind (A_{k, s} \, \cap \, G_{k, s})  \Bigg] \\ & ~~~~~~~~~~~~~~~ + T \left[ \exp \left\{ - \frac{(\lambda^2 + \lambda + 1 / 4)^2 s}{18 \sigma_{\omega}^2 (2 \mu_k - 1)^2}\right\} \times \ind\left( \mu_k \neq \frac{1}{2}\right) + 0 \times \ind\left( \mu_k = \frac{1}{2}\right) \right] \\
          & ~~~~~~~~~~~~~~~ + \frac{T}{2} \exp \left\{ - \frac{s (1 + 2 \lambda)^2}{2 \sigma_{\omega}^2 (1 + 2 \lambda^2)}\right\} \\
         \alignedoverset{\text{by the bound of } \Gamma_k}{\leq} T \E\exp \left\{ - \frac{36 (\lambda - 1)^2 C_1^2 \Delta_k^2 s^2}{6  \sigma_{\omega}^2 s  \left[ 2 \lambda^2 \overline{\zeta^{(k)}} +  2 \lambda^2 \sigma_1^2 + 25 (1 + 2 \lambda)^2 C_1^2 \Delta_k^2  \right] }\right\} \\
         & ~~~~~~~~~~~~~~~ + T \left[ \exp \left\{ - \frac{(\lambda^2 + \lambda + 1 / 4)^2 s}{18 \sigma_{\omega}^2 (2 \mu_k - 1)^2}\right\} \times \ind\left( \mu_k \neq \frac{1}{2}\right) + 0 \times \ind\left( \mu_k = \frac{1}{2}\right) \right] \\
          & ~~~~~~~~~~~~~~~ + \frac{T}{2} \exp \left\{ - \frac{s (1 + 2 \lambda)^2}{2 \sigma_{\omega}^2 (1 + 2 \lambda^2)}\right\} .
    \end{aligned}
\]
where $\overline{\zeta^{(k)}} = \frac{1}{s} \sum_{i = 1}^s \zeta^{(k)}_i$ with $\zeta^{(k)}_i$ are iid independent sub-Exponential variables such that $\zeta^{(k)}_i \sim \subE(8 \sqrt{2} \sigma_k^2)$.
Take
\[
    s \geq \big[ \Lambda_k + \mathfrak{a}_2 + \mathfrak{b}_2 + \Lambda_k \mathfrak{a}_2 \big] \left(\frac{1}{3}\log^{-1}2 \times \log \left\{ 3 \left[  1 + \frac{\sqrt{\pi \mathfrak{b}_2}}{2} + \frac{\sqrt{2 \pi} \Lambda_k\mathfrak{a}_2}{2 \mathfrak{b}_2^2} + \frac{2 \mathfrak{a}_2}{\mathfrak{b}_2} \right]  \right\}  + 1\right),
\]
by applying Lemma \ref{lem_subE_ineq} again,
\[
     \begin{aligned}
         & T \E\exp \left\{ - \frac{36 (\lambda - 1)^2 C_1^2 \Delta_k^2 s^2}{6  \sigma_{\omega}^2 s  \left[ 2 \lambda^2 \overline{\zeta^{(k)}} +  2 \lambda^2 \sigma_1^2 + 25 (1 + 2 \lambda)^2 C_1^2 \Delta_k^2  \right] }\right\} \\
         = & T \E \exp \left\{ - \frac{s}{\mathfrak{a}_2  \overline{\zeta^{(k)}} + \mathfrak{b}_2}\right\} \leq \frac{1}{3T}.
     \end{aligned}
\]
Therefore, we have $b_{k, s, 3} \leq T^{-1}$ if we take
\[
    \begin{aligned}
         s \geq s_{b, 3} (T) = \Bigg[\big[ \Lambda_k + \mathfrak{a}_2 + \mathfrak{b}_2 + \Lambda_k \mathfrak{a}_2 \big] & \left(\frac{1}{3}\log^{-1}2 \times \log \left\{ 3 \left[  1 + \frac{\sqrt{\pi \mathfrak{b}_2}}{2} + \frac{\sqrt{2 \pi} \Lambda_k\mathfrak{a}_2}{2 \mathfrak{b}_2^2} + \frac{2 \mathfrak{a}_2}{\mathfrak{b}_2} \right]  \right\}  + 1\right) \\
         & \vee \frac{18 \sigma_{\omega}^2 (2 \mu_k - 1)^2}{(\lambda^2 + \lambda + 1 / 4)^2}  \vee \frac{2 \sigma_{\omega}^2 (1 + 2 \lambda^2)}{(1 + 2 \lambda)^2} \Bigg] \times 3 \log T.
    \end{aligned}
\]
\end{proof}


\section{Technical Lemmas}\label{sec_lem_tech}

\begin{lemma}\label{lemma:gaussian_weight}
	Let $Z$ be a standard Gaussian variable, then the tail probability $\pr(Z > x)$ satisfies
	\[
	    \frac{1}{4} \exp (- x^2 ) < \pr(Z > x) \leq \frac{1}{2} \exp ( - {x^2} / {2} )
	\]
	for any $x \geq 0$. 
\end{lemma}

\begin{proof}
	Denote $Q(x) := \pr (Z > x)$ for $x > 0$ be the tail probability for a standard Gaussian variable $Z$. Let $Z_1$ and $Z_2$ are two independent standard Gaussian random variables, then
	\[
	    \begin{aligned} 
	        \pr (Z_1 \leq x, \, Z_2 \leq x) &=\int_{-x}^x \int_{-x}^x \frac{1}{2 \pi} \exp \left[\left(-z_1^2-z_2^2\right) / 2\right] \, \mathrm{d} z_1 \mathrm{d} z_2 \\ 
	        & \leq \int_0^{\sqrt{2} x} \int_0^{2 \pi} \frac{1}{2 \pi} \exp \left(-r^2 / 2\right) r \, \mathrm{d} \theta \, \mathrm{d} r \\ 
	        &=1-\exp (-x^2)
	    \end{aligned}
	\]
	giving $[1-2 Q(x)]^2 \leq 1-\exp (-x^2)$ for $x \geq 0$, or, equivalently, $\exp (-x^2) \leq 4 Q(x)-4 Q^2(x)$. But, since $4 Q^2(x)>0$ for all $x$, we get that
	\[
	    Q(x) > \frac{1}{4} \exp (-x^2)
	\]
    which gives the lower bound in the lemma. For the upper bound, we note that
    \[
        \begin{aligned}
            \exp \left(x^2 / 2\right) Q(x) &=\exp \left(x^2 / 2\right) \int_x^{\infty}(2 \pi)^{-1 / 2} \exp \left(-t^2 / 2\right) \mathrm{d} t \\
            &=\int_x^{\infty}(2 \pi)^{-1 / 2} \exp \left(-\left(t^2-x^2\right) / 2\right) \mathrm{d} t \\
            &<\int_x^{\infty}(2 \pi)^{-1 / 2} \exp \left(-(t-x)^2 / 2\right) \mathrm{d} t =\frac{1}{2},
         \end{aligned}
    \]
    which gives the upper bound.
\end{proof}

\begin{lemma}\label{lem_subGtosubE}
    Suppose $\{X_i\}_{i = 1}^n$ are i.i.d. zero-mean sub-Gaussian with variance proxy $\sigma^2$, and $\overline{X} = \frac{1}{n} \sum_{i = 1}^n X_i$, then
    \[
        \E \exp\big( \lambda\overline{X}^2 \big) \leq e^{9 / 8}
    \]
    for any $|\lambda| \leq \frac{n}{8 \sigma^2}$.
\end{lemma}

\begin{proof}
    Note that $X - \mu \sim \subG (\sigma^2)$ implies
    \begin{equation}\label{lem_subGtosubE_1}
        \overline{X} - \mu \sim \subG (\sigma^2 / n)
    \end{equation}
    by the fact that
    \[
        \begin{aligned}
        	\E \exp \left\{ s (\overline{X} - \mu) \right\} & = \prod_{i = 1}^n \exp \left\{ \frac{s}{n} ({X}_i - \mu) \right\} \\
        	(\text{by } X_i - \mu \text{ is sub-Gaussian}) & \leq \prod_{i = 1}^n \exp \left\{ \frac{s^2}{n^2} \frac{\sigma^2}{2}\right\} = \exp \left\{ \frac{s^2 (\sigma^2 / n)}{2}\right\}.
        \end{aligned}
    \]
    Then by Proposition 4.3 in \citet{zhang2020concentration}, $(\overline{X} - \mu)^2 - n^{-1} \var (X) \sim \subE (8\sqrt{2}\sigma^2 / n, 8\sigma^2 / n)$. Therefore,
    \[
        \begin{aligned}
            \E \exp\left\{ \lambda(\overline{X} - \mu)^2 \right\} & = \E \exp\left\{ \lambda \big( (\overline{X} - \mu)^2 - n^{-1}\var(X) \big) \right\} \E \exp \{ \lambda n^{-1}  \var(X)\} \\
            & \leq \exp \left\{ \frac{(8 \sqrt{2} \sigma^2 / n)^2 \lambda^2}{2}\right\} \cdot \exp \left\{ \lambda n^{-1} \sigma^2 \right\} = \exp \left\{ \frac{64 \lambda^2 \sigma^4}{n^2} + \frac{\lambda \sigma^2}{n}\right\} \leq e^{9 / 8}
        \end{aligned}
    \]
    for any $\lambda \leq \frac{n}{8 \sigma^2}$.
\end{proof}

\begin{lemma}\label{lem_gau_frac}
	Suppose $X$ and $Y$ are Gaussian variables satisfying $\E X = 0$ and $\E Y > 0$, then
	\[
	    \pr \left(\frac{|X|}{|Y|} > c \right) \leq 2 \pr (X > c Y) + \pr \left( Y < 0 \right)
	\]
	for any $c > 0$.
\end{lemma} 

\begin{proof}
        We have
	\[
	    \begin{aligned}
	    	\pr \left(\frac{|X|}{|Y|} > c \right) & = \pr \left(|X| > c |Y| \right) \\
	    	& = \pr \left( \{ X > c Y \} \, \cap \, \{ X > 0\} \, \cap \, \{ Y > 0 \} \right) + \pr \left( \{ - X > c Y \} \, \cap \, \{ X \leq 0\} \, \cap \, \{ Y > 0 \} \right) \\
	    	& ~~~~~~~~~~ + \pr \left( \{ X >-  c Y \} \, \cap \, \{ X > 0\} \, \cap \, \{ Y \leq 0 \} \right) + \pr \left( \{ - X > - c Y \} \, \cap \, \{ X \leq 0\} \, \cap \, \{ Y \leq 0 \}\right) \\
	    	& = \pr \left( \{ X > c Y \} \, \cap \, \{ X > 0\} \, \cap \, \{ Y > 0 \} \right)  + \pr \left( \{ - X > c Y \} \, \cap \, \{ - X \geq 0\} \, \cap \, \{ Y > 0 \} \right) + \pr (Y \leq 0) \\
	    	\alignedoverset{\text{by using $X$ is symmetric about $0$.}}{\leq} 2 \pr (  X > c Y ) + \pr (Y \leq 0) 
	    \end{aligned}
	\]
	for any $c > 0$.
\end{proof}

\begin{lemma}[sub-Exponential concentration]\label{lem_subE_ineq}
    Consider zero-mean independent random variables $ X_i \sim \subE(\lambda_i)$, and positive constants $a$ and $b$ satisfy $a \overline{X} + b > 0$ with the mean $\overline{X} := \frac{1}{s} \sum_{i = 1}^s X_i$. Then,
    \[
        \E \exp \left\{ - \frac{s}{a \overline{X} + b}\right\} \leq \left[  1 +  \frac{\sqrt{\pi b}}{2}  + \frac{\sqrt{2 \pi} \lambda a}{2 b^2} + \frac{2 a}{b} \right] s \exp \bigg\{ - \frac{s}{\sqrt{2 \lambda a} \vee (b + \lambda a)}\bigg\},
    \]
    for $s \in \mathbb{N}$ such that $2 \lambda a s / b \geq 1$, where $\lambda := \left(\frac{1}{s} \sum_{i=1}^s \lambda_i^2\right)^{1 / 2}$. Specially, we have $\log \E \exp \left\{ - \frac{s}{a \overline{X} + b}\right\} \lesssim -s$.
\end{lemma}

\begin{proof}
    Denote $Y := a \overline{X} + b$ as a strictly positive random variable. For any non-negative strictly increasing function $f$, we have
    \[
        \E f(Y) = \int_0^{\infty} \pr \left( f(Y) > r\right) \, \mathrm{d} r = \int_0^{\infty} \pr \left( Y > f^{-1} (r)\right) \, \mathrm{d} r.
    \]
    Then for any fixed $s \in \mathbb{N}$,
    \[
        \begin{aligned}
        	\E \exp \left\{ - \frac{s}{a\overline{X} + b}\right\} & = \E \exp \left\{ - \frac{s}{Y} \right\} = \int_0^{\infty} \pr \left( \exp \left\{ - \frac{s}{Y}\right\} > r \right) \, \mathrm{d} r \\
        	& = \int_0^1 \pr \left( Y > \frac{s}{\log r^{-1}} \right) \, \mathrm{d} r \\
        	& = \int_0^1 \pr \left( \overline{X} > \frac{1}{a} \left[ \frac{s}{\log r^{-1}} - b \right]\right) \, \mathrm{d} r \\
        	& = \int_0^{\exp \left\{ -\frac{s}{b}\right\}} \pr \left( \overline{X} > \underbrace{\frac{1}{a} \left[ \frac{s}{\log r^{-1}} - b \right]}_{\leq 0}\right) \, \mathrm{d} r + \int_{\exp \left\{ -\frac{s}{b}\right\}}^{1} \pr \left( \overline{X} > \underbrace{\frac{1}{a} \left[ \frac{s}{\log r^{-1}} - b \right]}_{> 0}\right) \, \mathrm{d} r \\
        	& \leq \exp \left\{ -\frac{s}{b}\right\} \times 1 + \int_{\exp \left\{ -\frac{s}{b}\right\}}^{1} \pr \left( \overline{X} > \underbrace{\frac{1}{a} \left[ \frac{s}{\log r^{-1}} - b \right]}_{> 0}\right) \, \mathrm{d} r \\
        	\alignedoverset{\text{by letting $u = \frac{s}{\log r^{-1}}$}}{=} \exp \left\{ -\frac{s}{b}\right\} + \int_b^{\infty} \pr \left( \overline{X} > \underbrace{\frac{u - b}{a}}_{>0} \right) \exp\left\{ - \frac{s}{u}\right\} \frac{s}{u^2} \, \mathrm{d} r.
        \end{aligned} 
    \]
    Recall $\overline{X}$ is the average of independent zero-mean sub-exponential variables, denote $\lambda := \left(\frac{1}{s} \sum_{i=1}^s \lambda_i^2\right)^{1 / 2}$. 
    Then, we have 
    \[
        \pr \left( \overline{X} > t \right) \leq \exp \left\{ -\frac{1}{2}\left(\frac{s t^2}{{\lambda}^2} \wedge \frac{s t}{\lambda}\right) \right\}
    \]
    for any $t > 0$ by Corollary 4.2.(c) in \citet{zhang2020concentration}. Therefore, we can furthermore bound the expectation as
    \[
        \begin{aligned}
            \E \exp \left\{ - \frac{s}{a\overline{X} + b}\right\} \alignedoverset{\text{by sub-Exponential inequality}}{\leq} \exp \left\{ -\frac{s}{b}\right\} +\int_b^{\infty} \frac{s}{u^2} \exp \left\{ - \frac{s}{u}\right\} \exp \left\{- \frac{1}{2} \frac{s (u - b)^2}{\lambda^2 a^2} \wedge \frac{s (u - b)}{\lambda a} \right\} \, \mathrm{d} u \\
        	& \leq \exp \left\{ -\frac{s}{b}\right\} +\frac{s}{b^2} \int_b^{\infty} \exp \left\{- \left[ \frac{s}{u} + \frac{1}{2} \frac{s(u - b)^2}{\lambda^2 a^2} \wedge \frac{s(u - b)}{\lambda a} \right] \right\} \, \mathrm{d} u \\
        	& = \exp \left\{ -\frac{s}{b}\right\} + \left[ \int_b^{b + \lambda  a} + \int_{b + \lambda a }^{\infty} \right] \frac{s}{u^2} \exp \left\{- \left[ \frac{s}{u} + \frac{1}{2} \frac{s(u - b)^2}{\lambda^2 a^2} \wedge \frac{s(u - b)}{\lambda a} \right] \right\} \, \mathrm{d} u \\
        	& =: \exp \left\{ -\frac{s}{b}\right\} + \big[ I + II\big].
        \end{aligned}
    \]
    The next step is bounding both $I$ and $II$. For $I$, we obtain
    \[
        \begin{aligned}
        	I & = \int_b^{b + \lambda  a} \frac{s}{u^2}  \exp \left\{ - \left[ \frac{s}{u} + \frac{s(u - b)^2}{2 \lambda^2 a^2}\right]\right\} \, \mathrm{d} u \\
        	& \leq \int_b^{b + \lambda  a} \frac{s}{b^2} \exp \left\{ - \left[ \frac{s}{b + \lambda  a} + \frac{s(u - b)^2}{2 \lambda^2 a^2}\right]\right\} \, \mathrm{d} u \\
        	& = \frac{s}{b^2} \exp \left\{ - \frac{s}{b + \lambda a}\right\} \int_b^{b + \lambda  a} \exp \left\{ - \frac{ s (u - b)^2}{2 \lambda^2 a^2}\right\} \, \mathrm{d} u \\
        	& \leq \frac{s}{b^2} \exp \left\{ - \frac{s}{b + \lambda a}\right\} \int_b^{\infty} \exp \left\{ - \frac{(u - b)^2}{2 \lambda^2 a^2 / s}\right\} \, \mathrm{d} u = \frac{\sqrt{2 \pi} \lambda a }{2 b^2} \exp \left\{ -\frac{s}{b + \lambda a}\right\}
        \end{aligned}
    \]
    For $II$, we decompose it as
    \[
        \begin{aligned}
            II &= \int_{b + \lambda  a}^{\infty} \frac{s}{u^2}  \exp \left\{ - \left[ \frac{s}{u} + \frac{s(u - b)}{2 \lambda a }\right]\right\} \, \mathrm{d} u \\
            & = \left[ \int_{b + \lambda a}^{\frac{2 \lambda a s}{b}} + \int_{\frac{2 \lambda a c}{b}}^{\infty} \right] \frac{s}{u^2}  \exp \left\{ - \left[ \frac{s}{u} + \frac{s(u - b)}{2 \lambda a }\right]\right\} \, \mathrm{d} u \\
            & = \int_{b + \lambda a }^{\frac{2 \lambda a s}{b}} \frac{s}{u^2}  \exp \left\{ - \left[ \frac{s}{u} + \frac{s(u - b)}{2 \lambda a }\right]\right\} \, \mathrm{d} u + \exp \left\{ -\frac{s}{b}\right\} \int_{\frac{2 \lambda a s}{b}}^{\infty} \frac{s}{u^2} \exp \left\{ - \frac{s(u - b) ( u - \frac{2 \lambda a s}{b})}{2 \lambda a u}\right\} \, \mathrm{d} u \\
            & =: II_1 + II_2 .
        \end{aligned}
    \] 
    And for $II_1$, 
    \[
        \begin{aligned}
            II_1 & = \int_{b + \lambda a}^{\frac{2 \lambda a s}{b}} \frac{s}{u^2}  \exp \left\{ - \left[ \frac{s}{u} + \frac{s(u - b)}{2 \lambda a }\right]\right\} \, \mathrm{d} u \\
            & = \int_{b + \lambda a}^{\frac{2 \lambda a s}{b}} \frac{s}{u^2}  \exp \left\{ - s \times \left[ \frac{1}{u} + \frac{(u - b)}{2 \lambda a }\right]\right\} \, \mathrm{d} u \\
            \alignedoverset{g(u) = \frac{1}{u} + \frac{(u - b)}{2 \lambda a } \text{ is increasing on } u \geq \sqrt{2 \lambda a}}{\leq} \int_{b + \lambda a}^{\frac{2 \lambda a s}{b}} \frac{s}{u^2}  \exp \left\{ - s \times \left[ \frac{1}{\sqrt{2 \lambda a} \vee (b + \lambda a )} + \frac{(\sqrt{2 \lambda a} \vee (b + \lambda a ) - b)}{2 \lambda a }\right]\right\} \, \mathrm{d} u  \\
            & \leq \frac{s}{(b + \lambda a) } \left[ \frac{2 \lambda a s}{b} - \big( b + \lambda a \big) \right] \exp \left\{ - \frac{s}{\sqrt{2 \lambda a} \vee (b + \lambda a )}\right\}.
        \end{aligned}
    \]
    For $II_2$, if $\frac{2 \lambda a s}{b} \geq 1$, i.e. $\lambda s \geq \frac{b}{2 a}$, we obtain that
    \[
        \begin{aligned}
        	II_2 & = \exp \left\{ -\frac{s}{b}\right\}  \int_{\frac{2 \lambda a s}{b}}^{\infty} \frac{s}{u^2} \exp \left\{ - \frac{s(u - b) ( u - \frac{2 \lambda a s}{b})}{2 \lambda a u}\right\} \, \mathrm{d} u \\
        	& \leq \exp \left\{ -\frac{s}{b}\right\}  \int_{\frac{2 \lambda a s}{b}}^{\infty} \frac{s}{u^2} \exp \left\{ - \frac{s(u - \frac{2 \lambda a s}{b})^2}{2 \lambda a u}\right\} \, \mathrm{d} u \\
            \alignedoverset{\text{by } \frac{2 \lambda a s}{b} \geq 1}{\leq} s \exp \left\{ -\frac{s}{b}\right\} \int_{\frac{2 \lambda a s}{b}}^{\infty} \frac{1}{u^{3 / 2}} \exp \left\{ - \frac{(u - \frac{2 \lambda a s}{b})^2}{2 \lambda a u}\right\} \, \mathrm{d} u \\
            \alignedoverset{\text{by } \frac{2 \lambda a s}{b} \geq 1}{\leq} s \exp \left\{ -\frac{s}{b}\right\} \sqrt{\frac{2 \lambda a s}{b}} \int_{\frac{2 \lambda a s}{b}}^{\infty} \frac{1}{u^{3 / 2}} \exp \left\{ - \frac{(u - \frac{2 \lambda a s}{b})^2}{2 \lambda a u}\right\} \, \mathrm{d} u \\
            & = s \exp \left\{ - \frac{s}{b}\right\}   \sqrt{\frac{2 \lambda a s}{b}} \frac{b}{2s}\sqrt{\frac{\pi}{2 \lambda a}} \left[ 1 - e^{4 s / b} \frac{2}{\sqrt{\pi}} \int_{2 \sqrt{s / b}}^{\infty} e^{-t^2} \,  \mathrm{d} t \right] \\
            \alignedoverset{\text{by Lemma \ref{lem_erfc}}}{\leq} s \exp \left\{ - \frac{s}{b}\right\}   \sqrt{\frac{2 \lambda a s}{b}} \frac{b}{2s}\sqrt{\frac{\pi}{2 \lambda a}} (1 - 0) = \frac{\sqrt{\pi b s}}{2} \exp \left\{ - \frac{s}{b}\right\}.
        \end{aligned}
    \]
    Combining these results, we obtain that
    \[
        \begin{aligned}
        	\E \exp \left\{ - \frac{s}{aX + b}\right\} & \leq \exp \left\{ -\frac{s}{b}\right\} + I +  II_1 + II_2 \\
        	& \leq \exp \left\{ -\frac{s}{b}\right\} + \frac{\sqrt{2 \pi} \lambda a s}{2 b^2} \exp \left\{ -\frac{s}{b + \lambda a}\right\} \\
        	& \qquad + \frac{s}{(b + \lambda a) } \left[ \frac{2 \lambda a s}{b} - \big( b + \lambda a \big) \right] \exp \left\{ - \frac{s}{\sqrt{2 \lambda a} \vee (b + \lambda a )}\right\} + \frac{\sqrt{\pi b s}}{2} \exp \left\{ - \frac{s}{b}\right\} \\
                & \leq \left[ 1 +  \frac{\sqrt{\pi b s}}{2}  \right] \exp \left\{ - \frac{s}{b}\right\} + \left[ \frac{\sqrt{2 \pi} \lambda a s}{2 b^2}  + \frac{2 \lambda a s^2}{b(b + \lambda s)}\right] \exp \bigg\{ - \frac{s}{\sqrt{2 \lambda a} \vee (b + \lambda a)}\bigg\} \\
                & \leq \left[ 1 +  \frac{\sqrt{\pi b s}}{2}  \right] \exp \left\{ - \frac{s}{b} \right\} + \left[ \frac{\sqrt{2 \pi} \lambda a}{2 b^2} + \frac{2 a}{b}\right] s \exp \bigg\{ - \frac{s}{\sqrt{2 \lambda a} \vee (b + \lambda a)}\bigg\}  \\
                \alignedoverset{\text{by } b \leq \sqrt{2 \lambda a} \vee (b + \lambda a)}{\leq} \left[  1 +  \frac{\sqrt{\pi b}}{2}  + \frac{\sqrt{2 \pi} \lambda a}{2 b^2} + \frac{2 a}{b} \right] s \exp \bigg\{ - \frac{s}{\sqrt{2 \lambda a} \vee (b + \lambda a)}\bigg\} ,
        \end{aligned}
    \]
    which gives the result we need.
\end{proof}

\begin{lemma}\label{lem_erfc}
	For any $x \geq 0$,
	\[
	    \frac{2 e^{x^2}}{\sqrt{\pi}} \int_x^{\infty} e^{-t^2} \, \mathrm{d} t \leq 1.
	\]
\end{lemma}

\begin{proof}
	\[
	    \begin{aligned}
	    	\frac{2 e^{x^2}}{\sqrt{\pi}} \int_x^{\infty} e^{-t^2} \, \mathrm{d} t & = \frac{2}{\sqrt{\pi}} \int_x^{\infty} e^{- (t^2 - x^2)} \, \mathrm{d} t \\
	    	& \leq \frac{2}{\sqrt{\pi}} \int_x^{\infty} e^{- (t - x)^2} \, \mathrm{d} t \\
	    	& = \frac{2}{\sqrt{\pi}} \int_0^{\infty} \frac{e^{- t^2}}{2 \times \frac{1}{2}} \, \mathrm{d} t = \frac{2}{\sqrt{\pi}} \frac{\sqrt{2 \pi \times \frac{1}{2}}}{2} = 1.
	    \end{aligned}
	\]
\end{proof}

\end{document}